\documentclass[12pt]{article} 
\usepackage{amsmath,amsfonts,amssymb,graphicx,url,amsthm, mathrsfs}
\numberwithin{equation}{section}
\numberwithin{figure}{section}
\numberwithin{table}{section}
\usepackage[margin=1in]{geometry}
\usepackage{enumitem}
\usepackage{xcolor,nicefrac}
\usepackage{graphicx}
\usepackage{booktabs}
\usepackage{longtable}
\usepackage{pdflscape}
\usepackage{multirow}
\usepackage{caption}
\usepackage{float}
\usepackage{bbm}
\usepackage{subcaption}
\usepackage{afterpage}
\usepackage{placeins}
\usepackage[noend]{algpseudocode}
\usepackage[section,ruled]{algorithm}
\usepackage[export]{adjustbox}
\usepackage[round,authoryear]{natbib}
\usepackage{chapterbib} 
\usepackage[mathscr]{eucal}
\usepackage{setspace}
\usepackage[hyperfootnotes=true]{hyperref}
\hypersetup{colorlinks=true,raiselinks=true,breaklinks=true,citecolor=blue}

\usepackage{tikz,pgfplots,pgfplotstable}
\pgfplotsset{compat=newest}
\usetikzlibrary{arrows,decorations,backgrounds,positioning,calc,matrix}
\usepgfplotslibrary{groupplots}
\usetikzlibrary{external}
\tikzset{
    png export/.style={
        external/system call/.add={}{; convert -density 600 -transparent white "\image.pdf" "\image.png"},
        /pgf/images/external info,
        /pgf/images/include external/.code={%
            \includegraphics
            [width=\pgfexternalwidth,height=\pgfexternalheight]
            {##1.png}%
        },
    },
    png export,
}

\newtheorem{theorem}{Theorem}[section]
\newtheorem{lemma}[theorem]{Lemma}
\newtheorem{corollary}[theorem]{Corollary}
\newtheorem{proposition}[theorem]{Proposition}
\newtheorem{definition}[theorem]{Definition}

\theoremstyle{definition}
\newtheorem{example}[theorem]{Example}

\theoremstyle{plain}

\newtheorem{assumption}{Assumption}
\usepackage{xparse}
\usepackage{verbatim}
\newsavebox{\fminipagebox}
\NewDocumentEnvironment{fminipage}{m O{\fboxsep}}
 {\par\kern#2\noindent\begin{lrbox}{\fminipagebox}
  \begin{minipage}{#1}\ignorespaces}
 {\end{minipage}\end{lrbox}%
  \makebox[#1]{%
    \kern\dimexpr-\fboxsep-\fboxrule\relax
    \fbox{\usebox{\fminipagebox}}%
    \kern\dimexpr-\fboxsep-\fboxrule\relax
  }\par\kern#2
 }

\def\letters{a,b,c,d,e,f,g,h,i,j,k,l,m,n,o,p,q,r,s,t,u,v,w,x,y,z}
\def\Letters{A,B,C,D,E,F,G,H,I,J,K,L,M,N,O,P,Q,R,S,T,U,V,W,X,Y,Z}
\makeatletter
\@for \@l:=\Letters \do{%
  \expandafter\edef\csname\@l bb\endcsname{\noexpand\ensuremath{%
  \noexpand\mathbb{\@l}}}%
  \expandafter\edef\csname\@l bf\endcsname{{\noexpand\bf \@l}}%
  \expandafter\edef\csname\@l cal\endcsname{\noexpand\ensuremath{%
  \noexpand\mathcal{\@l}}}%
  \expandafter\edef\csname\@l eu\endcsname{\noexpand\ensuremath{%
  \noexpand\EuScript{\@l}}}%
  \expandafter\edef\csname\@l frak\endcsname{\noexpand\ensuremath{%
  \noexpand\mathfrak{\@l}}}%
  \expandafter\edef\csname\@l rm\endcsname{{\noexpand\rm \@l}}%
  \expandafter\edef\csname\@l scr\endcsname{\noexpand\ensuremath{%
  \noexpand\mathscr{\@l}}}%
}
\@for \@l:=\letters \do{%
  \expandafter\edef\csname\@l bf\endcsname{{\noexpand\bf \@l}}%
  \expandafter\edef\csname\@l frak\endcsname{\noexpand\ensuremath{%
  \noexpand\mathfrak{\@l}}}%
  \expandafter\edef\csname\@l scr\endcsname{\noexpand\ensuremath{%
  \noexpand\mathscr{\@l}}}%
}
\makeatother
\definecolor{shadecolor}{rgb}{0.6, 0.6, 0.6} 
\definecolor{red}{rgb}{1,0,0.2} 
\definecolor{darkgreen}{rgb}{0, 0.6, 0}

\definecolor{myred}{rgb}{1,0,0.2}
\definecolor{MYRED}{rgb}{1,0,0.2}
\newcommand{\isdef}{\mathrel{\mathrel{\mathop:}=}}

\newcommand{\HS}{\operatorname{HS}}
\newcommand{\Sig}{\mathnormal{\Sigma}}
\newcommand{\op}{\mathrm{op}}

\newcommand{\tr}{\operatorname{tr}}

\newcommand{\wt}[1]{\widetilde{#1}}

\renewcommand{\d}{\operatorname{d}\!}
\newcommand{\bs}{\boldsymbol}

\DeclareMathOperator{\spn}{span}

\DeclareMathOperator*\argmin{\arg\!\min}

\DeclareMathOperator\dist{dist}

\renewcommand{\mid}{\hspace*{0.2ex}|\hspace*{0.2ex}}

\begin{document}

\title{
\textbf{Kernel-Based Nonparametric Tests for Shape Constraints}
}

\author{
Rohan Sen
\thanks{Department of Finance, Hong Kong University of Science and Technology,
Clear Water Bay, Kowloon, Hong Kong SAR.
Email: \texttt{rohansen@ust.hk}.
The author is grateful to Paul Schneider, Michael Multerer,
Patrick Gagliardini, Damir Filipovi{\'c}, Markus Pelger,
Anastasija Tetereva, Xinghua Zheng, Pierluigi Vallarino,
and Emanuele Luzzi for helpful comments and suggestions.
The author also thanks participants at the 19th International
Joint Conference CFE-CMStatistics 2025 at the University of London,
the 18th Annual Meeting of the Society for Financial Econometrics
(SoFiE) 2026 at the University of Macau,
the HKUST IAS-SBM Joint Workshop on Financial Econometrics
in the Big Data Era 2026,
and the Financial Econometrics and Machine Learning (FinEML)
Conference 2026 at the University of Geneva
for valuable discussions and feedback.}
}

\date{}

\maketitle
\begin{abstract}
Many hypotheses in economics and finance are naturally expressed in terms of the shape of unknown functions. Prominent examples include the monotonicity of expected returns with respect to market risk or bond maturity and the convexity of the pricing kernel as a function of the market return. We develop kernel-based nonparametric tests of monotonicity and convexity when the underlying functions can be inferred only from noisy and possibly indirect payoff observations. We derive finite-sample bounds for payoff estimation under weak dependence and construct a joint one-sided Wald test based on the derivatives of the underlying function at prespecified points. The tests are asymptotically valid under serial dependence and consistent against fixed violations. Empirically, we find that several apparent departures from theoretical predictions in equity markets are not statistically significant. For Treasury bills, maturity premia rise at short maturities but do not increase monotonically over the entire maturity range, with stronger evidence against increasing premia when conditioning on the short rate. \\[20pt]


\end{abstract}

\newpage
\section{Introduction}\label{sec:introduction}

Does greater exposure to market risk imply a higher expected return? Do investors earn an increasing premium for holding longer-maturity securities? Financial theory often offers qualitative predictions about such relationships. When the market risk premium is positive, the \emph{capital asset pricing model} (CAPM) predicts that expected excess returns increase with market beta. At the same time, the traditional liquidity preference hypothesis suggests that Treasury term premia increase with maturity \citep{Fama1986,Richardson1992}. The pricing kernel, which determines how investors value payoffs across future states, is decreasing and convex in the market return in representative-agent models with risk-averse and prudent preferences when aggregate consumption is proportional to market wealth \citep{HansenSingleton1983,Kimball1990}. Likewise, in option pricing, the requirement that state prices be nonnegative translates into restrictions on the slope and curvature of the call pricing function with respect to the strike \citep{AitSahaliaDuarte2003}. In each of these examples, financial theory is informative about certain properties of the relationship but leaves considerable latitude as to its functional form. If a parametric specification is adopted, an empirical assessment of the theoretical restrictions becomes inseparable from an assessment of the specification itself. A nonparametric approach offers the possibility of examining these restrictions without committing to a particular functional form. The difficulty, however, is that even when the underlying relationship conforms to the theory, its estimated derivatives may suggest otherwise, particularly in small samples. Imposing the restrictions during estimation may improve statistical precision, but precludes an investigation of whether the data actually support them. It is this distinction between imposing and testing shape restrictions that motivates the present paper.

The distinction is of more than methodological interest, as several long-standing ``puzzles" in the empirical literature show. \citet{FAMA2006} find that the relationship between market beta and average returns varies across historical periods and that differences in beta unrelated to size and book-to-market characteristics are largely unrewarded, which raises the question of whether expected returns increase with beta at all, even without imposing the linear form of the CAPM. The empirical literature on the pricing kernel offers an equally striking example. Early studies find nonmonotonic behavior around the at-the-money region \citep{aitsahalia2000nonparametric,jackwerth2000recovering,rosenberg2002empirical}, later studies document U-shaped patterns \citep{Bakshi2010,christoffersen2012capturing,Polkovnichenko2013}, and the conclusions depend on how market conditions are accounted for: \citet{SONG2016176} emphasize the role of volatility, whereas \citet{linnshiveshumway} find a decreasing shape when conditioning information is handled consistently. More recent work studies state-dependent valuation and the impact of option demand \citep{SCHREINDORFER2025104106,AlmeidaFreire2026,DalderopLinton2026,Luzzi2025}. What remains unclear in many of these studies is how much of the apparent departure from theory can be attributed to the manner in which the underlying relationship is estimated, and how much reflects economically meaningful variation across return states or market environments. Resolving this question calls for formal statistical inference on shape restrictions.

In many financial applications, the problem is considerably more intricate than that of determining the shape of an ordinary regression function. The relationship of economic interest is not observed directly but must be inferred from the prices or returns to which it gives rise. Consider, for instance, the relationship between firm characteristics and expected returns. A characteristic-managed portfolio combines the returns of many firms, so that what is observed at the portfolio level reflects the combined influence of their characteristics. This connects traditional cross-sectional regressions \citep{FamaMacBeth1973} to models that allow nonlinear effects of characteristics \citep{Freyberger2020,KozakNagelSantosh2020,GuKellyXiu2020} and accommodate variation in expected returns and factor exposures with market conditions \citep{nagelsingleton2011,GuKellyXiu2021Autoencoder}. Similar considerations arise in the construction of portfolio policies from firm characteristics \citep{brandt2009parametric} and in the recovery of a discount curve from Treasury prices, each of which may reflect cash flows at several maturities \citep{FilipovicPelgerYe2026Discount}. In option markets, the spanning representation of \citet{carr2001optimal} implies that the slope and curvature of a smooth terminal-payoff function determine positions in the underlying asset and in options across strikes. However, what one observes is the aggregate payoff of these positions. The difficulty is thus twofold: one must recover a function from observations that may reveal relatively little about its local behavior and subsequently determine whether the estimated features of that function can be attributed to sampling variation.

Our approach starts from the observation that, even when the function of interest cannot be observed directly, its ability to explain the realized payoffs provides a natural criterion for estimation. We accordingly seek a function whose generated payoffs approximate the observed ones in the mean-square sense. The choice of a suitable \emph{reproducing kernel Hilbert space} (RKHS) serves two purposes: it accommodates a rich variety of functional forms and provides the analytical structure needed to study their derivatives. There is, however, a fundamental difficulty with estimating a function solely on the basis of the payoffs it generates. Two functions that are indistinguishable in terms of their payoffs may have very different slopes and curvatures. Tikhonov regularization allows us to resolve this ambiguity by selecting a unique functional target without prescribing its shape. We establish high-probability finite-sample bounds for the estimation of the generated payoffs under polynomially decaying $\tau$-mixing. Having estimated the function without imposing the shape restrictions, we turn to the question of which features of the estimated relationship are statistically distinguishable from sampling variation. The problem is not simply to determine whether an individual derivative differs significantly from zero, but whether the collection of derivative estimates over a region is compatible with the restrictions suggested by theory. To this end, we develop a joint one-sided Wald test for inequalities on function values or derivatives at prespecified points, accounting for serial dependence through estimation of the long-run covariance matrix. Under strong mixing, the configuration in which all inequalities bind yields a chi-bar-square limiting distribution, from which critical values controlling asymptotic size can be obtained. The resulting tests are consistent against fixed violations and allow us to ascertain whether there is sufficient statistical evidence to reject the restrictions without requiring them to hold during estimation.

To illustrate the empirical relevance of the proposed framework, we estimate equity risk premia and Treasury term premia nonparametrically and subsequently test the monotonicity postulates offered by financial theory. In each application, the restrictions are tested on the population regularized function at finitely many points within economically relevant, prespecified regions. For equities, we examine how expected excess returns vary with market beta and eight firm characteristics, motivated by earlier evidence that questions simple parametric specifications \citep{FAMA2006,Freyberger2020}. Although the fitted curves display markedly different patterns across characteristics, many of the apparent violations of monotonicity are statistically insignificant. For Treasury bills, we revisit the liquidity preference hypothesis \citep{Richardson1992}, taking into account the prevailing interest-rate environment, which is particularly relevant for the valuation of near-money assets \citep{Nagel2016}. The analysis also reflects the broader importance of conditioning information in asset pricing \citep{Hansen1987,nagelsingleton2011}. From 1973 to 2001, our unconditional tests provide evidence against increasing maturity premia only at the longest maturities considered. Once we allow premia to depend on the prevailing short rate, the conditional tests also reject weakly increasing premia at intermediate maturities (five to eight months) in all three interest-rate environments examined.

The statistical foundations of our testing procedure can be traced to the early literature on ordered and multivariate one-sided alternatives \citep{Bartholomew1959,Kudo1963}. \citet{Gourieroux1982} extend this theory to linear regression models, while \citet{Shapiro1985,Shapiro1988} develops a more general treatment of chi-bar-square asymptotics through projections onto convex cones. Of particular relevance to our setting is the work of \citet{Wolak1987,Wolak1989}, who places linear inequality restrictions under the null hypothesis rather than the alternative. The former develops exact tests in Gaussian linear regression models, while the latter establishes asymptotic results for a broader class of linear econometric models. Both studies derive the least-favorable distribution from the configuration in which all inequalities bind. For unknown functions, a large literature studies estimation under monotonicity and convexity restrictions, including least squares estimation of multivariate convex regression functions \citep{Sen2011}; see \citet{GuntuboyinaSen2018} for a review. \citet{Hall2000,Juditsky2002,Chetverikov_2019} develop nonparametric tests for monotonicity and convexity, while \citet{FangSeo2021} study inference for general convex-cone restrictions using projection methods and bootstrap critical values. In finance, \citet{PattonTimmermann2010} propose a test of whether the evidence supports an increasing relationship across sorted portfolios, and \citet{RomanoWolf2013} show that establishing such an ordering requires care, since a procedure may favor monotonicity even when expected returns are nonmonotone across intermediate portfolios. These studies seek to establish an ordering from the data. Our question is complementary: when theory predicts monotonicity, is there sufficient evidence to reject it? Following the inequality-null formulation of \citet{Wolak1989}, we express the restrictions in terms of derivatives of a regularized function learned from potentially indirect payoff observations, evaluated at economically specified points, rather than differences between sorted portfolios. This permits a common treatment of restrictions on slopes and curvature, but requires connecting the classical testing theory to functional estimation and accounting for dependence in the observed returns.

That the payoffs may be identified without the function generating them being uniquely determined places our analysis within the broader econometric literature on inverse problems. \citet{CarrascoFlorensRenault2007,CarrascoFlorensRenault2014} study identification, ill-posedness, regularization, and asymptotic inference in such problems. Particularly relevant to our approach is the work of \citet{CarrascoFlorens2000,CarrascoFlorens2014}, who exploit the RKHS associated with the covariance operator of the moment conditions to formulate continuum GMM and characterize its asymptotic efficiency, together with \citet{CarrascoChernovFlorensGhysels2007}, who address the regularization of this operator in the estimation of dynamic models. Whereas these studies use the RKHS norm to construct efficient estimators, we model the unknown function in an RKHS rich enough to accommodate a wide range of sufficiently smooth functional forms, including polynomials, log-polynomials, and splines. Beyond this functional flexibility, the choice of an appropriate RKHS allows us to exploit the derivative reproducing property, whereby derivative evaluations constitute bounded linear functionals on the space \citep{zhou2008derivative}. Inference on slopes and curvature can thus be formulated in terms of inner products of the type studied by \citet{CarrascoFlorensRenault2014}, although our object of inference is the regularized population function itself. This also distinguishes our approach from nonparametric instrumental-variable methods that exploit shape restrictions for identification \citep{FreybergerHorowitz2015,ChetverikovWilhelm2017}. In financial econometrics, kernel methods have been employed to value portfolios and manage risk, estimate conditional moments, and recover discount curves from bond prices \citep{Boudabsa2021,Filipovic2024,FilipovicPelgerYe2026Discount}. Rather than focusing solely on the recovery of such functions, we develop statistical tests of the qualitative restrictions that financial theory places on them. The theoretical analysis draws on regularized learning theory \citep{Caponnetto2006}, its developments under dependent sampling \citep{SteinwartHushScovel2009,BlanchardZadorozhnyi2019}, and asymptotic inference for regularized kernel estimators \citep{Hable2012}.

The rest of the article proceeds as follows. In Section~\ref{sec:framework}, we introduce the general payoff-based framework and develop its RKHS formulation. Section~\ref{sec:stat_prop} establishes the statistical properties of the learning procedure, while Section~\ref{sec:inference} develops the asymptotic theory underlying our tests for shape restrictions. We present the empirical applications in Section~\ref{sec:empirical_applications} and conclude in Section~\ref{sec:conclusion}. We collect the mathematical proofs, the Monte Carlo simulation studies, and the computational details along with the supplementary results for the empirical exercises in the Appendix. 

\section{General framework}\label{sec:framework}

\paragraph{Notation.}
For $s\in\Nbb$, $\Cscr^s(\Xcal)$ denotes the class of $s$-times continuously differentiable functions from $\Xcal$ to $\Rbb$. For a multi-index $\bs\alpha=(\alpha_1,\ldots,\alpha_d)\in\Nbb_0^d$, we denote $|\bs\alpha|\isdef\alpha_1+\cdots+\alpha_d$ and $D^{\bs\alpha}f \isdef f^{(\bs\alpha)}$. The notations $O$, $o$, $O_p$, and $o_p$ are used in their usual sense, and $a\lesssim b$ when $a$ is bounded by a constant multiple of $b$ uniformly in the parameters under consideration. For a linear operator $A$, $A^\ast$ denotes its adjoint; for $f$ and $g$ in a Hilbert space $\Hscr$, $(f\otimes g)u\isdef\langle u,g\rangle_\Hscr f$; and $P_\Hscr(u;\Mscr)$ denotes the metric projection of $u$ onto a nonempty closed convex set $\Mscr\subset\Hscr$. We provide all the necessary details used in the proofs and results in Appendix~\ref{sec:prelims} for completeness.
\subsection{Optimal payoffs}
\label{subsec:population_payoff_projection}
We consider a real vector space of candidate functions $\Fcal$ that represents admissible specifications of the underlying economic relationship from which the observed payoffs are generated, with $h\in\Fcal$ denoting a generic candidate. Consider $Z$, the observable input(s) with the law $\Pbb$ that determine how any $h$ enters the payoff function. For any realization $Z=z$, we write $\Rscr(h;z)$ for the payoff implied by $h$, and let $Y=Y(Z) \in L^2_\Pbb$ denote the \emph{target payoff} against which it is compared. We make the following assumptions.

\medskip \noindent \textit{Standing payoff assumptions}: (i) For any $z$, the payoff depends linearly on the candidate $h$, i.e., the map $h \to\Rscr(h;z)$ is linear on $\Fcal$; and (ii) $\Rscr(h;Z)\in L^2_{\Pbb}$ for every $h\in\Fcal$. Under these assumptions, the \emph{payoff operator}
\begin{equation}\label{eq:payoff_operator}
    J_\Pbb \colon \Fcal \to L^2_\Pbb, \qquad  J_{\Pbb}h=\Rscr(h;Z)
\end{equation}
that assigns to each candidate $h\in \Fcal$, the $L^2_\Pbb$-equivalence class of payoffs it induces is well-defined and linear.  The range $\operatorname{Im}(J_{\Pbb}) \subset L^2_\Pbb $ consists of the payoffs generated by $\Fcal$. To compare these generated payoffs with the target payoff, we use the mean squared loss 
\begin{equation}\label{eq:pop-loss-payoff}
    \Ecal(h)
    \isdef 
    \frac12
    \Ebb^{\Pbb}
    \left[
        \bigl(Y-\Rscr(h;Z)\bigr)^2
    \right]
    =
    \frac12
    \|Y-J_{\Pbb}h\|_{L^2_{\Pbb}}^2.
\end{equation}
The loss function distinguishes candidate functions only through the payoffs they generate. Thus, the population problem is a best-approximation problem in the payoff space: among all payoffs attainable via $\Fcal$, we aim to find the one closest to $Y$. Since this space need not be closed, the best approximation may not be attained within this class. Therefore, we consider its closure and define the \emph{projected payoff} and the corresponding optimal loss
\begin{equation}\label{eq:Rstar-payoff}
    \Rscr^\star
    \isdef 
    \argmin_{\Rscr\in\overline{\operatorname{Im}(J_{\Pbb})}}
    \|Y-\Rscr\|_{L^2_{\Pbb}}^2, \qquad 
    \Ecal^\star \isdef 
    \frac12
    \|Y-\Rscr^\star\|_{L^2_{\Pbb}}^2.
\end{equation}
Hence, $\Ecal^\star$ is the infimum of the mean squared loss achievable by the candidate class, allowing for arbitrarily close approximations. By the \emph{projection theorem}\footnote{see \citet[Theorem 3.16 and Corollary 3.24]{Bauschke2017}.}, 
\begin{equation*}
    Y-\Rscr^\star
    \perp
    \overline{\operatorname{Im}(J_{\Pbb})},
    \qquad
    \Ecal^\star
    =
    \inf_{h\in\Fcal}\Ecal(h),
\end{equation*}
and we also obtain the exact identity
\begin{equation}\label{eq:excess-loss-identity}
    \Ecal(h)-\Ecal^\star
    =
    \frac12
    \|J_{\Pbb}h-\Rscr^\star\|_{L^2_{\Pbb}}^2,
    \qquad h\in\Fcal.
\end{equation}
The identity has a simple interpretation: minimizing the population loss is equivalent to bringing the payoff generated by $h$ as close as possible to the projected payoff. 

\subsection{Examples}
\label{subsec:examples}
We provide two examples that illustrate the payoff framework discussed above. 

\begin{example}[Characteristic-sorted portfolio returns]
\label{ex:characteristics}
At date $t$, suppose we observe $N_t$ firms, with characteristic $x_{i,t}$ and subsequent excess return $r_{i,t+1}$ for firm $i$, $1\leq i\leq N_t$. To describe how expected excess returns depend on this characteristic, we use the conditional mean function $h(x)=\Ebb^\Pbb[r_{i,t+1}\mid x_{i,t}=x]$. Let $\omega_{i,t}$ denote portfolio weights determined using information available at date $t$, summing to one for a fully invested portfolio. The realized portfolio return is $Y_t=\sum_{i=1}^{N_t}\omega_{i,t}r_{i,t+1}$, while its model-implied expected return is $\Rscr(h;Z_t)=\sum_{i=1}^{N_t}\omega_{i,t}h(x_{i,t})$, where $Z_t$ comprises the observed characteristics, portfolio weights, and subsequent returns. The population criterion \eqref{eq:pop-loss-payoff} evaluates candidate functions according to how well the expected returns they assign to individual firms explain realized portfolio returns across dates. Under suitable assumptions, the firm-level conditional mean minimizes this criterion.\footnote{Given that $h\in\Fcal$, the firm-level conditional mean minimizes the population criterion if and only if $\Ebb^\Pbb[(Y_t-\Rscr(h;Z_t))\,\Rscr(g;Z_t)]=0$ for all $g\in\Fcal$, in which case $J_\Pbb h=\Rscr^\star$. Conditioning on each firm's own characteristic need not ensure this orthogonality when portfolio weights depend on additional return-predictive information. A sufficient condition is $\Ebb^\Pbb[r_{i,t+1}\mid x_{1,t},\dots,x_{N_t,t},\omega_{1,t},\dots,\omega_{N_t,t}]=h(x_{i,t})$ for each $i$. Without the orthogonality condition, the criterion targets the projected payoff $\Rscr^\star$, whose generating functions need not share the slope of the firm-level conditional mean.} More generally, the criterion identifies the best approximation to the observed portfolio returns in the maintained payoff space. 

The economic predictions of interest translate naturally into restrictions on the derivatives of $h$. If the characteristic is market beta, for instance, the CAPM predicts a positive slope when the market risk premium is positive. Rather than prescribing a linear relationship, we can learn $h$ nonparametrically and test whether its slope is nonnegative over economically relevant ranges of beta. The same formulation can also capture predictions of decreasing, convex, or concave relationships between expected returns and other firm characteristics.
\end{example}

\begin{example}[Option spanning and tradable portfolios]
\label{ex:option-span}
Consider an investor who trades the index forward and European options with a common maturity. Normalize the date-$t$ forward price to one, let $R_{t+1}$ denote the gross forward return, and express option strikes $K$ in forward-normalized units. For a sufficiently smooth function $h$, the spanning formula of \citet{bakshimadan2000spanning,carr2001optimal} gives
\begin{align*}
    h(R_{t+1})
    ={}& h(1)+h'(1)(R_{t+1}-1)\\
    &+\int_0^1 h''(K)(K-R_{t+1})_+\, \d K
    +\int_1^\infty h''(K)(R_{t+1}-K)_+\,\d K.
\end{align*}
This identity has a direct trading interpretation. The function $h$ determines a terminal payoff through a position $h(1)$ in the riskless asset, a position $h'(1)$ in the forward, and an option strip with position density $h''(K)$. The options are puts for strikes below the forward and calls for strikes above it. Thus, the investor's forward and option positions are jointly determined by the derivatives of a single function $h$. In practice, the investor can trade only the options listed in the market. At date $t$, let $K_{1,t}<\cdots<K_{J_t,t}$ denote the observed
forward-normalized strikes, and let $R^O_{j,t+1}$ denote the
net payoff from a long position in option $j$, equal to its
terminal payoff minus the date-$t$ forward value of its premium,
with all quantities expressed in forward-normalized units. Let $w_{j,t}$ denote the numerical integration weight determined by the observed strike grid. The corresponding zero-cost portfolio payoff is $\Rscr(h;Z_t)
    =
    h'(1)(R_{t+1}-1)
    +
    \sum_{j=1}^{J_t}
    w_{j,t}h''(K_{j,t})R^O_{j,t+1},$ where $Z_t$ denotes the strikes, integration weights, and realized forward and option returns. The investor holds $h'(1)$ units of the forward excess payoff and $w_{j,t}h''(K_{j,t})$ units of option $j$. These positions define a systematic trading strategy that adjusts to the options available at each date. The resulting payoff is linear in $h$ and can be compared with a target payoff $Y$ using the population criterion \eqref{eq:pop-loss-payoff}. An important application is option-based SDF recovery: taking $Y=1$, the projected payoff $\Rscr^\star$ is the maximum Sharpe payoff in the closure of the maintained traded-portfolio class, and whenever this maximum Sharpe ratio is finite, so that $1-\Rscr^\star\neq0$, the normalized residual gives the corresponding minimum-second-moment SDF; see \citet{hansen1991implications}, \citet{almostmodelfree}, and \citet{Luzzi2025}. In the continuum-spanning benchmark, the derivatives of this SDF have signs opposite to those of the optimal $h$, so monotonicity and convexity restrictions can be studied through the implied forward and option positions, providing evidence on how the market values different index-return states.
\end{example}

\subsection{Modeling payoffs in reproducing kernel Hilbert space (RKHS)}
\label{subsec:rkhs_estimator_representer}

Section~\ref{subsec:population_payoff_projection} identifies the payoff that best approximates the target payoff $Y$, but the shape restrictions we study apply to the underlying candidate functions from which such payoffs are generated. In Example~\ref{ex:characteristics}, the slope of the firm-level conditional mean $h$ captures how expected returns change with the sorting characteristic, and under the orthogonality condition stated there, $h$ generates the projected payoff. In Example~\ref{ex:option-span}, the first and second derivatives determine the positions in the forward and across option strikes. Yet different functions can generate the same payoff while having different derivative behavior. Accordingly, we represent candidate functions in an appropriate \emph{reproducing kernel Hilbert space} (RKHS) $\Hscr\subset\Fcal$ with sufficient smoothness for the derivative evaluations of interest.\footnote{A natural choice is a sufficiently smooth Sobolev RKHS, which accommodates flexible nonlinear functions while ensuring, through Sobolev embedding, that the relevant derivative evaluations are bounded linear functionals; see \citet{Adams_book,Wen05}. The properties of such spaces and derivative reproducing results used below are provided in Appendix~\ref{subsec:RKHS}.} The RKHS norm also provides the penalty used below to select a unique function for inference.

\medskip
\noindent\textit{Standing RKHS payoff assumption.}
For every realization $z$, the mapping $h\mapsto\Rscr(h;z)$ is a bounded linear functional on $\Hscr$. Let $\psi(z)\in\Hscr$ denote its unique Riesz representer. We assume that $\Hscr$ is separable, which holds for the RKHS of every continuous kernel on a separable domain, and that $z\mapsto\psi(z)$ is Borel measurable and satisfies $ \Ebb^\Pbb\|\psi(Z)\|_\Hscr^2<\infty.$ By the Riesz representation theorem, the payoff can be written as
\begin{equation}\label{eq:score_representer}
    \Rscr(h;z)
    =
    \langle h,\psi(z)\rangle_\Hscr,
    \qquad h\in\Hscr.
\end{equation}
The representer $\psi(z)$ summarizes the information in $z$ relevant for evaluating a candidate function $h$. Its form depends on the payoff construction and the chosen kernel. Restricting the payoff operator $J_\Pbb$ in \eqref{eq:payoff_operator} to $\Hscr$, we write
\begin{equation*}
    J_{\Pbb,\Hscr}\colon\Hscr\to L^2_\Pbb,
    \qquad
    J_{\Pbb,\Hscr}h
    \isdef
    J_\Pbb h
    =
    \Rscr(h;Z).
\end{equation*}
Its adjoint $J_{\Pbb,\Hscr}^\ast\colon L^2_\Pbb \to\Hscr$ satisfies
\begin{equation}\label{eq:payoff_adjoint}
    J_{\Pbb,\Hscr}^\ast U
    =
    \Ebb^\Pbb[U\psi(Z)],
    \qquad U\in L^2_\Pbb.
\end{equation}
These objects let us express the population problem in the same form as ordinary ridge regression. Define
\begin{equation}\label{eq:operator_definitions}
    \mu
    \isdef
    J_{\Pbb,\Hscr}^\ast Y
    =
    \Ebb^\Pbb[Y\psi(Z)],
    \qquad
    \Sig
    \isdef
    J_{\Pbb,\Hscr}^\ast J_{\Pbb,\Hscr}
    =
    \Ebb^\Pbb[\psi(Z)\otimes\psi(Z)].
\end{equation}
These are the population moments entering the normal equation.\footnote{In conventional finite-dimensional ridge regression with random regressors, the population moments entering the normal equation are $\mathbb E[XY]$ and $\mathbb E[XX^\top]$, where $X$ is the regressor vector and $Y$ the response. Here, the payoff representer $\psi(Z)$ plays the role of $X$, so $\mu$ and $\Sig$ are their Hilbert-space counterparts; see \citet{hsu14} for the random-design ridge regression and \citet{Caponnetto2006} and \citet{Nystrom2015} for the associated kernel learning framework.} The \emph{uncentered covariance operator} $\Sig$ satisfying 
\begin{equation*}
    \langle f, \Sig g \rangle_\Hscr = \Ebb^\Pbb\langle f, \psi(Z) \rangle_\Hscr \langle g, \psi(Z) \rangle_\Hscr = \int \Rscr(f;z) \Rscr(g;z) \d\Pbb(z),
\end{equation*}
also characterizes which directions in $\Hscr$ affect the generated payoff. For any $g\in\Hscr$,  
\begin{equation}\label{eq:payoff_norm_identity}
    \|J_{\Pbb,\Hscr}g\|_{L^2_\Pbb}^2
    =
    \langle g,\Sig g\rangle_\Hscr
    =
    \|\Sig^{1/2}g\|_\Hscr^2.
\end{equation}
Consequently, $g\in\operatorname{Ker}(\Sig)$ if and only if $\mathscr R(g;Z)=0$ almost surely. Thus, two functions that differ by an element of $\operatorname{Ker}(\Sig)$ generate the same payoff, even though their derivatives may differ. Restricting the candidate class to $\Hscr$ may also limit the payoffs available to approximate $Y$. We therefore define the \emph{RKHS-projected payoff} as
\begin{equation}\label{eq:restricted-payoff-space}
    \Rscr^\star_\Hscr
    \isdef
    \argmin_{\Rscr\in\overline{\operatorname{Im}(J_{\Pbb,\Hscr})}}
    \|Y-\Rscr\|_{L^2_\Pbb}^2.
\end{equation}
This projection may differ from the unrestricted payoff $\Rscr^\star$ defined in Section~\ref{subsec:population_payoff_projection}. Moreover, although $\Rscr^\star_\Hscr$ is uniquely defined in payoff space, it need not be generated by any function in $\Hscr$, and even when such a function exists, it need not be unique. To obtain a well-defined functional target for shape inference, we introduce a ridge penalty. For $\lambda>0$, consider the ridge problem
\begin{equation}\label{eq:population_problem}
    \underset{h\in\Hscr}{\operatorname{minimize}} \; 
        \frac12
        \Ebb^\Pbb\left[\bigl(Y-\Rscr(h;Z)\bigr)^2\right]
        +
        \frac{\lambda}{2}\|h\|_\Hscr^2=\frac12
        \|Y-J_{\Pbb,\Hscr}h\|_{L^2_\Pbb}^2
        +
        \frac{\lambda}{2}\|h\|_\Hscr^2.
\end{equation}
This is the population counterpart of kernel ridge regression, with the unique solution, which we call the \emph{regularized target}\footnote{The exact form in \eqref{eq:population_solution} is a standard result obtained by applying \citet[Theorem~5.1]{EnglHankeNeubauer1996} to Problem~\ref{eq:population_problem}.} given by 
\begin{equation}\label{eq:population_solution}
    h_\lambda
    \isdef
    \left(\Sig+\lambda I_\Hscr\right)^{-1}\mu.
\end{equation}
The ridge penalty ensures uniqueness and selects a function orthogonal to directions that do not affect the generated payoff. Indeed,
\begin{equation*}
    \mu\in\operatorname{Im}(J_{\Pbb,\Hscr}^\ast)
    \subseteq
    \overline{\operatorname{Im}(J_{\Pbb,\Hscr}^\ast)}
    =
    \operatorname{Ker}(\Sig)^\perp,
\end{equation*}
and $\Hscr$ admits the orthogonal decomposition $\Hscr
    =
    \operatorname{Ker}(\Sig)
    \oplus
    \operatorname{Ker}(\Sig)^\perp.$ Since $(\Sig+\lambda I_\Hscr)^{-1}$ preserves both subspaces, it follows that $h_\lambda\in\operatorname{Ker}(\Sig)^\perp$. The shape restrictions studied below are therefore formulated in terms of the values and derivatives of the uniquely determined regularized target $h_\lambda$, rather than those of an arbitrary function generating the same payoff.

\subsection{Regularization and approximation}
\label{subsec:reg_approx}
The regularized target $h_\lambda$ is uniquely defined, but the payoff it generates may still differ from the projected payoff $\mathscr R^\star$ in Section~\ref{subsec:population_payoff_projection} for two reasons. First, regularization introduces bias even within the payoff space $\operatorname{Im}(J_{\Pbb,\Hscr})$ generated by the functions in $\Hscr$. Second, restricting the candidate class from $\mathcal F$ to $\Hscr$ may limit the payoffs available for approximating $Y$, giving rise to a form of \emph{model misspecification error}. These effects are distinct: the former vanishes as $\lambda$ decreases, whereas the latter depends on the choice of function space.

\paragraph{Regularization bias.}
Consider first what happens within the payoff space generated by $\Hscr$. Proposition~\ref{prop:reg_consistency} below shows that the payoff generated by $h_\lambda$ approaches the RKHS-projected payoff $\Rscr^\star_\Hscr$ as $\lambda\downarrow0$. To obtain convergence of the function itself, however, $\Rscr^\star_\Hscr$ must actually be generated by some element of $\Hscr$. We say that $\Rscr^\star_\Hscr$ is \emph{attained} if there exists $h\in\Hscr$ such that
$J_{\Pbb,\Hscr}h=\Rscr^\star_\Hscr$. When this holds, there may still be several functions producing the same payoff. Among them, define the \emph{minimum-norm target}
\begin{equation}\label{eq:h0_definition}
    h_0
    \isdef
    \argmin\bigl\{
        \|h\|_\Hscr:
        h\in\Hscr,\quad
        J_{\Pbb,\Hscr}h=\Rscr^\star_\Hscr
    \bigr\}.
\end{equation}

\begin{proposition}[Tikhonov convergence and regularization rates]
\label{prop:reg_consistency}
As $\lambda\downarrow0$:
\begin{enumerate}
    \item
    $J_{\Pbb,\Hscr}h_\lambda\to\Rscr^\star_\Hscr$ in $L^2_\Pbb$.
    
    \item
    If the RKHS-projected payoff is attained, then
    $h_\lambda\to h_0$ in $\Hscr$.
    
    \item
    If, in addition, for some $0<\nu\leq\tfrac12$, the \emph{source condition}
    \begin{equation}\label{eq:source_condition}
        h_0\in\operatorname{Im}(\Sig^\nu)
    \end{equation}
    holds, then
    \begin{equation}\label{eq:source_rates}
        \|h_\lambda-h_0\|_\Hscr
        \lesssim
        \lambda^\nu,
        \qquad
        \|J_{\Pbb,\Hscr}h_\lambda-\Rscr^\star_\Hscr\|_{L^2_\Pbb}
        \lesssim
        \lambda^{\nu+1/2}.
    \end{equation}
\end{enumerate}
\end{proposition}
Without further regularity, the first two convergences need not occur at any uniform rate when $\operatorname{Im}(J_{\Pbb,\Hscr})$ is not closed; see \citet[Proposition~3.11]{EnglHankeNeubauer1996}. The source condition supplies the additional structure needed for quantitative rates. It requires $h_0$ to lie sufficiently strongly in the payoff-identified spectral directions of $\Sig$: larger $\nu$ corresponds to stronger alignment. For Tikhonov regularization, the payoff-norm rate in \eqref{eq:source_rates} reaches $O(\lambda)$ at $\nu=\tfrac12$, so larger $\nu$ does not improve that rate. The proof is given in Appendix~\ref{appendix:proofs_framework}.

\paragraph{RKHS approximation error.}
Letting $\lambda$ vanish removes regularization bias, but it does not undo the restriction from $\Fcal$ to $\Hscr$. The remaining error is $ \|\Rscr^\star_\Hscr-\Rscr^\star\|_{L^2_\Pbb},$ which is deterministic for fixed $\Hscr$, $\Fcal$, and $\Pbb$, and does not depend on $\lambda$. This approximation error vanishes precisely when $ \Rscr^\star\in\overline{\operatorname{Im}(J_{\Pbb,\Hscr})}.$ A stronger condition, independent of the particular target payoff $Y$, is
\begin{equation}\label{eq:density_equivalence}
    \overline{\operatorname{Im}(J_{\Pbb,\Hscr})}
    =
    \overline{\operatorname{Im}(J_\Pbb)}.
\end{equation}
For example, \eqref{eq:density_equivalence} holds when every payoff generated from $\Fcal$ can be approximated in $L^2_\Pbb$ by payoffs generated from $\Hscr$; the reverse inclusion already follows from $\Hscr\subseteq\Fcal$.

\section{Statistical learning}\label{sec:stat_prop}

Section~\ref{sec:framework} characterizes the population payoff and establishes how the RKHS restriction and regularization affect its approximation. We now examine how accurately this payoff can be recovered from a finite sample of serially dependent observations. We first show how the estimator can be computed, then study how closely its generated payoff approaches the population benchmark. Finally, we combine sampling error, regularization bias, and RKHS approximation error into a single bound that shows the roles of sample size and the regularization level explicitly.

\subsection{Empirical problem and finite-dimensional representation}
\label{subsec:empirical_problem}
Given observations $Z_1,\ldots,Z_T$, we write $Y_t=Y(Z_t)$ and $\psi_t=\psi(Z_t)$ for $1\leq t\leq T$. Consider the \emph{sample payoff operator} $\widehat J_T\colon \Hscr \to \Rbb^T$ defined as
\begin{equation}\label{eq:empirical_payoff_operator}
    \widehat J_T h \isdef  \frac1{\sqrt T}\left[\langle h,\psi_1 \rangle_\Hscr, \ldots, \langle h,\psi_T\rangle_\Hscr \right]^\top  =  \frac1{\sqrt T}\left[\Rscr(h;Z_1), \ldots, \Rscr(h;Z_T) \right]^\top, 
\end{equation}
whose adjoint $\widehat J_T^\ast\colon \Rbb^T\to\Hscr$ is
\begin{equation}\label{eq:empirical_payoff_adjoint}
    \widehat J_T^\ast\bs a
    =
    \frac1{\sqrt T}\sum_{t=1}^T a_t\psi_t,
    \qquad
    \bs a=(a_1,\ldots,a_T)^\top\in\Rbb^T.
\end{equation}
The normalized response vector is $ \bs y_T\isdef\frac{1}{\sqrt{T}}
    \big[
        Y_1,\ldots,Y_T
    \big]^\top \in \Rbb^T.$ The sample analogues of the population quantities $\mu$ and $\Sig$ in \eqref{eq:operator_definitions} are thus given by \footnote{In ordinary finite-dimensional ridge regression, the corresponding sample moments are $T^{-1}\mathbf X^\top\mathbf y$ and $T^{-1}\mathbf X^\top\mathbf X$, where $\mathbf X$ is the design matrix and $\mathbf y$ the response vector. Here, the payoff representers $\psi_t$ play the role of the regressors.}
\begin{equation}\label{eq:empirical_moments}
    \widehat\mu_T
    \isdef
    \widehat J_T^\ast\bs y_T
    =
    \frac1T\sum_{t=1}^T Y_t\psi_t,
    \qquad
    \widehat\Sig_T
    \isdef
    \widehat J_T^\ast\widehat J_T
    =
    \frac1T\sum_{t=1}^T\psi_t\otimes\psi_t.
\end{equation}
In particular, $\widehat\mu_T\in\Hscr$ and the empirical uncentered covariance operator $\widehat \Sig_T$ is positive, self-adjoint, and of finite rank\footnote{this follows from the property of rank-one operators.}, and satisfies 
\begin{equation*}
    \langle f,\widehat \Sig_T g\rangle_\Hscr =  \frac{1}{T}\sum_{t=1}^T \langle f, \psi_t \rangle_\Hscr \langle g,\psi_t\rangle_\Hscr =\frac{1}{T} \sum_{t=1}^T \Rscr(f;Z_t)\Rscr(g;Z_t).  
\end{equation*}
The empirical loss can be defined as the sample counterpart of \eqref{eq:pop-loss-payoff}:
\begin{equation}\label{eq:empirical_loss}
    \widehat\Ecal_T(h)
    \isdef  \frac{1}{2T}
    \sum_{t=1}^T
    \bigl(Y_t-\Rscr(h;Z_t)\bigr)^2 = 
    \frac12
    \|\bs y_T-\widehat J_T h\|_2^2.
\end{equation}
The empirical regularized problem is therefore
\begin{equation}\label{eq:empirical_problem}
    \underset{h\in\Hscr}{\operatorname{minimize}} \; 
        \widehat\Ecal_T(h)
        +\frac{\lambda}{2}\|h\|_\Hscr^2
    =\frac12
        \|\bs y_T-\widehat J_T h\|_{2}^2
        +\frac{\lambda}{2}\|h\|_\Hscr^2.
\end{equation}
By \citet[Theorem~5.1]{EnglHankeNeubauer1996}, this problem has the unique solution
\begin{equation}\label{eq:empirical_solution}
    \widehat h_{\lambda,T}
    =
    \left(
        \widehat J_T^\ast\widehat J_T
        +\lambda I_\Hscr
    \right)^{-1}
    \widehat J_T^\ast\bs y_T
    =
    \left(
        \widehat\Sig_T+\lambda I_\Hscr
    \right)^{-1}
    \widehat\mu_T.
\end{equation}

\paragraph{Finite-dimensional representation.}
Although the optimization takes place in the possibly infinite-dimensional space $\Hscr$, the estimator belongs to the span of the observed payoff representers. Define this finite-dimensional subspace as
\begin{equation*}
    \Hscr_T
    \isdef
    \spn\{\psi_1,\ldots,\psi_T\}.
\end{equation*}
Let $\bs K_T$ denote the Gram matrix of these representers, and let $\bs y$ denote the unscaled response vector
\begin{equation}\label{eq:empirical_payoff_gram}
    \bs K_T
    \isdef
    \bigl[
        \langle\psi_t,\psi_s\rangle_\Hscr
    \bigr]_{t,s=1}^T
    \in\Rbb^{T\times T},
    \qquad
    \bs y
    \isdef
    \sqrt{T}\bs y_T
    =
    (Y_1,\ldots,Y_T)^\top \in\Rbb^T.
\end{equation}
Then
\begin{equation*}
    \widehat h_{\lambda,T}
    =
    \sum_{t=1}^T
    \widehat b_{\lambda, t}\psi_t
    \in\Hscr_T,
    \qquad
    \widehat{\bs b}_{\lambda,T}
    =
    (\bs K_T+T\lambda\bs I_T)^{-1}\bs y.
\end{equation*}
Equivalently,
\begin{equation}\label{eq:coefficient_quadratic_problem}
    \widehat{\bs b}_{\lambda,T}
    =
    \argmin_{\bs b\in\Rbb^T}
    \left\{
        \frac12
        \bs b^\top
        (\bs K_T+T\lambda\bs I_T)
        \bs b
        -
        \bs y^\top\bs b
    \right\}.
\end{equation}
This is the familiar kernel ridge system, with a Gram matrix that now incorporates the structure of the observed payoffs rather than direct evaluations of the unknown function. In Example~\ref{ex:characteristics}, this structure arises from the aggregation of firm-level expected returns into portfolio returns, whereas in Example~\ref{ex:option-span}, it arises from the derivative evaluations that determine forward and option positions. Although the economic mechanisms generating these payoffs differ, both lead to the same computational formulation. The derivation of \eqref{eq:coefficient_quadratic_problem} is given in Appendix~\ref{appendix:proofs_stat_prop}.

\subsection{Estimation error}\label{subsec:finite_sample_learning}
The question now is how closely a payoff learned from a finite history of observations approximates the payoff generated by the population regularized function. To address this, our finite-sample analysis explicitly accounts for their temporal dependence. We first define the pair process
\begin{equation}\label{eq:psi_t}
    \zeta_t
    \isdef
    (\psi_t,Y_t)
    \in
    \Hscr\oplus\Rbb.
\end{equation}

\begin{assumption}[Almost-sure boundedness]
\label{assumption:boundedness}
There exist constants $\kappa,\kappa_Y<\infty$ such that
\begin{equation*}
    \|\psi_t\|_\Hscr\leq\kappa,
    \qquad
    |Y_t|\leq\kappa_Y
    \qquad
    \text{almost surely for every $t$}.
\end{equation*}
\end{assumption}

\begin{assumption}[Strict stationarity]
\label{assumption:stationarity}
The process $(\zeta_t)_{t\in\Zbb}$ is strictly stationary.
\end{assumption}

\begin{assumption}[Polynomial $\tau$-mixing]
\label{assumption:tau_mixing}
There exist constants $c_\tau>0$ and $\gamma_\tau>0$ such that
\begin{equation}\label{eq:tau_mixing_assumption}
    \tau_\zeta(m)
    \leq
    c_\tau m^{-\gamma_\tau},
    \qquad
    m\in\Nbb.
\end{equation}
\end{assumption}
\noindent Assumptions~\ref{assumption:boundedness}--\ref{assumption:tau_mixing}
are maintained throughout this section and are familiar in statistical learning with dependent observations. Appendix~\ref{subsec:dependence_measures} provides the relevant definitions and references.\footnote{These assumptions are strictly for the finite-sample learning results in this section. Section~\ref{sec:inference} states separately the conditions required for the asymptotic shape inference.} Under strict stationarity,
the population quantities from \eqref{eq:operator_definitions} satisfy
\begin{equation}\label{eq:stationary_population_moments}
    \mu
    =
    \Ebb^\Pbb[Y_t\psi_t],
    \qquad
    \Sig
    =
    \Ebb^\Pbb[\psi_t\otimes\psi_t],
\end{equation}
independently of $t$. In particular, $\mu=\Ebb^\Pbb[\psi_t]$ when $Y_t\equiv1$. For $\lambda>0$, set 
\begin{equation*}
    \Sig_\lambda\isdef\Sig+\lambda I_\Hscr, \qquad \widehat\Sig_{\lambda,T}
\isdef\widehat\Sig_T+\lambda I_\Hscr.
\end{equation*}
Then, the unique optimizers to Problems~\ref{eq:population_problem} and \ref{eq:empirical_problem} can be written, respectively, as 
\begin{equation}\label{eq:solution_characterizations}
h_\lambda= \Sig_\lambda^{-1}\mu,\qquad \widehat h_{\lambda,T}
= \widehat\Sig_{\lambda,T}^{-1}\widehat\mu_T.
\end{equation}
The estimation error studied in this subsection is the payoff-space error between the estimator and the regularized target, $\|J_{\Pbb, \Hscr}\widehat h_{\lambda,T} - J_{\Pbb, \Hscr} h_\lambda \|_{L^2_\Pbb}$. By the payoff-operator identity in \eqref{eq:payoff_norm_identity}, we obtain
\begin{equation}\label{eq:estimation_error}
    \|\Rscr(\widehat h_{\lambda,T};Z)-\Rscr(h_\lambda;Z)\|_{L^2_\Pbb}=\|J_{\Pbb,\Hscr}\widehat h_{\lambda,T}-J_{\Pbb,\Hscr}h_\lambda\|_{L^2_\Pbb}=\|\Sig^{1/2}(\widehat h_{\lambda,T}-h_\lambda)\|_\Hscr.
\end{equation}
Define the \emph{effective dimension} associated with $\Sig$ as
\begin{equation}\label{eq:effective_dimension}
    \Ncal(\lambda)
    \isdef
    \Ebb^\Pbb\|\Sig_\lambda^{-1/2}\psi_t\|_\Hscr^2
    =
    \tr\!\left\{\Sig(\Sig+\lambda I_\Hscr)^{-1}\right\}.
\end{equation}
The effective dimension $\Ncal(\lambda)$ measures, in spectral terms, the directions in $\Hscr$ that remain statistically relevant at regularization level  $\lambda$. In standard kernel learning, it is interpreted as a scale-dependent spectral measure of the \emph{capacity} of the hypothesis class $\Hscr$ relative to the input distribution $\Pbb$. 

We derive the finite-sample bound using concentration arguments for dependent observations following \citet{BlanchardZadorozhnyi2019}; see Appendix~\ref{appendix:proofs_stat_prop} for more details. Under the boundedness assumption, the population residual satisfies
\begin{equation*}
    \left|Y_t-\langle h_\lambda,\psi_t\rangle_\Hscr\right| \leq \kappa_Y + \kappa \|h_\lambda\|_\Hscr =: R_\lambda < \infty.
\end{equation*}
Enlarging $\kappa_Y$ if necessary, which preserves Assumption~\ref{assumption:boundedness}, we take $\kappa_Y>0$, so that $R_\lambda\geq\kappa_Y>0$. 
In what follows, we also assume $\Sig \neq 0$, so that $\Ncal(\lambda)>0$ and $\kappa>0$, and the quantities below are well defined and positive.\footnote{For the degenerate case $\Sig=0$, $\psi(Z)=0$ almost surely, so the estimation error is identically zero.} 
To account for serial dependence, define the \emph{effective sample-size bound}
\begin{equation}\label{eq:process_effective_sample_sizes}
    \ell^\star(\lambda,T)
    \isdef
    \min\left\{
    \left\lfloor
    \left(
    \frac{R_\lambda\sqrt{\lambda\Ncal(\lambda)}}
    {(\kappa+2R_\lambda)c_\tau}
    \right)^{\frac{2}{2\gamma_\tau+1}}
    \left(\frac{T}{2}\right)^{\frac{2\gamma_\tau}{2\gamma_\tau+1}}
    \right\rfloor,\;
    T
    \right\},
\end{equation}
For any $\lambda>0$, $\ell^\star(\lambda,T)$ grows at a rate $T^{2\gamma_\tau/(2\gamma_\tau+1)}$. The exponent approaches one as $\gamma_\tau$ increases. For $\delta \in (0,1)$, define 
\begin{equation}\label{eq:effective_sample_size_thresholds}
    \ell_0(\lambda,\delta)
    \isdef
    \frac{900\kappa^2}{\lambda}
    \max\{\Ncal(\lambda),1\}
    \log^2\!\left(\frac4\delta\right),
\end{equation}
and the corresponding \emph{sample size threshold},
\begin{equation}\label{eq:T_0}
    T_0(\lambda,\delta)
    \isdef
    \max\left\{
    \left\lceil 2\left\lceil\ell_0(\lambda,\delta)\right\rceil^{\frac{2\gamma_\tau+1}{2\gamma_\tau}}
    \left(
    \frac{(\kappa+2R_\lambda)c_\tau}
    {R_\lambda\sqrt{\lambda\Ncal(\lambda)}}
    \right)^{\frac{1}{\gamma_\tau}}
    \right\rceil,\;
    \left\lceil\ell_0(\lambda,\delta)\right\rceil
    \right\}.
\end{equation}
The quantity $\ell_0(\lambda,\delta)$ specifies the effective sample size sufficient for the empirical regularized covariance to be a stable perturbation of $\Sig_\lambda$ with high probability. The threshold $T_0(\lambda,\delta)$ gives a sufficient sample-size condition for the payoff bound in Theorem~\ref{thm:generic_finite_sample_rkhs_projection}.

\begin{theorem}[Finite-sample payoff estimation error]\label{thm:generic_finite_sample_rkhs_projection}
Under Assumptions~\ref{assumption:boundedness}--\ref{assumption:tau_mixing}, for any $\lambda>0$, $\delta\in(0,1)$, and $T\geq T_0(\lambda,\delta)$, the following bound holds with probability at least $1-\delta$.
\begin{equation}\label{eq:payoff_norm_bound}
    \|J_{\Pbb,\Hscr}\widehat h_{\lambda,T}
    -J_{\Pbb,\Hscr}h_\lambda\|_{L^2_\Pbb}
    \leq
    2R_\lambda\log\!\left(\frac4\delta\right)
    \left[
    13\sqrt{\frac{\Ncal(\lambda)}{\ell^\star(\lambda,T)}}
    +
    \frac{42\kappa}{\sqrt\lambda\,\ell^\star(\lambda,T)}
    \right].
\end{equation}
\end{theorem}
\noindent 
The leading term depends on the effective dimension relative to the effective sample size, while the second term decays faster as the effective sample size increases. The proof in Appendix~\ref{appendix:proofs_stat_prop} also establishes the same bound for the weighted error $\|\Sig_\lambda^{1/2}(\widehat h_{\lambda,T}-h_\lambda)\|_{\Hscr}$. Since $\Sig_\lambda\geq\lambda I_{\Hscr}$, this yields the following $\Hscr$-norm bound.
\begin{corollary}[$\Hscr$-norm estimation bound]\label{cor:h_norm_estimation_bound}
Under the conditions of Theorem~\ref{thm:generic_finite_sample_rkhs_projection}, the following bound holds with probability at least $1-\delta$.
\begin{equation*}
    \|\widehat h_{\lambda,T}-h_\lambda\|_\Hscr
    \leq
    \frac{2R_\lambda}{\sqrt\lambda}
    \log\!\left(\frac4\delta\right)
    \left[
    13\sqrt{\frac{\Ncal(\lambda)}{\ell^\star(\lambda,T)}}
    +
    \frac{42\kappa}{\sqrt\lambda\,\ell^\star(\lambda,T)}
    \right].
\end{equation*}
\end{corollary}
\noindent We can obtain explicit rates for the above bounds by considering the decay of the eigenvalues of $\Sig$. Suitable Sobolev-type and tensor-product constructions exhibit polynomial eigenvalue decay, often with logarithmic corrections \citep{Kuhn2015,Krieg2018}. As the spectrum of $\Sig$ depends on both the RKHS and the payoff operator, we state the decay condition directly in terms of its eigenvalues. 
\begin{equation}\label{eq:spectral_decay} 
\lambda_m(\Sig) \lesssim m^{-p}(\log m)^q \qquad\text{for all sufficiently large }m, \qquad p>1,\quad q\geq0,
\end{equation}
The following proposition establishes the resulting bound on the effective dimension; its proof is given in Appendix~\ref{appendix:proofs_stat_prop}.
\begin{proposition}[Effective dimension under log-polynomial decay]\label{prop:effdim_conversion} 
Under \eqref{eq:spectral_decay}, there exists $\lambda_0\in(0,1)$ such that, for every $\lambda\in(0,\lambda_0]$, 
\begin{equation*}
\Ncal(\lambda) \lesssim \lambda^{-1/p} \bigl(\log(1/\lambda)\bigr)^{q/p}. 
\end{equation*} 
\end{proposition} 
\noindent Combining Proposition~\ref{prop:effdim_conversion} with
Theorem~\ref{thm:generic_finite_sample_rkhs_projection} gives the following finite-sample bound under log-polynomial spectral decay.
\begin{corollary}[Specialization under log-polynomial spectral decay]\label{cor:finite_sample_spectral} 
Suppose Assumptions~\ref{assumption:boundedness}--\ref{assumption:tau_mixing} and \eqref{eq:spectral_decay} hold. There exists $\lambda_0\in(0,1)$ such that, for every $\lambda\in(0,\lambda_0]$ and $\delta\in(0,1)$ satisfying $T\geq T_0(\lambda,\delta)$, the following inequality holds with probability at least $1-\delta$. \begin{equation}\label{eq:finite_sample_spectral_bound} 
\begin{aligned} 
\|J_{\Pbb,\Hscr}\widehat h_{\lambda,T} -J_{\Pbb,\Hscr}h_\lambda\|_{L^2_\Pbb} \lesssim{}& R_\lambda\log\!\left(\frac4\delta\right) \left[ \frac{ \lambda^{-1/(2p)} \bigl(\log(1/\lambda)\bigr)^{q/(2p)} }{ \sqrt{\ell^\star(\lambda,T)} } + \frac{ \kappa }{ \sqrt{\lambda}\,\ell^\star(\lambda,T) } \right]. 
\end{aligned} 
\end{equation} 
\end{corollary} 
\noindent Larger $p$ implies faster spectral decay and a smaller effective dimension; $q$ contributes to the logarithmic correction. The implicit constants in \eqref{eq:finite_sample_spectral_bound} depend only on the constants in \eqref{eq:spectral_decay}.

\subsection{Payoff error bound}\label{subsec:combined_bound}

The {total payoff error} $ \|\Rscr(\widehat h_{\lambda,T};Z)-\Rscr^\star\|_{L^2_\Pbb}$ admits the following decomposition by the triangle inequality:  
\begin{equation}\label{eq:total_payoff_error_decomposition}
\begin{aligned}
    \|\Rscr(\widehat h_{\lambda,T};Z)-\Rscr^\star\|_{L^2_\Pbb}
    &=\|J_{\Pbb,\Hscr}\widehat h_{\lambda,T}-\Rscr^\star\|_{L^2_\Pbb}\\
    &\leq\underbrace{\|J_{\Pbb,\Hscr}\widehat h_{\lambda,T}-J_{\Pbb,\Hscr}h_\lambda\|_{L^2_\Pbb}}_{\text{sampling error}}\\
    &\quad+\underbrace{\|J_{\Pbb,\Hscr}h_\lambda-\Rscr^\star_\Hscr\|_{L^2_\Pbb}}_{\text{regularization bias}}
    +\underbrace{\|\Rscr^\star_\Hscr-\Rscr^\star\|_{L^2_\Pbb}}_{\text{RKHS approximation error}}.
\end{aligned}
\end{equation}
The following result combines the finite-sample bound with the two deterministic error terms.

\begin{theorem}[Total payoff error]
\label{thm:combined_payoff_bound}
Under Assumptions~\ref{assumption:boundedness}--\ref{assumption:tau_mixing}, for any $\lambda>0$, $\delta\in(0,1)$, and 
$T\geq T_0(\lambda,\delta)$, the following statements hold. 
\begin{enumerate}
\item With probability at least $1-\delta$,
\begin{equation}\label{eq:combined_bound}
\begin{aligned}
    \|\Rscr(\widehat h_{\lambda,T};Z)-\Rscr^\star\|_{L^2_\Pbb}
    \leq{}&
    2R_\lambda\log\!\left(\frac4\delta\right)
    \left[
    13\sqrt{\frac{\Ncal(\lambda)}
    {\ell^\star(\lambda,T)}}
    +
    \frac{42\kappa}
    {\sqrt\lambda\,\ell^\star(\lambda,T)}
    \right]
    \\
    &+
    \|J_{\Pbb,\Hscr}h_\lambda
    -\Rscr^\star_\Hscr\|_{L^2_\Pbb}
    +
    \|\Rscr^\star_\Hscr
    -\Rscr^\star\|_{L^2_\Pbb}.
\end{aligned}
\end{equation}

\item If, in addition, $\Rscr^\star_\Hscr$ is attained and the source condition
\eqref{eq:source_condition} holds for some
$0<\nu\leq\tfrac12$, then, with probability at least $1-\delta$,
\begin{equation}\label{eq:combined_source_bound}
\begin{aligned}
    \|J_{\Pbb,\Hscr}\widehat h_{\lambda,T}
    -\Rscr^\star\|_{L^2_\Pbb}
    \lesssim{}&
    R_\lambda\log\!\left(\frac4\delta\right)
    \left[
    \sqrt{\frac{\Ncal(\lambda)}
    {\ell^\star(\lambda,T)}}
    +
    \frac{\kappa}
    {\sqrt\lambda\,\ell^\star(\lambda,T)}
    \right]
    \\
    &+
    \lambda^{\nu+1/2}
    +
    \|\Rscr^\star_\Hscr
    -\Rscr^\star\|_{L^2_\Pbb}.
\end{aligned}
\end{equation}
\end{enumerate}
\end{theorem}

\begin{proof}
Part~1 follows by substituting the finite-sample bound
\eqref{eq:payoff_norm_bound} into
\eqref{eq:total_payoff_error_decomposition}. Under the additional conditions
in Part~2, the regularization rate
\eqref{eq:source_rates} gives $  \|J_{\Pbb,\Hscr}h_\lambda
    -\Rscr^\star_\Hscr\|_{L^2_\Pbb}
    \lesssim
    \lambda^{\nu+1/2},$ which proves the second bound.
\end{proof}
\noindent If \eqref{eq:density_equivalence} holds, the RKHS approximation error is
zero, and the final term in \eqref{eq:combined_bound} and
\eqref{eq:combined_source_bound} vanishes.

\section{Statistical inference for shape properties}
\label{sec:inference}

The previous section establishes a finite-sample bound on the error in the payoff generated by the regularized RKHS estimator. It describes estimation accuracy but does not establish whether the regularized target $h_\lambda$ satisfies the shape restrictions suggested by financial theory. We now develop statistical tests for such restrictions. Recall furthermore that no shape restrictions were imposed during estimation, leaving the data free to reveal departures from the theoretical predictions.

When functions are sufficiently smooth, many shape constraints can be expressed as sign conditions on suitable derivatives. Positivity corresponds to nonnegative function values, whereas monotonicity and convexity translate, respectively, into sign restrictions on the first and second derivatives with respect to a scalar argument. For multivariate functions, comparable restrictions can be formulated along selected coordinates.\footnote{In the empirical applications, we only consider derivatives with respect to one coordinate at a time, such as market beta, a firm characteristic, or Treasury maturity, holding the remaining coordinates fixed. We do not test mixed partial derivatives.}
This approach is especially convenient in smooth RKHSs, including appropriate Sobolev spaces. Using the derivative reproducing property from Theorem~\ref{result:zhou} in Appendix~\ref{sec:prelims}, derivatives of the fitted function can be computed directly from its kernel representation using the same coefficients obtained in the learning stage, without separate derivative estimation or extra smoothing choices.

A shape constraint on a continuous domain typically implies an infinite collection of inequalities. To make joint inference feasible, we instead evaluate the relevant derivatives on a predetermined finite grid. Let $\mathcal Z=\{\boldsymbol\xi_1,\ldots,\boldsymbol\xi_n\}\subset\mathcal X$ denote this grid, with $n$ held fixed in the asymptotic analysis. For a chosen derivative multi-index $\boldsymbol\alpha$ and sign $\varsigma\in\{-1,1\}$, define the signed derivative vector of the regularized target by
\begin{equation}\label{eq:shape-null-generic}
   \boldsymbol\theta_{\lambda,n}
=
\varsigma\bigl(
D^{\boldsymbol\alpha}h_\lambda(\boldsymbol\xi_1),
\ldots,
D^{\boldsymbol\alpha}h_\lambda(\boldsymbol\xi_n)
\bigr)^\top \in \Rbb^n.
\end{equation}
The corresponding shape-testing problem is then 
\begin{equation}\label{eq:hypothesis_test}
\begin{aligned}
    H_0&:\bs\theta_{\lambda,n}\in\Rbb_+^n,
    \qquad\text{versus}\qquad
    H_1:\bs\theta_{\lambda,n}\notin\Rbb_+^n.
\end{aligned}
\end{equation}
In Example~\ref{ex:characteristics}, a hypothesized monotonic relation between expected returns and a firm characteristic becomes a sign condition on the corresponding first derivative along the ordered characteristic. Likewise, in Example~\ref{ex:option-span}, the first and second derivatives of the terminal payoff determine, respectively, the positions in the underlying asset and in the option strip. Monotonicity and convexity can therefore be tested via separate applications of the same finite-grid restriction. The grid may be chosen sufficiently fine to capture economically relevant parts of the domain. The test seeks evidence that at least one signed derivative violates its prescribed restriction. Rejection indicates a statistically significant violation somewhere on the testing grid, whereas non-rejection does not establish that the restriction holds. Moreover, even a sufficiently fine grid cannot guarantee that the function satisfies the restriction throughout the continuous domain.

Testing these restrictions requires the joint sampling distribution of the estimated derivative vector, accounting for dependence across grid points. We derive this distribution from the asymptotic behavior of the regularized estimator and use it to construct a feasible one-sided Wald test. The asymptotic results apply to any $\lambda>0$ as $T\to\infty$. The relationship between these finite-grid conclusions and the minimum-norm target $h_0$ is examined at the end of the section.

\subsection{Asymptotic distribution of the estimator}
\label{subsec:asymptotics}

\paragraph{Assumptions for asymptotic inference.}
In this section, we retain Assumption~\ref{assumption:boundedness}  and Assumption~\ref{assumption:stationarity} from Section~\ref{sec:stat_prop}. For asymptotic inference,
however, we replace Assumption~\ref{assumption:tau_mixing} on polynomial $\tau$-mixing with the
\emph{strong-mixing condition}, which permits us to establish the Gaussian
limit of the estimator and consistently estimate its long-run covariance.

\begingroup
\renewcommand{\theassumption}{3\ensuremath{^{\prime}}}
\renewcommand{\theHassumption}{3prime}
\begin{assumption}[Polynomial $\alpha$-mixing]
\label{assumption:strong_mixing}
There exist constants $c_\alpha>0$ and $\gamma_\alpha>2$ such that
\begin{equation*}
    \alpha_\zeta(m)
    \leq
    c_\alpha m^{-\gamma_\alpha},
    \qquad
    m\in\Nbb.
\end{equation*}
\end{assumption}
\endgroup
\setcounter{assumption}{3}

\noindent The dependence conditions for shape inference and finite-sample learning differ because the two analyses rely on distinct probabilistic results.\footnote{The finite-sample learning bound uses the concentration
inequality of \citet[Corollary~3.7]{BlanchardZadorozhnyi2019}
under polynomial $\tau$-mixing. For asymptotic inference,
the Hilbert-space central limit theorem follows from
\citet[Corollary~1 and Theorem~4]{Merlevede1997}, 
while consistency of the long-run covariance estimator
uses \citet[Theorem~2]{deJong2000}, adapted from \citet[Theorem~1]{Hansen1992}. The latter requires
$\gamma_\alpha(1/2-2/p)>1$ for some $p>4$, which is
feasible when $\gamma_\alpha>2$; the central limit theorem
alone permits weaker polynomial decay. The conditions are
verified in Appendices~\ref{appendix:proofs_asymptotic_props}
and~\ref{appendix:proofs_finite_grid}.}
The asymptotic properties of the empirical estimator $\widehat h_{\lambda,T}$ are determined by
sampling variation in the empirical payoff moments
$\widehat\mu_T$ and $\widehat\Sig_T$. Since the estimator depends on both through the
regularized normal equation, we must
consider their joint distribution. The first moment takes values in
$\Hscr$, whereas the second is a Hilbert--Schmidt operator on
$\Hscr$. Accordingly, we consider the Hilbert space direct sum 
\begin{equation}\label{eq:Hbb}
    \Hbb
    \isdef
    \Hscr\oplus\mathsf{HS}(\Hscr),
\end{equation}
with inner product
\begin{equation*}
    \left\langle(h,A),(g,B)\right\rangle_{\Hbb}
    \isdef
    \langle h,g\rangle_{\Hscr}
    +
    \langle A,B\rangle_{\HS},
    \qquad
    (h,A), \, (g,B)\in\Hbb.
\end{equation*}
Since we work in a separable RKHS $\Hscr$, $\Hbb$ is itself a
separable Hilbert space. Centering the single-period
contributions to the two population moments, 
\begin{equation*}
\begin{aligned}
    \widetilde\psi_t
    \isdef
    Y_t\psi_t-\mu,
    \qquad
    \widetilde\Sig_t
    \isdef
    \psi_t\otimes\psi_t-\Sig,
\end{aligned}
\end{equation*}
define the centered process
\begin{equation}\label{eq:zeta_tilde}
    \widetilde\zeta_t
    \isdef
    (\widetilde\psi_t,\widetilde\Sig_t)
    \in\Hbb.
\end{equation}
The centered sequence retains the stationarity, mixing, and
boundedness conditions required to apply the Hilbert-space central limit theorem.

\begin{proposition}[Stationarity, mixing, and moments of the centered process]
\label{proposition:well_defined}
Under Assumptions~\ref{assumption:boundedness},
\ref{assumption:stationarity}, and
\ref{assumption:strong_mixing}, the process
$(\widetilde\zeta_t)_{t\in\Zbb}$ is well-defined as an $\Hbb$-valued
sequence and satisfies the following. 
\begin{enumerate}
    \item It is strictly stationary.
    \item It is strongly mixing with coefficients satisfying
    $\alpha_{\widetilde\zeta}(m)
        \le
        \alpha_\zeta(m)$ for all $m\in\Nbb$.
    \item It is almost surely bounded in $\Hbb$:
    $\|\widetilde\zeta_t\|_{\Hbb}
        \le
        \kappa_{\widetilde\zeta}$ a.s.\
    for some $\kappa_{\widetilde\zeta}>0$.
    \item Hence,
    $\Ebb^{\Pbb}
        \|\widetilde\zeta_1\|_{\Hbb}^{2+\delta}
        \le
        \kappa_{\widetilde\zeta}^{2+\delta}
        <\infty$
    for every finite $\delta>0$.
\end{enumerate}
\end{proposition}

\noindent Define the sample mean of the centered process:
\begin{equation}\label{eq:sample_mean}
    \Bar S_T
    \isdef
    \frac1T\sum_{t=1}^T\widetilde\zeta_t
    =
    \bigl(
        \widehat\mu_T-\mu,\,
        \widehat\Sig_T-\Sig
    \bigr).
\end{equation}
These deviations affect the estimator jointly through the normal equation. To express their combined effect, we introduce the bounded linear map
\begin{equation}\label{eq:F}
    F_\lambda\colon\Hbb\to\Hscr,
    \qquad
    F_\lambda(h,A)
    \isdef
    h-Ah_\lambda.
\end{equation}
With this notation, we can write the estimation error exactly as
\begin{equation}\label{eq:difference_equivalent}
    \widehat h_{\lambda,T}-h_\lambda
    =
    \widehat\Sig_{\lambda,T}^{-1} \, 
    F_\lambda(\Bar S_T).
\end{equation}
This representation isolates the sampling fluctuation affecting the estimator while preserving
the dependence between the first and second empirical moments. Under the strong mixing assumption, the combined moment error $F_\lambda(\Bar S_T)$ admits the following Gaussian limit.
\begin{proposition}[Asymptotic Gaussianity]
\label{proposition:asymptotic_gauss}
For any $\lambda>0$, under the conditions of
Proposition~\ref{proposition:well_defined},
\begin{equation*}
    \sqrt T\,F_\lambda(\Bar S_T)
    \overset{d}{\longrightarrow}
    \Ncal_{\Hscr}(0,\Qcal_\lambda)
    \qquad\text{as} \; T\to\infty,
\end{equation*}
where
\begin{equation}\label{eq:Q_lambda}
    \Qcal_\lambda
    \isdef
    \sum_{\ell\in\Zbb}
    \Ebb^{\Pbb}
    \left[
        F_\lambda(\widetilde\zeta_0)
        \otimes
        F_\lambda(\widetilde\zeta_\ell)
    \right].
\end{equation}
\end{proposition}

\noindent The covariance operator $\Qcal_\lambda$ accounts for both contemporaneous and serial dependence in the combined moment errors through their long-run covariance. The Hilbert-space central limit theorem applied to the centered moment process also implies
\begin{equation}\label{eq:SN_rate}
    \|\Bar S_T\|_{\Hbb}
    =
    O_p(T^{-1/2}).
\end{equation}
As the empirical regularized inverse converges to its population
counterpart, the Gaussian limit transfers from the moment errors
to the estimator itself. We obtain the following consistency
and asymptotic normality results.

\begin{theorem}[Asymptotic properties of the sample estimator]
\label{thm:asymptotic_properties}
Under the conditions of
Proposition~\ref{proposition:asymptotic_gauss}, the following hold as
$T\to\infty$:
\begin{enumerate}
    \item
    $\widehat h_{\lambda,T}
        \overset{a.s.}{\longrightarrow}
        h_\lambda.$

    \item
    $\sqrt T
        \bigl(
            \widehat h_{\lambda,T}-h_\lambda
        \bigr)
        \overset{d}{\longrightarrow}
        \Ncal_{\Hscr}(0,\Vcal_\lambda),$
    where
    \begin{equation}\label{eq:V_lambda}
        \Vcal_\lambda
        \isdef
        \Sig_\lambda^{-1}
        \Qcal_\lambda
        \Sig_\lambda^{-1}.
    \end{equation}
\end{enumerate}
\end{theorem}

\noindent  Proofs of the probabilistic and operator-convergence statements, including the justification of the long-run covariance representation, are given in Appendix~\ref{appendix:estimator_prop}.

\subsection{Finite-grid derivative restrictions}
\label{subsec:grid_inference}
The preceding subsection establishes the asymptotic distribution of $\widehat h_{\lambda,T}$ in $\Hscr$. Derivative evaluations on the fixed grid $\Zcal = \{\bs \xi_1,\ldots,\bs\xi_n\}$ form a bounded linear transformation of the fitted function, so the Hilbert-space Gaussian limit transfers directly to the vector of signed derivatives.

\paragraph{Finite-grid Gaussian limit.}
By the derivative reproducing property,\footnote{The defining property of an RKHS $\Hscr$ is that point evaluation is a bounded linear functional and hence admits the reproducing representation $h(x)=\langle h,\phi(x)\rangle_{\Hscr}$. Here, $\phi\colon\Xcal\to\Hscr$ is the canonical feature map, which associates each $x\in\Xcal$ with the Riesz representer $\phi(x)=\Kcal(x, \cdot)$ of point evaluation. Under additional smoothness conditions, as established in Theorem~\ref{result:zhou} of Appendix~\ref{sec:prelims}, this property extends to derivative evaluations: $D^{\bs\alpha}h(x)=\langle h,\phi^{(\bs\alpha)}(x)\rangle_{\Hscr}$. Sufficiently smooth Sobolev RKHSs provide an important class of examples.}
the derivative evaluations on the fixed grid can be represented through the following bounded linear operator:
\begin{equation}\label{eq:second_derivative_evaluation}
    \Dscr_{n,\bs\alpha}\colon\Hscr\to\Rbb^n,
    \qquad
    \Dscr_{n,\bs\alpha}h
    \isdef
    \Bigl[
        \langle h,\phi^{(\bs\alpha)}(\bs\xi_j)\rangle_{\Hscr}
    \Bigr]_{j=1}^n
    =
    \Bigl[
        h^{(\bs\alpha)}(\bs\xi_j)
    \Bigr]_{j=1}^n,
\end{equation}
whose adjoint maps a vector of grid weights into the corresponding
linear combination of derivative representers
\begin{equation*}
    \Dscr_{n,\bs\alpha}^{\ast}\colon\Rbb^n\to\Hscr,
    \qquad
    \Dscr_{n,\bs\alpha}^{\ast}(\bs\omega)
    \isdef
    \sum_{j=1}^n
    \omega_j\phi^{(\bs\alpha)}(\bs\xi_j).
\end{equation*}
The population and estimated signed derivative vectors introduced
above can therefore be written as
\begin{equation}\label{eq:theta}
    \bs\theta_{\lambda,n}
    \isdef
    \varsigma\,
    \Dscr_{n,\bs\alpha}h_\lambda,
    \qquad
    \widehat{\bs\theta}_{\lambda,n,T}
    \isdef
    \varsigma\,
    \Dscr_{n,\bs\alpha}\widehat h_{\lambda,T}.
\end{equation}
The sign $\varsigma$ determines the direction of the shape
restriction but does not affect the covariance matrix of the
derivative estimates. To characterize this covariance, we define the population and empirical regularized inverse applied to each derivative representer. The regularized inverse maps
the derivative representer at a grid point into the direction
through which these errors influence its estimate. Accordingly,
for $1\le j\le n$, define
\begin{equation}\label{eq:u}
    u_j
    \isdef
    \Sig_\lambda^{-1}\phi^{(\bs\alpha)}(\bs\xi_j),
    \qquad
    \widehat u_j
    \isdef
    \widehat\Sig_{\lambda,T}^{-1}
    \phi^{(\bs\alpha)}(\bs\xi_j).
\end{equation}
Combining these directions with the centered moment fluctuations
$F_\lambda(\widetilde\zeta_t)$ defined in \eqref{eq:F}
gives the population grid-score vector for each date $t=1,\ldots,T$ as
\begin{equation}\label{eq:gridscore_main}
    \bs g_t
    \isdef
    \Dscr_{n,\bs\alpha}
    \Sig_\lambda^{-1}F_\lambda(\wt\zeta_t)
    =
    \Bigl[
        \langle F_\lambda(\wt\zeta_t),u_j\rangle_{\Hscr}
    \Bigr]_{j=1}^n
    \in\Rbb^n.
\end{equation}
Each component of this vector captures the first-order contribution of the date-\(t\) observation. The grid-score process is centered and inherits
strict stationarity and strong mixing from the underlying
observations. Its serial dependence, together with dependence
across grid points, is summarized by the long-run covariance
matrix
\begin{equation}\label{eq:Omega_longrun}
    \bs\Omega_{\lambda,n}
    \isdef
    \sum_{\ell\in\Zbb}
    \Ebb^{\Pbb}\!\left[
        \bs g_0\bs g_\ell^{\top}
    \right]
    \in\Rbb^{n\times n},
    \qquad
    [\bs\Omega_{\lambda,n}]_{k,j}
    =
    \langle u_k,\Qcal_\lambda u_j\rangle_{\Hscr}.
\end{equation}
Its representation in terms of $\Qcal_\lambda$ follows from the long-run covariance operator proved in Proposition~\ref{proposition:asymptotic_gauss}. Using the estimator’s Gaussian limit together with the bounded grid-evaluation map then gives us the joint asymptotic distribution of the estimated derivatives; the proof is given in Appendix~\ref{appendix:proofs_finite_grid}.

\begin{proposition}[Asymptotic distribution of derivative evaluations]
\label{proposition:asymptotic_convexity}
Under the conditions of
Proposition~\ref{proposition:asymptotic_gauss},
\begin{equation*}
    \sqrt{T}
    \Bigl(
        \widehat{\bs\theta}_{\lambda,n,T}
        -
        \bs\theta_{\lambda,n}
    \Bigr)
    \overset{d}{\longrightarrow}
    \Ncal_n(\bs 0,\bs\Omega_{\lambda,n})  \qquad\text{as} \; T\to\infty.
\end{equation*}
\end{proposition}

\paragraph{Feasible long-run covariance estimation.}
The long-run covariance matrix $\bs\Omega_{\lambda,n}$ depends on the unobserved population grid scores $\bs g_t$. To make inference feasible, we first construct their plug-in counterparts, beginning with the empirical version of the centered moment contribution $F_\lambda(\widetilde\zeta_t)$ as follows:
\begin{equation}\label{eq:sample_Q}
    \widehat F_t
    \isdef
    (Y_t\psi_t-\widehat\mu_T)
    -
    (\psi_t\otimes\psi_t-\widehat\Sig_T)
    \widehat h_{\lambda,T}.
\end{equation}
The estimated regularized inverse and derivative directions
give the feasible grid scores
\begin{equation}\label{eq:gridscore_feasible}
    \widehat{\bs g}_t
    \isdef
    \Dscr_{n,\bs\alpha}
    \widehat\Sig_{\lambda,T}^{-1}\widehat F_t
    =
    \Bigl[
        \langle\widehat u_j,\widehat F_t\rangle_{\Hscr}
    \Bigr]_{j=1}^n
    \in\Rbb^n.
\end{equation}
To accommodate serial correlation,
define the sample lag-covariance matrices, for $\ell\ge0$, by
\begin{equation}\label{eq:lag_covariances}
    \widehat{\bs\Gamma}^{g}_\ell
    \isdef
    \frac1T
    \sum_{t=1}^{T-\ell}
    \widehat{\bs g}_t
    \widehat{\bs g}_{t+\ell}^{\top}.
\end{equation}
Following \citet{NeweyWest1987}, we combine these matrices
using Bartlett weights to obtain the heteroskedasticity and
autocorrelation-consistent covariance estimator
\begin{equation}\label{eq:HAC_main}
    \widehat{\bs\Omega}_{\lambda,n,T}
    \isdef
    \widehat{\bs\Gamma}^{g}_0
    +
    \sum_{\ell=1}^{L_T}
    w_\ell
    \left(
        \widehat{\bs\Gamma}^{g}_\ell
        +
        (\widehat{\bs\Gamma}^{g}_\ell)^{\top}
    \right),
    \qquad
    w_\ell
    \isdef
    1-\frac{\ell}{L_T+1}.
\end{equation}
The bandwidth $L_T$ determines how many lagged covariance terms enter
the estimate, while the Bartlett weights preserve positive
semidefiniteness. Under the stated strong-mixing conditions,
a bandwidth that increases sufficiently slowly with the sample
size yields a consistent estimate of the population long-run
covariance.

\begin{theorem}[Consistency of the feasible long-run covariance estimator]
\label{thm:consistency}
Suppose the conditions of
Proposition~\ref{proposition:asymptotic_gauss} hold. If
$L_T\to\infty$ and $L_T=o(T^{1/4})$, then
\begin{equation}\label{eq:HAC_consistency_main}
    \widehat{\bs\Omega}_{\lambda,n,T}
    \overset{p}{\longrightarrow}
    \bs\Omega_{\lambda,n}.
\end{equation}
Moreover, $\widehat{\bs\Omega}_{\lambda,n,T}\geq0$, and if
$\bs\Omega_{\lambda,n}>0$, then
$\widehat{\bs\Omega}_{\lambda,n,T}$ is invertible with probability tending to
one.
\end{theorem}
\noindent The proof first establishes consistency of the idealized estimator based on the
population grid scores and then shows that replacing those scores by their
feasible plug-in counterparts is asymptotically negligible; see
Appendix~\ref{appendix:proofs_finite_grid}.

\subsection{Cone null and one-sided Wald statistic}
\label{subsec:test_stat}
The null hypothesis in \eqref{eq:hypothesis_test} requires the signed derivatives to satisfy their restrictions simultaneously at every grid point. Since these derivatives are computed from the same dependent observations, any evidence against the null must reflect their joint sampling uncertainty. A natural way to quantify departures is the covariance-adjusted distance between the estimated derivative vector and the set of vectors that satisfy the null. Employing the feasible long-run covariance estimator from the previous subsection, we define the one-sided Wald statistic
\begin{equation}\label{eq:test_statistic}
    W_T
    \isdef
    \min_{\bs w\in\Rbb_+^n}
    T
    (\widehat{\bs\theta}_{\lambda,n,T}-\bs w)^\top
    \widehat{\bs\Omega}_{\lambda,n,T}^{-1}
    (\widehat{\bs\theta}_{\lambda,n,T}-\bs w).
\end{equation}
This statistic equals zero whenever the estimated derivative vector satisfies all restrictions, and it increases as the magnitude of violations grows relative to their joint sampling uncertainty; we can minimize it as a nonnegative least-squares problem, as shown in Appendix~\ref{appendix:test_statistic}. We assume the population long-run covariance matrix $\bs\Omega_{\lambda,n}$ is positive definite, so its feasible estimator is invertible with probability approaching one.

\paragraph{Least-favorable null calibration.}
As the null is composite, the statistic's asymptotic distribution depends on which inequalities are binding. If a restriction holds strictly, the probability of its estimated counterpart violating the restriction approaches zero. By contrast, when all inequalities bind, each coordinate can affect the distance to the feasible set. Hence, the maximal asymptotic rejection probability arises at the vertex $\bs\theta_{\lambda,n}=\bs0$, where all signed derivatives are zero. 
We therefore calibrate the test at this least-favorable
configuration to control the asymptotic rejection probability
at each fixed point of the composite null.
At the vertex, the asymptotic limit of the estimated derivative vector
implies that the Wald statistic converges to the squared Mahalanobis
distance\footnote{For a positive definite matrix $\bs V$, consider the inner product $\langle \bs x, \bs y\rangle_{\bs V^{-1}} \isdef \bs x^\top \bs V^{-1} \bs y$ and write the induced norm as
$\|\bs x\|_{\bs V^{-1}}^2=\bs x^\top\bs V^{-1}\bs x$. 
For any closed convex set $\Mscr$, the \emph{Mahalanobis distance} is defined as 
\begin{equation*}
    \dist_{\bs V^{-1}}(\bs x,\Mscr)
    \isdef
    \min_{\bs y\in\Mscr}
    \|\bs x-\bs y\|_{\bs V^{-1}}.
\end{equation*}
For a closed convex cone $\Mscr$,  $\Mscr^\circ \isdef \{\bs x: \bs x^\top \bs V^{-1}\bs y \leq 0 \, \text{for every} \, \bs y \in \Mscr\}$ denotes its \emph{polar cone}
under this inner product.} of a centered Gaussian vector from the null cone. This distance
follows a chi-bar-square distribution, the classical limiting distribution arising from projections under cone restrictions.
\begin{definition}[Chi-bar-square distribution]
\label{def:chi_bar}
Let $\Mscr\subset\Rbb^n$ be a closed convex cone, and let
$\bs Z\sim\Ncal_n(\bs 0,\bs V)$, where $\bs V$ is symmetric and
positive definite. Then $\Bar{\chi}^2(\bs V,\Mscr)$ denotes the
distribution of
\begin{equation*}
    \bs Z^\top\bs V^{-1}\bs Z
    -
    \min_{\bs x\in\Mscr}
    (\bs Z-\bs x)^\top
    \bs V^{-1}
    (\bs Z-\bs x).
\end{equation*}
\end{definition}

\noindent
For the null cone $\Rbb_+^n$, the \emph{Moreau decomposition} shows
that the squared Mahalanobis distance to it equals the
squared norm of the projection onto its polar. Thus, the
relevant cone in Definition~\ref{def:chi_bar} is
$(\Rbb_+^n)^\circ$; see
Appendix~\ref{appendix:chi_bar_projection}.

\begin{theorem}[One-sided Wald statistic]
\label{thm:test_statistic}
Suppose the conditions of
Theorem~\ref{thm:consistency} hold and
$\bs\Omega_{\lambda,n}>0$. Under the least-favorable null
$\bs\theta_{\lambda,n}=\bs0$,
\begin{equation}\label{eq:WT_limit}
    W_T
    \overset{d}{\longrightarrow}
    \Bar{\chi}^2\!\left(
        \bs\Omega_{\lambda,n},
        (\Rbb_+^n)^\circ
    \right)\qquad \text{as $T \to \infty$}.
\end{equation}
\end{theorem}

\noindent
The chi-bar-square distribution admits the finite-mixture
representation
\begin{equation*}
\Pr(\Bar{\chi}^2(\bs V,\Mscr)\leq c)
=\sum_{j=0}^{n}\omega_j(n,\bs V,\Mscr)
\Pr(\chi_j^2\leq c),
\end{equation*}
where the nonnegative weights sum to one and depend on the
covariance matrix and the cone, and $\chi_0^2$ denotes a
point mass at zero; see
\citet[Theorem~3.4.2]{Silvapulle2001}.

\paragraph{Asymptotic size control and consistency.}
Although the test is calibrated using the asymptotic distribution at the vertex, the null also includes cases in which some or all inequalities hold strictly. The following proposition establishes that the asymptotic rejection probability in these cases does not exceed that at the vertex.
\begin{proposition}[Least-favorable null calibration]
\label{prop:level}
Suppose the conditions of Theorem~\ref{thm:test_statistic} hold and
$\bs\theta_{\lambda,n}\in\Rbb_+^n$. Then, for every continuity point
$c$ of $W\sim
\Bar\chi^2\!\left(
    \bs\Omega_{\lambda,n},
    (\Rbb_+^n)^\circ
\right)$,
\begin{equation}\label{eq:levelbound_main}
    \limsup_{T\to\infty}\Pr(W_T>c)
    \leq
    \Pr(W>c).
\end{equation}
\end{proposition}
\noindent For implementation, we approximate the asymptotic null distribution by simulating the Gaussian cone-projection statistic using the estimated long-run covariance matrix, as discussed in \citet[Section~3.5]{Silvapulle2001}. The theoretical size guarantee does not account for Monte Carlo approximation error; see Appendix~\ref{appendix:test_statistic} for details. Size control alone does not ensure that the test can detect shape violations. The following result establishes consistency against fixed violations.

\begin{proposition}[Consistency against fixed violations]
\label{prop:fixed_power}
Suppose the conditions of Theorem~\ref{thm:test_statistic} hold and
$\bs\theta_{\lambda,n}\notin\Rbb_+^n$. Then
\begin{equation*}
    \frac{W_T}{T}
    \overset{p}{\longrightarrow}
    \dist^2_{\bs\Omega_{\lambda,n}^{-1}}
    \left(
        \bs\theta_{\lambda,n},
        \Rbb_+^n
    \right)
    >0.
\end{equation*}
Hence, $W_T\overset{p}{\longrightarrow}\infty$, and for every fixed finite
critical value $c$, $\Pr(W_T>c)\to1$.
\end{proposition}

\noindent
Thus, vertex calibration controls the asymptotic rejection
probability at each fixed point of the composite null but is
generally conservative when some inequalities hold strictly:
at a fixed interior null point, the rejection probability
vanishes. Alternatively, inference
can account for the inequalities binding at a particular
boundary configuration, as in the local testing approach
of \citet{Wolak1989LocalGlobal}. We retain vertex calibration
without estimating the binding set. 
Proofs are given in Appendix~\ref{appendix:chi_bar_projection}.

\paragraph{Relation to the minimum-norm target.}
The preceding inference results are designed for the regularized target $h_\lambda$. Their implications for the minimum-norm target $h_0$ depend on whether the two functions have the same derivative signs on the testing grid. Under the source condition introduced in Section~\ref{subsec:reg_approx}, convergence of $h_\lambda$ to $h_0$ also implies convergence of their derivatives at the grid points. However, an economic function generating the same payoff may differ from $h_0$ by an element of $\operatorname{Ker}(\Sig)$ whose derivatives need not vanish on the grid. Identifying the derivative signs of such a function therefore requires additional structure. The following proposition establishes conditions under which strict derivative signs are preserved between $h_0$ and $h_\lambda$. The proof is given in Appendix~\ref{appendix:chi_bar_projection}.

\begin{proposition}[Regularization robustness of finite-grid inference]
\label{prop:regularization_bridge}
Suppose the RKHS-projected payoff
$\Rscr^\star_\Hscr$ defined in
\eqref{eq:restricted-payoff-space} is attained, i.e., there exists some $h\in\Hscr$ such that $J_{\Pbb, \Hscr}h = \Rscr^\star_\Hscr$. Let $h_0$ be the
minimum-norm target defined in
\eqref{eq:h0_definition} and suppose the source condition
\eqref{eq:source_condition} holds for some
$0<\nu\leq\tfrac12$. Define
\begin{equation}\label{eq:theta_zero}
    \bs\theta_{0,n}
    \isdef
    \varsigma\Dscr_{n,\bs\alpha}h_0.
\end{equation}
Then, as $\lambda\downarrow0$,
\begin{equation}\label{eq:grid_regularization_bias}
    \max_{1 \leq j \leq n} |\theta_{\lambda,n, j}-\theta_{0,n, j}|
    \leq
    \left(
        \max_{1\leq k\leq n}
        \|\phi^{(\bs\alpha)}(\bs\xi_k)\|_\Hscr
    \right)
    \|h_\lambda-h_0\|_\Hscr
    \lesssim
    \lambda^\nu.
\end{equation}
Therefore,
\begin{enumerate}
    \item If
    $\min_{1\leq j\leq n}\theta_{0,n,j}>0$, then
    $\bs\theta_{\lambda,n}\in\operatorname{int}(\Rbb_+^n)$
    for all sufficiently small $\lambda$;
    \item If
    $\min_{1\leq j\leq n}\theta_{0,n,j}<0$, then there exists
    $\bar\lambda>0$ such that, for every
    $0<\lambda\leq\bar\lambda$ for which the conditions of
    Theorem~\ref{thm:test_statistic} hold,
    \begin{equation*}
        \bs\theta_{\lambda,n}\notin\Rbb_+^n
        \qquad\text{and}\qquad
        \Pr(W_T>c)\longrightarrow1
    \end{equation*}
    for every fixed finite critical value $c$.
\end{enumerate}
If $\bs\theta_{0,n}\in\partial\Rbb_+^n$, the bound alone does not
determine on which side of the boundary
$\bs\theta_{\lambda,n}$ lies.
\end{proposition}


\section{Empirical applications}
\label{sec:empirical_applications}

We examine whether expected returns increase with market beta, firm characteristics, and Treasury maturity. In the equity setting, restrictions of this kind have been studied across characteristic-sorted portfolios, notably by \citet{PattonTimmermann2010} and \citet{RomanoWolf2013}. Consider $J$ portfolios ordered by an economic characteristic, let $m_j$ denote the population expected return of portfolio $j$, and define $\Delta_j=m_{j+1}-m_j$ for $j=1,\ldots,J-1$. For our procedure, let $\theta_{\lambda,n,j}$ denote the signed derivative evaluation defined in \eqref{eq:shape-null-generic}. The table below compares their hypotheses with those of the classical one-sided inequality test.
\begin{center}
\begingroup
\renewcommand{\arraystretch}{1}
{
\setlength{\tabcolsep}{7pt}
\begin{tabular}{@{}p{0.32\textwidth}p{0.30\textwidth}p{0.32\textwidth}@{}}
\hline
\textbf{Procedure} &
\textbf{Null hypothesis $H_0$} &
\textbf{Alternative hypothesis $H_1$} \\
\hline
Patton--Timmermann (MR) &
$\Delta_j\leq 0$ for all $j$ &
$\min_j\Delta_j>0$ \\

Romano--Wolf &
$\min_j\Delta_j\leq 0$ &
$\min_j\Delta_j>0$ \\

Wolak &
$\Delta_j\geq 0$ for all $j$ &
$\min_j\Delta_j<0$ \\

Present paper &
$\theta_{\lambda,n,j}\geq 0$ for all $j$ &
$\min_j\theta_{\lambda,n,j}<0$ \\
\hline
\end{tabular}
}
\endgroup
\end{center}
The first two procedures seek evidence of strictly increasing expected returns but differ in their null hypotheses. The monotonic-relation (MR) test of \citet{PattonTimmermann2010} contrasts nonpositive adjacent differences with strictly positive ones. As \citet{RomanoWolf2013} emphasize, these hypotheses do not exhaust the parameter space, since configurations containing both positive and nonpositive differences are excluded. Their formulation instead places every configuration inconsistent with strict increase under the null, without ruling out nonmonotonic relationships a priori.

The last two procedures address the converse problem. In the inequality-testing approach of \citet{Wolak1989}, weak monotonicity is imposed under the null hypothesis, and rejection is evidence that this restriction fails; see also \citet[Section~2.4]{PattonTimmermann2010}. Our tests share this orientation: rejection provides \emph{evidence against the maintained shape restriction} on the preselected testing grid, whereas a failure to reject does not confirm monotonicity. The procedures also differ in what they seek to infer: differences between adjacent portfolios compare expected returns across groups sorted by characteristics, while our restrictions apply to derivatives of an estimated function. The two are mathematically related but not equivalent.\footnote{
For a continuously differentiable scalar function,
$h_\lambda(b)-h_\lambda(a)
=\int_a^b h_\lambda'(x)\,\d x$ for $a<b$.
Nonnegative derivatives throughout an interval imply
nonnegative differences in function values at its endpoints,
but nonnegative differences between finitely many points
do not rule out negative derivatives within the intervals.
Conversely, nonnegative derivatives at finitely many grid
points do not establish global monotonicity. Moreover,
adjacent differences in portfolio expected returns need
not equal differences in $h_\lambda$ evaluated at the
corresponding portfolio-average characteristics.}

\subsection{Data and empirical implementation}
\label{subsec:data_empirical_implementation}

The equity applications use monthly U.S.\ stock returns and characteristics from \citet{JensenKellyPedersen2023}. At each formation date, we sort eligible stocks into ten groups of approximately equal size, either by market beta or by one of eight firm characteristics spanning firm size, valuation ratios, momentum, and short- and long-term return reversals. Within each group, we form a value-weighted portfolio and model its expected excess return as a function of its characteristic. The beta sample covers 611 formation months, from January 1975 to November 2025, and the characteristics sample covers 743 months, from January 1964 to November 2025. The Treasury applications use monthly CRSP data on Treasury bills with one to twelve months to maturity. We study how expected excess returns vary with remaining maturity, both unconditionally and conditional on the one-month Treasury yield, in three subsamples: 1964--1972, 1973--2001, and 2009--2025, comprising 108, 348, and 203 months, respectively.%
\footnote{We omit 2002--2008 to maintain comparable
coverage toward the longer end of the maturity range.
The U.S.\ Treasury discontinued regular 52-week bill
issuance in 2001 and reintroduced it only in June 2008,
limiting the availability of longer-maturity bills during
the intervening period. See the
\href{https://www.treasurydirect.gov/research-center/timeline/bills/}
{U.S.\ Treasury's historical issuance timeline}.}
Since our objective is inference on the shape of the estimated relationship rather than prediction, we select the regularization parameter through the effective dimension of the estimator. This approach controls the complexity of the fit without using the observed returns in the selection criterion. Across applications, we use a common \emph{effective dimension target} of $d=5$; Appendix~\ref{appendix:lambda_selection} gives the details. We then test whether the resulting regularized relationship $h_\lambda$ is weakly increasing or weakly decreasing on finite grids within three prespecified regions of the beta, characteristic, or maturity range. Appendix~\ref{appendix:additional_details} describes the kernels, weights, testing grids, and long-run covariance estimation. Monte Carlo studies under serial dependence provide evidence on the finite-sample behavior of the proposed tests, with power against monotonicity violations increasing substantially with sample size (Appendix~\ref{appendix:monte_carlo_details}).

\subsection{Equity expected returns}
\label{subsec:equity_expected_returns}
\label{subsec:beta_return_application}

\paragraph{Market beta and expected excess returns.}  Using the beta-sorted portfolios
as described above, we conduct our experiment across three
regions: low beta ($0$--$0.75$), intermediate beta
($0.75$--$1.25$), and high beta ($1.25$--$2.75$).
For each region, we test whether expected excess returns
are weakly increasing or weakly decreasing in beta.
\begin{figure}[!ht]
\centering
\includegraphics[width=0.55\textwidth]{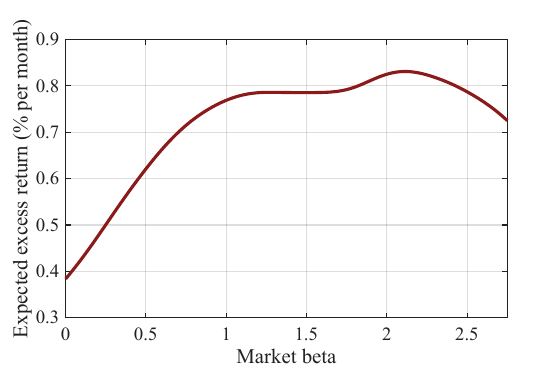}
\caption{Expected excess returns and market beta, in percent per month.}
\label{fig:beta_return_selected_periods}
\end{figure}

\begin{table}[!ht]
\centering
\caption{Shape tests for expected excess returns and market beta}
\label{tab:beta_return_selected_periods}
\renewcommand{\arraystretch}{1.5}
\begin{tabular}{@{}lccc@{}}
\toprule
& Low Beta & Intermediate Beta & High Beta \\
\midrule
Weakly increasing & $1.000$ & $1.000$ & $0.674$ \\
Weakly decreasing & $0.015^{**}$ & $0.176$ & $0.875$ \\
\bottomrule
\end{tabular}

\medskip
\begin{minipage}{\textwidth}
\footnotesize
\textit{Notes:}
The null hypotheses are weakly increasing and weakly
decreasing expected excess returns in market beta
for the regularized relationship of Section~\ref{subsec:data_empirical_implementation},
tested separately on the grid of each region.
The sample contains 611 months. Inference uses long-run covariance estimation with five lags
and 10,000 Monte Carlo draws.
$^{*}$, $^{**}$, and $^{***}$ indicate significance
at the 10\%, 5\%, and 1\% levels, respectively.
\end{minipage}
\end{table}
The estimated relationship, displayed in
Figure~\ref{fig:beta_return_selected_periods},
is generally upward sloping, with the largest increase
occurring among low-beta portfolios. Expected excess
returns flatten at intermediate beta values and change
little at higher values, rising slightly before
declining toward the upper end of the range. 
Table~\ref{tab:beta_return_selected_periods}
shows that weakly increasing expected returns are not
rejected in any of the three regions, while weakly decreasing
returns are rejected only in the low-beta region,
at the 5\% significance level.

\paragraph{Firm characteristics and expected excess returns.}

Following \citet{PattonTimmermann2010}, we examine eight firm
characteristics for which expected excess returns are hypothesized
to decrease with firm size, short-term reversal, and long-term
reversal, and increase with book-to-market, cash flow-to-price,
earnings-to-price, dividend yield, and momentum.
For each characteristic, we test both directions of monotonicity
within the three regions specified in Table~\ref{tab:e2_shape_tests}.
\begin{figure}[!htbp]
\centering
\includegraphics[width=0.92\textwidth]{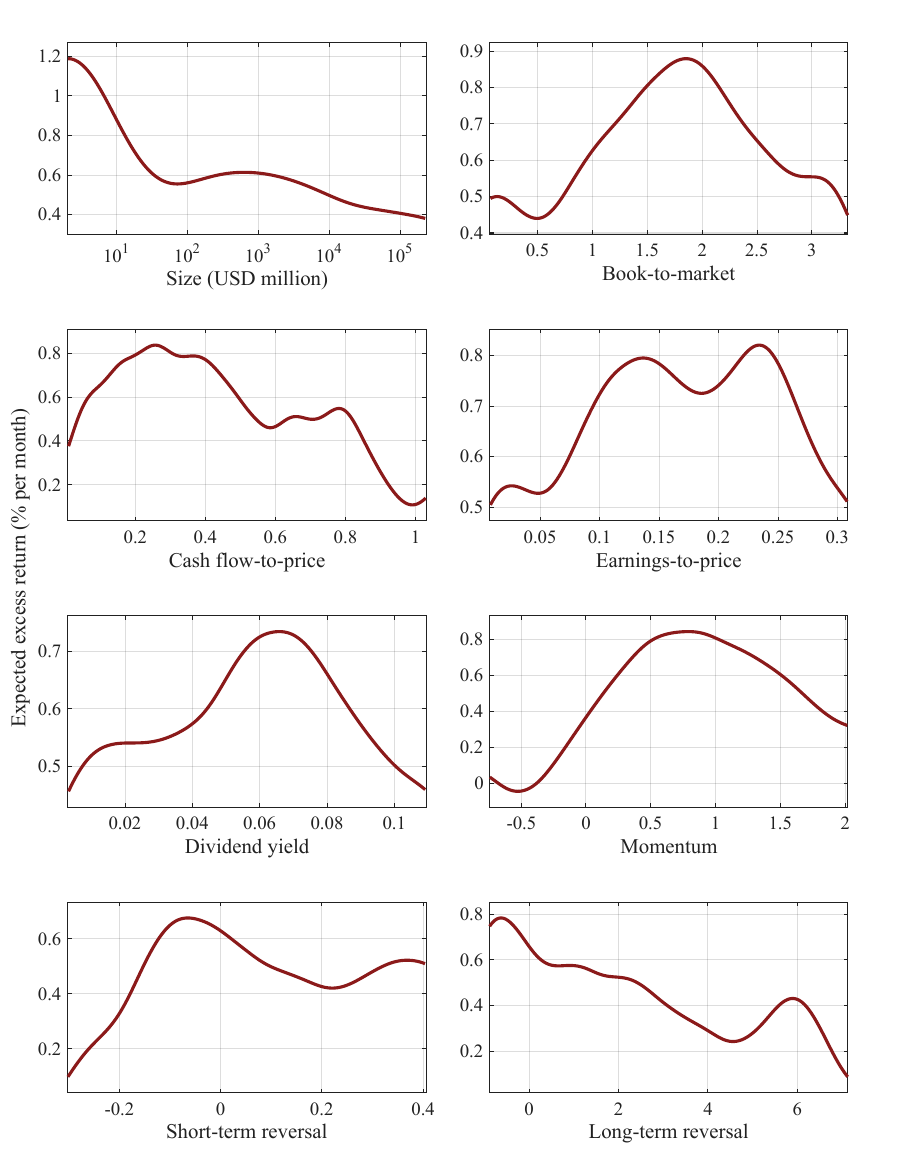}
\caption{Estimated expected excess returns for eight firm
characteristics, in percent per month.}
\label{fig:e2_profiles}
\end{figure}

\begin{table}[!htbp]
\centering
\caption{Shape tests for expected excess returns
and firm characteristics}
\label{tab:e2_shape_tests}
\renewcommand{\arraystretch}{1.25}
\begin{tabular}{@{}lccc@{}}
\toprule
& Region 1 & Region 2 & Region 3 \\
\midrule
\multicolumn{4}{@{}l}{\textbf{Size}} \\
\quad Weakly increasing
& $0.031\rlap{$^{**}$}$ & $0.417$ & $0.221$ \\
\quad Weakly decreasing
& $0.607$ & $0.834$ & $1.000$ \\
\midrule
\multicolumn{4}{@{}l}{\textbf{Book-to-Market}} \\
\quad Weakly increasing
& $0.405$ & $1.000$ & $0.584$ \\
\quad Weakly decreasing
& $0.204$ & $0.017\rlap{$^{**}$}$ & $0.223$ \\
\midrule
\multicolumn{4}{@{}l}{\textbf{Cash Flow-to-Price}} \\
\quad Weakly increasing
& $1.000$ & $1.000$ & $0.509$ \\
\quad Weakly decreasing
& $<0.001\rlap{$^{***}$}$ & $0.116$ & $0.880$ \\
\midrule
\multicolumn{4}{@{}l}{\textbf{Earnings-to-Price}} \\
\quad Weakly increasing
& $0.602$ & $1.000$ & $0.779$ \\
\quad Weakly decreasing
& $0.096\rlap{$^{*}$}$ & $0.109$ & $0.394$ \\
\midrule
\multicolumn{4}{@{}l}{\textbf{Dividend Yield}} \\
\quad Weakly increasing
& $1.000$ & $1.000$ & $0.399$ \\
\quad Weakly decreasing
& $0.031\rlap{$^{**}$}$ & $0.066\rlap{$^{*}$}$ & $0.087\rlap{$^{*}$}$ \\
\midrule
\multicolumn{4}{@{}l}{\textbf{Momentum}} \\
\quad Weakly increasing
& $0.657$ & $1.000$ & $0.598$ \\
\quad Weakly decreasing
& $0.037\rlap{$^{**}$}$ & $0.002\rlap{$^{***}$}$ & $0.265$ \\
\midrule
\multicolumn{4}{@{}l}{\textbf{Short-Term Reversal}} \\
\quad Weakly increasing
& $0.687$ & $0.155$ & $0.239$ \\
\quad Weakly decreasing
& $0.057\rlap{$^{*}$}$ & $1.000$ & $0.689$ \\
\midrule
\multicolumn{4}{@{}l}{\textbf{Long-Term Reversal}} \\
\quad Weakly increasing
& $0.219$ & $0.281$ & $0.363$ \\
\quad Weakly decreasing
& $0.405$ & $0.740$ & $1.000$ \\

\bottomrule
\end{tabular}

\medskip
\begin{minipage}{\textwidth}
\footnotesize
\textit{Notes:}
The null hypotheses are weakly increasing and weakly
decreasing expected excess returns in each characteristic
for the regularized relationship of Section~\ref{subsec:data_empirical_implementation},
tested separately on the grids of three regions.
The regions, in order, are:
Size $[5,100]$, $[100,10000]$,
$[10000,200000]$;
Book-to-Market $[0.1,0.6]$, $[0.6,1.2]$,
$[1.2,2.2]$;
Cash Flow-to-Price $[0.01,0.1]$, $[0.1,0.25]$,
$[0.25,0.4]$;
Earnings-to-Price $[0.01,0.05]$, $[0.05,0.1]$,
$[0.1,0.2]$;
Dividend Yield $[0.004,0.02]$, $[0.02,0.05]$,
$[0.05,0.08]$;
Momentum $[-0.6,0]$, $[0,0.5]$, $[0.5,1.1]$;
Short-Term Reversal $[-0.2,-0.05]$,
$[-0.05,0.05]$, $[0.05,0.25]$;
and Long-Term Reversal $[-0.75,0]$,
$[0,1]$, $[1,3.5]$.
Size is measured in USD millions.
The sample contains 743 months.
Inference uses long-run covariance estimation
with six lags and 10,000 Monte Carlo draws.
$^{*}$, $^{**}$, and $^{***}$ indicate significance
at the 10\%, 5\%, and 1\% levels, respectively.
\end{minipage}
\end{table}
Although the fitted relationships in Figure~\ref{fig:e2_profiles} vary considerably, none of the monotonicity predictions examined by \citet{PattonTimmermann2010} is rejected at the $5\%$ significance level. Short-term reversal is the only characteristic for which the predicted direction is rejected at the $10\%$ level ($p=0.057$). The tests provide stronger evidence against monotonicity in the opposite direction. Weakly increasing returns are rejected for smaller firms, whereas weakly decreasing returns are rejected in specific ranges of book-to-market, cash flow-to-price, dividend yield, and momentum. The strongest rejections occur for cash flow-to-price and momentum, with $p$-values below $0.001$ and equal to $0.002$, respectively. For long-term reversal, neither direction is rejected. Overall, the results are broadly in line with established characteristic-return patterns, and significant rejections of the opposite monotonicity restrictions are confined to particular parts of the characteristic distributions.

\subsection{Treasury term premia}
\label{subsec:treasury_bill_application}

\paragraph{Term structure of expected Treasury-bill excess returns.}
We turn to the relationship between Treasury-bill excess
returns and remaining maturity. We examine this relationship
separately over 1964--1972, 1973--2001, and 2009--2025,
distinguishing three maturity ranges: 1--5, 5--8, and
8--12 months. Within each range, we test whether expected
excess returns are weakly increasing or weakly decreasing
with maturity.

\begin{figure}[!htbp]
\centering
\includegraphics[width=0.6\textwidth]{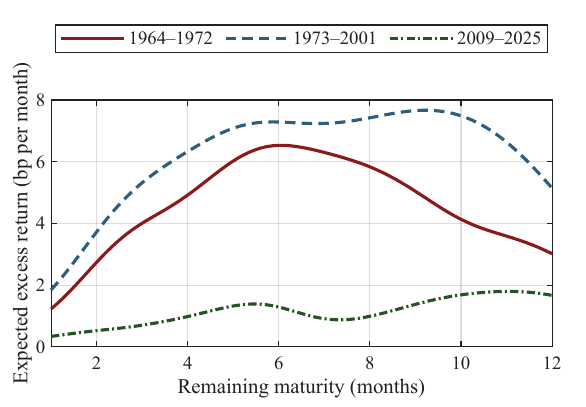}
\caption{Expected excess returns and remaining maturity.
The three curves correspond to the subsamples 1964--1972,
1973--2001, and 2009--2025. Expected excess returns
are expressed in basis points per month.}
\label{fig:treasury_derivative_slices}
\end{figure}

\begin{table}[!htbp]
\centering
\caption{Shape tests for Treasury-bill maturity premia}
\label{tab:treasury_shape_counts}
\renewcommand{\arraystretch}{1.5}
\begin{tabular}{@{}lccc@{}}
\toprule
& 1--5 Months & 5--8 Months & 8--12 Months \\
\midrule

\multicolumn{4}{@{}l}{\textbf{Panel A: 1964--1972 ($T=108$)}} \\
\quad Weakly increasing
& $1.000$ & $0.012^{**}$ & $<0.001^{***}$ \\
\quad Weakly decreasing
& $<0.001^{***}$ & $0.003^{***}$ & $1.000$ \\

\midrule
\multicolumn{4}{@{}l}{\textbf{Panel B: 1973--2001 ($T=348$)}} \\
\quad Weakly increasing
& $1.000$ & $0.502$ & $<0.001^{***}$ \\
\quad Weakly decreasing
& $<0.001^{***}$ & $0.046^{**}$ & $0.427$ \\

\midrule
\multicolumn{4}{@{}l}{\textbf{Panel C: 2009--2025 ($T=203$)}} \\
\quad Weakly increasing
& $1.000$ & $0.035^{**}$ & $0.159$ \\
\quad Weakly decreasing
& $<0.001^{***}$ & $0.004^{***}$ & $<0.001^{***}$ \\

\bottomrule
\end{tabular}

\medskip
\begin{minipage}{\textwidth}
\footnotesize
\textit{Notes:}
The null hypotheses are weakly increasing and weakly
decreasing expected excess returns in remaining maturity
for the regularized relationship of Section~\ref{subsec:data_empirical_implementation},
tested separately on the grid of each region. The three panels
cover 108, 348, and 203 months, respectively.
Inference uses long-run covariance estimation with
4, 5, and 4 lags in Panels A--C, respectively, and
10,000 Monte Carlo draws.
$^{*}$, $^{**}$, and $^{***}$ indicate significance
at the 10\%, 5\%, and 1\% levels, respectively.
\end{minipage}
\end{table}
Figure~\ref{fig:treasury_derivative_slices} and
Table~\ref{tab:treasury_shape_counts} reveal considerable
variation in the maturity profile of expected excess returns
across historical periods.
During 1964--1972, expected returns exhibit a hump-shaped
pattern, peaking at about six months.
The tests reject weakly decreasing returns at short maturities,
both directions of monotonicity at intermediate maturities of
five to eight months, and weakly increasing returns at longer
maturities.
During 1973--2001, expected returns rise until about nine
months before declining.
The evidence against weakly increasing maturity premia is
confined to the eight- to twelve-month range, while weakly
decreasing returns are rejected at shorter and intermediate
maturities.
The post-2008 period presents a different picture:
weakly decreasing returns are rejected in all three regions,
but weakly increasing returns are also rejected at intermediate
maturities.
These departures from monotonicity, however, are economically
small, with differences in expected returns generally on the
order of one basis point per month.

\paragraph{Conditional term structure of expected Treasury-bill excess returns}

We next examine how the relationship between
Treasury-bill excess returns and remaining maturity
varies with the prevailing short rate. For each
subsample, we consider representative interest-rate
environments corresponding to the 10th, 50th, and
90th percentiles of the one-month Treasury yield.
For 2009--2025, we retain the median and
90th-percentile states.\footnote{The 10th-percentile
short rate is close to the median near the zero
lower bound and is therefore omitted.} At each interest-rate level,
we examine expected excess returns across maturities
and test the same monotonicity restrictions
as in the previous analysis.
\begin{figure}[!ht]
\centering
\includegraphics[width=0.9\textwidth]
{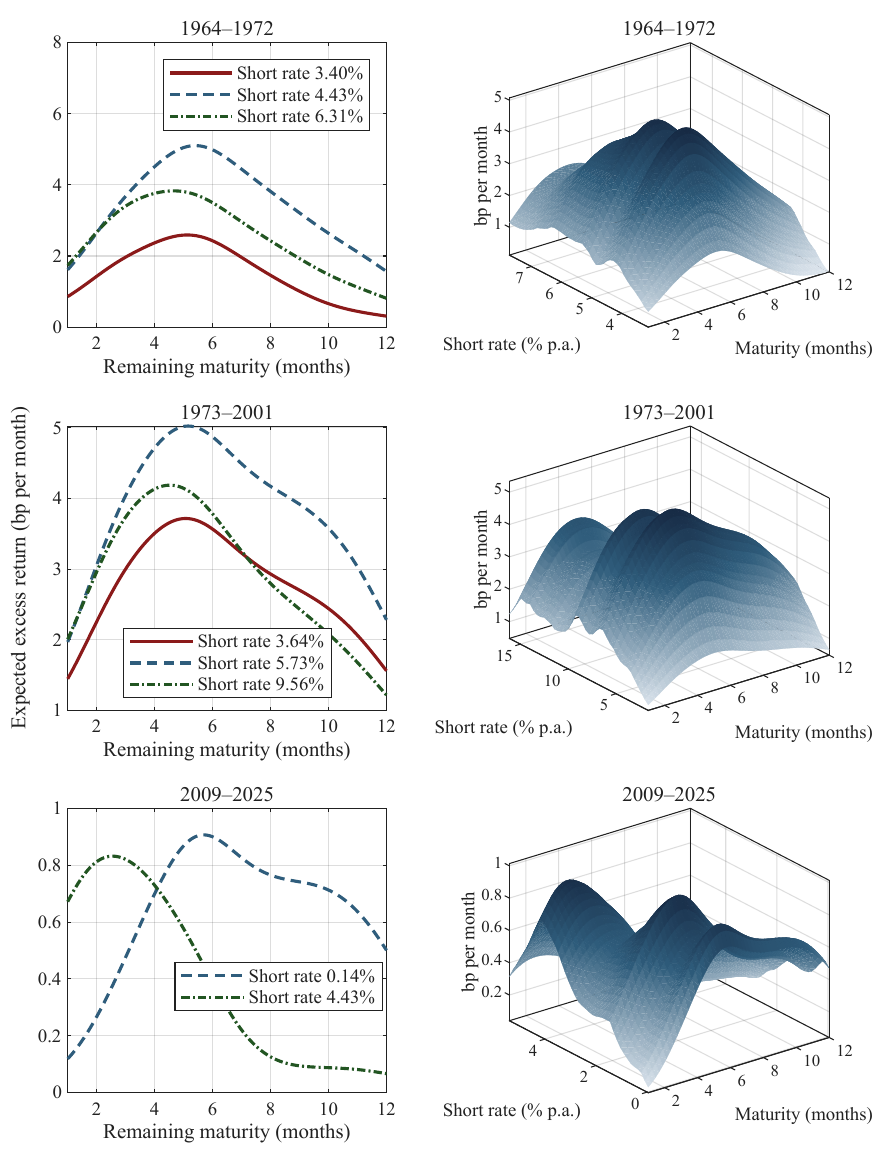}
\caption{Conditional Treasury-bill maturity premia.
The left panels display expected excess returns
across remaining maturity at selected short-rate
states; the right panels show the fitted
expected-return surfaces. Expected excess returns
are in basis points per month, and short rates
are in percent per annum.}
\label{fig:e4_profiles_surfaces}
\end{figure}

\begin{table}[!htbp]
\centering
\caption{Shape tests for conditional Treasury-bill
maturity premia}
\label{tab:e4_shape_tests}
\renewcommand{\arraystretch}{1.25}

\begin{tabular}{@{}lccc@{}}
\toprule
& 1--5 Months & 5--8 Months & 8--12 Months \\
\midrule

\multicolumn{4}{@{}l}{
\textbf{Panel A: 1964--1972 ($T=108$)}} \\
\addlinespace[2pt]

\multicolumn{4}{@{}l}{\textit{Weakly increasing}} \\
\quad Short Rate 3.40\%
& $1.000$ & $<0.001^{***}$ & $<0.001^{***}$ \\
\quad Short Rate 4.43\%
& $1.000$ & $<0.001^{***}$ & $<0.001^{***}$ \\
\quad Short Rate 6.31\%
& $0.317$ & $0.002^{***}$ & $0.001^{***}$ \\

\addlinespace[2pt]
\multicolumn{4}{@{}l}{\textit{Weakly decreasing}} \\
\quad Short Rate 3.40\%
& $<0.001^{***}$ & $0.536$ & $1.000$ \\
\quad Short Rate 4.43\%
& $<0.001^{***}$ & $0.127$ & $1.000$ \\
\quad Short Rate 6.31\%
& $0.004^{***}$ & $1.000$ & $1.000$ \\

\midrule
\multicolumn{4}{@{}l}{
\textbf{Panel B: 1973--2001 ($T=348$)}} \\
\addlinespace[2pt]

\multicolumn{4}{@{}l}{\textit{Weakly increasing}} \\
\quad Short Rate 3.64\%
& $1.000$ & $<0.001^{***}$ & $<0.001^{***}$ \\
\quad Short Rate 5.73\%
& $1.000$ & $<0.001^{***}$ & $<0.001^{***}$ \\
\quad Short Rate 9.56\%
& $0.165$ & $<0.001^{***}$ & $<0.001^{***}$ \\

\addlinespace[2pt]
\multicolumn{4}{@{}l}{\textit{Weakly decreasing}} \\
\quad Short Rate 3.64\%
& $<0.001^{***}$ & $0.507$ & $1.000$ \\
\quad Short Rate 5.73\%
& $<0.001^{***}$ & $0.392$ & $1.000$ \\
\quad Short Rate 9.56\%
& $<0.001^{***}$ & $1.000$ & $1.000$ \\

\midrule
\multicolumn{4}{@{}l}{
\textbf{Panel C: 2009--2025 ($T=203$)}} \\
\addlinespace[2pt]

\multicolumn{4}{@{}l}{\textit{Weakly increasing}} \\
\quad Short Rate 0.14\%
& $1.000$ & $0.045^{**}$ & $<0.001^{***}$ \\
\quad Short Rate 4.43\%
& $0.005^{***}$ & $<0.001^{***}$ & $0.024^{**}$ \\

\addlinespace[2pt]
\multicolumn{4}{@{}l}{\textit{Weakly decreasing}} \\
\quad Short Rate 0.14\%
& $<0.001^{***}$ & $0.035^{**}$ & $1.000$ \\
\quad Short Rate 4.43\%
& $<0.001^{***}$ & $1.000$ & $1.000$ \\

\bottomrule
\end{tabular}

\medskip
\begin{minipage}{\textwidth}
\footnotesize
\textit{Notes:}
The null hypotheses are weakly increasing and weakly
decreasing expected excess returns in remaining maturity
for the regularized relationship of Section~\ref{subsec:data_empirical_implementation},
conditional on the indicated one-month Treasury yield
(percent per annum). The tests use 108, 348, and 203 monthly observations
and long-run covariance estimation with 4, 5, and
4 lags in Panels A--C, respectively, and
10,000 Monte Carlo draws.
$^{*}$, $^{**}$, and $^{***}$ indicate significance
at the 10\%, 5\%, and 1\% levels, respectively.
\end{minipage}
\end{table}
Conditioning on the short rate reveals a more pronounced hump-shaped maturity profile. In 1964--1972 and 1973--2001, expected excess returns rise to a peak at intermediate maturities and decline thereafter at every short-rate level considered, and the tests reject weakly increasing returns beyond five months in each state (Figure~\ref{fig:e4_profiles_surfaces} and Table~\ref{tab:e4_shape_tests}). The contrast with the unconditional analysis is particularly striking for 1973--2001, when the unconditional relationship rises until about nine months and weakly increasing premia are rejected only beyond eight months. The substantial variation in short rates during this period suggests that pooling different interest-rate environments may obscure departures from increasing maturity premia. After 2008, the tests again reject monotonicity, but expected excess returns remain below one basis point per month, so these departures are economically modest. For 1973--2001, the conditional rejections at intermediate
maturities persist under alternative effective dimension
targets, whereas the corresponding unconditional result
is sensitive to this choice; see Appendix~\ref{appendix:lambda_selection}.
\subsection{Discussion}

Our beta results are consistent with the direction predicted by the CAPM, but the evidence of an increasing relationship is concentrated among low-beta portfolios. The relatively flat profile at intermediate and high beta values echoes the findings of \citet{FamaFrench1992}. Whereas \citet{PattonTimmermann2010} and \citet{RomanoWolf2013} examine whether an increasing ordering can be established across beta-sorted portfolios, our analysis identifies the ranges of beta in which the data provide evidence of an increase in expected returns.

For firm characteristics, the findings accord with the size and value effects documented by \citet{FamaFrench1992} and complement the monotonicity evidence of \citet{PattonTimmermann2010}. An advantage of examining the shape of these relationships regionally is that the evidence for a characteristic premium need not be equally strong across the entire distribution. This is particularly relevant for size and momentum, for which \citet{PattonTimmermann2010} find weaker evidence of monotonicity, and for short-term reversal, whose relationship with expected returns \citet{RomanoWolf2013} find to vary across firm-size groups.

The Treasury results qualify the traditional liquidity preference hypothesis. Although investors earn higher expected returns as maturity initially increases, this pattern need not extend to longer maturities. \citet{Fama1986} documents considerable variation in the shape of term premia across economic conditions. Our findings complement this evidence by highlighting the relevance of the prevailing interest-rate environment when assessing maturity premia. They also illustrate the importance of distinguishing statistically significant departures from monotonicity from their economic magnitude, particularly in periods of low interest rates.

\section{Conclusion}\label{sec:conclusion}

The shape of a financial relationship may be of considerable economic interest even when the relationship itself cannot be observed directly. The question is how much can be learned about its shape from the payoffs it generates and whether apparent departures from theoretical predictions are statistically significant. We estimate the relationship in a reproducing kernel Hilbert space using Tikhonov-regularized least squares, then use its estimated derivatives to test for monotonicity, convexity, and related properties rather than enforcing those restrictions during estimation. We derive finite-sample bounds on payoff estimation error under serial dependence and show that a one-sided Wald statistic constructed from the estimated derivatives has a chi-bar-square limiting distribution under the least-favorable null, ensuring asymptotic size control. The test is consistent against fixed violations. Since different functions can produce identical payoffs, the restrictions being tested apply to the regularized target selected by the estimator, not necessarily to the underlying economic relationship.

For equities, we generally do not reject the predicted monotonicity of premia with respect to market beta and firm characteristics. Tests of the reverse restrictions show where the evidence for these premia is concentrated, including among low-beta portfolios and the smallest firms. Treasury maturity premia increase at short maturities, but not over the entire maturity range. Conditioning on the prevailing short rate strengthens the evidence against increasing premia, highlighting why the interest-rate environment matters when evaluating the liquidity preference hypothesis. The value of leaving the shape unrestricted during estimation is precisely that departures from theoretical predictions, rather than being ruled out in advance, can become a source of economic insight.

\bibliographystyle{ecta-doi}
\bibliography{literature}

\appendix
\begin{center}
{\LARGE\textbf{Appendix for Kernel-Based Nonparametric Tests for Shape Constraints}}
\end{center}
\pdfbookmark[0]{Appendix}{appendixtitle}
\vspace{1em}

\section{Technical preliminaries}\label{sec:prelims}
Throughout the appendix, we use the same notation and conventions introduced in the main text. We recall standard results on Hilbert spaces and RKHS, together with the dependence notions used in the proofs.

\subsection{Hilbert spaces and operators}
\label{subsec:hilbert_operators}
We follow the material in \citet{ReedSimon, schatten, DunfordSchwartz,
Bauschke2017}. Let $\Hscr$ be a separable Hilbert space with
orthonormal basis $(e_j)_{j\in\Nbb}$.
The space $\Bscr(\Hscr)$ of bounded linear operators is
equipped with the operator norm $\|\cdot\|_{\op}$.
For an operator $A$, write $A^\ast$ for its adjoint and
$|A|\isdef(A^\ast A)^{1/2}$.
The Hilbert--Schmidt operators form a Hilbert space
$\mathsf{HS}(\Hscr)$ with inner product $ \langle A,B\rangle_{\HS}
    \isdef
    \tr(B^\ast A)
    =
    \sum_{j\in\Nbb}
    \langle Ae_j,Be_j\rangle_{\Hscr},$ independent of the choice of orthonormal basis.
The trace-class operators form the space
$\mathsf{T}(\Hscr)\isdef
\{A\in\Bscr(\Hscr):\tr(|A|)<\infty\}$.
The standard inclusions and norm inequalities are
\begin{equation*}
    \mathsf{T}(\Hscr)
    \subset\mathsf{HS}(\Hscr)
    \subset\Bscr(\Hscr),
    \qquad
    \|A\|_{\op}
    \le\|A\|_{\HS}
    \le\tr(|A|),
\end{equation*}
where the inequalities hold for trace-class operators. For $f,g\in\Hscr$, the rank-one operator $f\otimes g$
is defined by
$(f\otimes g)u\isdef\langle u,g\rangle_{\Hscr}f$.
It is trace class and satisfies
\begin{equation*}
    \langle A,f\otimes g\rangle_{\HS}
    =
    \langle Ag,f\rangle_{\Hscr},
    \qquad
    \|f\otimes g\|_{\HS}
    =
    \tr(|f\otimes g|)
    =
    \|f\|_{\Hscr}\|g\|_{\Hscr}.
\end{equation*}
For a positive self-adjoint operator $A$,
we write
$\lambda_{\min}(A)\isdef\inf\sigma(A)$ and
$\lambda_{\max}(A)\isdef\sup\sigma(A)$,
where $\sigma(A)$ denotes its spectrum.
For self-adjoint operators, $A\ge B$ denotes the
Löwner order, meaning that $A-B$ is positive semidefinite.

We also use the following notation for metric projections. Let $\Mscr$ be a nonempty closed convex subset of
a Hilbert space $\Hscr$. For $u\in\Hscr$, define
\begin{equation*}
    P_{\Hscr}(u;\Mscr)
    \isdef
    \operatorname*{argmin}_{v\in\Mscr}
    \|u-v\|_{\Hscr}.
\end{equation*}
The subscript identifies the relevant inner product. The following projection theorem
provides the corresponding characterization.

\begin{theorem}[{Projection theorem;
\citet[Theorem~3.16 and Corollary~3.24]{Bauschke2017}}]
\label{thm:projection_theorem}
Let $\Mscr$ be a nonempty closed convex subset of a Hilbert space
$\Hscr$. Then, for every $u\in\Hscr$, the projection
$P_{\Hscr}(u;\Mscr)$ is well-defined and unique, and is characterized by
\begin{equation*}
    \left\langle
        u-P_{\Hscr}(u;\Mscr),
        v-P_{\Hscr}(u;\Mscr)
    \right\rangle_{\Hscr}
    \le 0
    \qquad
    \text{for every }v\in\Mscr.
\end{equation*}
If $\Mscr$ is a closed linear subspace, then
$u-P_{\Hscr}(u;\Mscr)\perp\Mscr$.
\end{theorem}

\subsection{Reproducing kernel Hilbert spaces (RKHS)}
\label{subsec:RKHS}
We recall the basic properties of reproducing kernel Hilbert spaces used
throughout the paper. For further background and applications, see
\citet{Wen05,RKHS_book,pau_rag_16,bach2024learning}. Let
$\Xcal\subset\Rbb^d$, and let
$\Kcal\colon\Xcal\times\Xcal\to\Rbb$ be a function such that, for
every finite set $\{\bs x_1,\ldots,\bs x_N\}\subset\Xcal$, the Gram matrix $  \bs K\isdef\left[\Kcal(\bs x_i,\bs x_j)\right]_{i,j=1}^N\in\Rbb^{N\times N}$
is symmetric and positive semidefinite. The RKHS $\Hscr$ associated with
$\Kcal$ is the completion of $ \spn\left\{\phi(\bs x)\isdef \Kcal(\bs x,\cdot):\bs x\in\Xcal\right\}$ under the inner product determined by
$\langle\phi(\bs x),\phi(\bs y)\rangle_{\Hscr}=\Kcal(\bs x,\bs y).$ The element $\phi(\bs x)$ is the unique Riesz representer of the
evaluation functional at $\bs x\in\Xcal$, and $\Kcal$ is called the
\emph{reproducing kernel} of $\Hscr$. The \emph{reproducing property} states 
\begin{equation*}
    f(\bs x)=
    \langle f,\phi(\bs x)\rangle_{\Hscr}
    \qquad
    \text{for every }f\in\Hscr,\ \bs x\in\Xcal.
\end{equation*}
Let $\Xcal \subset \Rbb^d$ be compact, which is the closure of its nonempty interior $\operatorname{int}(\Xcal)$. Then, $\Kcal\colon\Xcal\times \Xcal\to\Rbb$ is called a \emph{Mercer kernel} if it is a continuous kernel function. It is well-known that the RKHS induced by a Mercer kernel is separable; see
\citet[Proposition~11.7]{pau_rag_16}. 

Define
$\Acal_s\isdef\{\bs\alpha\in\Nbb_0^d:|\bs\alpha|\le s\}$. The following 
result from \citet{zhou2008derivative} gives the key tool for the derivative reproducing property.
\begin{theorem}[{\citet[Theorem~1]{zhou2008derivative}}]
\label{result:zhou}
Let $s\in\Nbb$, and let
$\Kcal\colon\Xcal\times\Xcal\to\Rbb$ be a Mercer kernel satisfying
$\Kcal\in\Cscr^{2s}(\Xcal\times\Xcal)$. Then:
\begin{enumerate}
    \item For any $\bs x\in\Xcal$ and
    $\bs\alpha\in\Acal_s$, $ \phi^{(\bs\alpha)}(\bs x)
        \isdef
        D^{\bs\alpha}\Kcal(\bs x,\cdot)
        \in\Hscr.$
    \item A partial derivative reproducing property holds true for $\bs \alpha \in \Acal_s$:
    \begin{equation*}
        f^{(\bs\alpha)}(\bs x)
        =
        \left\langle
            f,\phi^{(\bs\alpha)}(\bs x)
        \right\rangle_{\Hscr}
        \qquad
        \text{for every }f\in\Hscr,\ \bs x\in\Xcal.
    \end{equation*}

    \item The inclusion
    $\Hscr\hookrightarrow\Cscr^s(\Xcal)$ is well-defined and compact,
    with
    \begin{equation*}
        \|f\|_{\Cscr^s(\Xcal)}
        \le
        \sqrt{
            d^s
            \|\Kcal\|_{\Cscr^{2s}(\Xcal\times\Xcal)}
        }\,
        \|f\|_{\Hscr}
        \qquad
        \text{for every }f\in\Hscr.
    \end{equation*}
\end{enumerate}
\end{theorem}

\subsection{Measures of dependence}\label{subsec:dependence_measures}

We define the dependence notions used in this article. For more background, refer to \citet{Rosenblatt1956,Bosq2000,Bradley2005,Dedecker2007,BlanchardZadorozhnyi2019}. All stochastic processes are indexed by $\Zbb$. 

\begin{definition}[Strict stationarity]\label{def:strict_stationarity}
Let $\Xscr\isdef(X_t)_{t\in\Zbb}$ be a stochastic process taking values
in a measurable space. The process $\Xscr$ is \emph{strictly stationary}
if, for every $n\ge1$, every $t_1,\ldots,t_n\in\Zbb$, and every lag
$h\in\Zbb$, $ \left(X_{t_1+h},\ldots,X_{t_n+h}\right)
    \overset{d}{=}
    \left(X_{t_1},\ldots,X_{t_n}\right).$ Equivalently, all finite-dimensional distributions are invariant under time
shifts.
\end{definition}
\noindent For a process $\Xscr=(X_t)_{t\in\Zbb}$, define $ \Fscr_{-\infty}^{j}(\Xscr)\isdef\sigma(X_t:t\le j),$ and $\Fscr_{j}^{\infty}(\Xscr)\isdef\sigma(X_t:t\ge j).$
Given two $\sigma$-algebras $\Ascr$ and $\Gscr$, define the
\emph{$\alpha$-mixing coefficient} by
\begin{equation*}
    \alpha(\Ascr,\Gscr)
    \isdef
    \sup_{A\in\Ascr,\,G\in\Gscr}
    \left|
        \Pbb(A\cap G)-\Pbb(A)\Pbb(G)
    \right|.
\end{equation*}

\begin{definition}[Strong mixing]\label{def:strong_mixing}
Let $\Xscr\isdef(X_t)_{t\in\Zbb}$ be a strictly stationary stochastic
process. We call $\Xscr$ \emph{strongly mixing} if
$\alpha_\Xscr(m)\to0$ as $m\to\infty$, where
\begin{equation}\label{eq:strong_mixing_coefficients}
    \alpha_\Xscr(m)
    \isdef
    \alpha\!\left(
        \Fscr_{-\infty}^{0}(\Xscr),
        \Fscr_{m}^{\infty}(\Xscr)
    \right),
    \qquad m\in\Nbb.
\end{equation}
\end{definition}
\noindent Let $\Xcal$ be a subset of a Banach space $(\Bcal,\|\cdot\|)$. Denote
by $\Cscr_b^{0,1}(\Xcal)$ the space of bounded real-valued Lipschitz
functions on $\Xcal$. For $f\in\Cscr_b^{0,1}(\Xcal)$, set
\begin{equation*}
    |f|_{\Cscr^{0,1}}
    \isdef
    \sup_{\substack{x,y\in\Xcal\\x\ne y}}
    \frac{|f(x)-f(y)|}{\|x-y\|},
    \qquad
    \|f\|_{\Cscr^{0,1}}
    \isdef
    \|f\|_\infty+|f|_{\Cscr^{0,1}}.
\end{equation*}
Define the unit Lipschitz class $\Cscr_1^{0,1}(\Xcal)\isdef\left\{
f\in\Cscr_b^{0,1}(\Xcal):|f|_{\Cscr^{0,1}}\le1\right\}.$

\begin{definition}[$\tau$-mixing]\label{def:tau_mixing}
Let $\Xscr\isdef(X_t)_{t\in\Zbb}$ be a strictly stationary stochastic
process taking values in $\Xcal$. We say that $\Xscr$ is
\emph{$\tau$-mixing} if $\tau_\Xscr(m)\to0$ as $m\to\infty$, where
\begin{equation*}
    \tau_\Xscr(m)
    \isdef
    \sup_{f\in\Cscr_1^{0,1}(\Xcal)}
    \left\|
        \Ebb\!\left[
            f(X_m)
            \,\middle|\,
            \Fscr_{-\infty}^{0}(\Xscr)
        \right]
        -
        \Ebb\!\left[f(X_m)\right]
    \right\|_\infty,
    \qquad m\in\Nbb.
\end{equation*}
\end{definition}
\noindent If $\tau_\Xscr(m)\le cm^{-\gamma},$ for some constants $c\ge0$, $\gamma>0$ and all $m \in \Nbb$, then $\Xscr$ is said to
be \emph{polynomially $\tau$-mixing}. The $\tau$-mixing property is preserved under $1$-Lipschitz maps, as follows directly from the definition.


\section{Proofs for Section~\ref{sec:framework}}
\label{appendix:proofs_framework}

\begin{lemma}[Payoff-operator properties]
\label{lemma:payoff_operator}
Under the standing payoff and RKHS payoff conditions, $J_{\Pbb,\Hscr}$ is a Hilbert--Schmidt operator whose adjoint is given by \eqref{eq:payoff_adjoint}, while $\Sig=J_{\Pbb,\Hscr}^\ast J_{\Pbb,\Hscr}$ is positive and trace-class.
\end{lemma}

\begin{proof}
For any $h \in \Hscr$, \eqref{eq:score_representer} and Cauchy--Schwarz gives 
\begin{equation*}
    \|J_{\Pbb,\Hscr}h\|_{L^2_\Pbb}^2
    =
    \Ebb^\Pbb\!\left[\langle h,\psi(Z)\rangle_\Hscr^2\right]
    \leq
    \Ebb^\Pbb\!\left[\|\psi(Z)\|_\Hscr^2\right]\|h\|_\Hscr^2,
\end{equation*}
so $J_{\Pbb,\Hscr}$ is bounded. Since $\Hscr$ is separable and $\psi$ is Borel measurable, $\psi(Z)$ is strongly measurable by the Pettis measurability theorem \citep[Theorem~II.1.2]{diestel1977vector}; hence so is $U\psi(Z)$ for every real random variable $U$. For $U\in L^2_\Pbb$, $ \Ebb^\Pbb\!\left[|U|\,\|\psi(Z)\|_\Hscr\right]\leq\|U\|_{L^2_\Pbb}\left\{\Ebb^\Pbb\!\left[\|\psi(Z)\|_\Hscr^2\right]\right\}^{1/2}
<\infty,$ which makes the Bochner integral in \eqref{eq:payoff_adjoint} well-defined. In particular, taking $U=Y$ shows that $\mu=\Ebb^\Pbb[Y\psi(Z)]$ is well-defined. Moreover, $   \langle J_{\Pbb,\Hscr}h,U\rangle_{L^2_\Pbb}=\Ebb^\Pbb\!\left[U\langle h,\psi(Z)\rangle_\Hscr\right]=\left\langle h,\Ebb^\Pbb[U\psi(Z)]\right\rangle_\Hscr,$
which proves the adjoint formula. Now, $\psi(Z)\otimes\psi(Z)$ is positive and therefore, 
\begin{equation*}
    \tr\left(|\psi(Z) \otimes \psi(Z)|\right) = \tr\left(\psi(Z)\otimes \psi(Z)\right) = \|\psi(Z)\|_\Hscr^2.
\end{equation*}
The map $f\mapsto f\otimes f$ from $\Hscr$ into the Banach space $(\mathsf{T}(\Hscr),\tr(|\cdot|))$ is continuous, since $\tr(|f\otimes f-g\otimes g|)\leq(\|f\|_\Hscr+\|g\|_\Hscr)\|f-g\|_\Hscr$. Hence $\psi(Z)\otimes\psi(Z)$ is Borel measurable with separable range in $\mathsf{T}(\Hscr)$, and therefore strongly measurable. By the standing assumption, $\Ebb^\Pbb[\tr(|\psi(Z)\otimes\psi(Z)|)]=\Ebb^\Pbb[\|\psi(Z)\|_\Hscr^2]<\infty$, so $\Sig=\Ebb^\Pbb[\psi(Z)\otimes\psi(Z)]$ is a Bochner integral in $\mathsf{T}(\Hscr)$, hence a trace-class operator. Since $A\mapsto\langle Ah,h\rangle_\Hscr$ and $A\mapsto\tr(A)$ are bounded linear functionals on $\mathsf{T}(\Hscr)$, they commute with the Bochner integral: $\langle\Sig h,h\rangle_\Hscr=\Ebb^\Pbb[\langle h,\psi(Z)\rangle_\Hscr^2]\geq0$, so $\Sig$ is positive, and $ \tr(\Sig)=\Ebb^\Pbb[\|\psi(Z)\|_\Hscr^2]<\infty.$ Finally, since $\Sig=J_{\Pbb,\Hscr}^\ast J_{\Pbb,\Hscr}$ (see \eqref{eq:operator_definitions}), and $\Sig$ is trace-class, it follows that $J_{\Pbb,\Hscr}$ is Hilbert--Schmidt.
\end{proof}

\begin{proof}[Proof of Proposition~\ref{prop:reg_consistency}]
\emph{1.} The definition of $\Rscr_\Hscr^\star$ in \eqref{eq:restricted-payoff-space} implies the orthogonality relation $ Y-\Rscr^\star_\Hscr\perp\overline{\operatorname{Im}(J_{\Pbb,\Hscr})} = \operatorname{Ker}(J_{\Pbb,\Hscr}^\ast)^{\perp}$, that is, $J_{\Pbb,\Hscr}^\ast Y=J_{\Pbb,\Hscr}^\ast\Rscr^\star_\Hscr.$ Since $J_{\Pbb,\Hscr}$ is compact, $\|J_{\Pbb,\Hscr}h_\lambda-\Rscr^\star_\Hscr\|_{L^2_\Pbb} \to 0$ as $\lambda\to0$, following the same arguments as in the proof of \citet[Lemma 2.3]{Boudabsa2021}.

\emph{2.} Suppose that the RKHS-projected payoff is attained. Using $J_{\Pbb,\Hscr}^\ast Y=J_{\Pbb,\Hscr}^\ast\Rscr^\star_\Hscr$ in \eqref{eq:population_solution} gives $h_\lambda=(\Sig+\lambda I_\Hscr)^{-1}J_{\Pbb,\Hscr}^\ast\Rscr^\star_\Hscr.$
Since $\Rscr^\star_\Hscr\in\operatorname{Im}(J_{\Pbb,\Hscr})$ and
$h_0$ is the minimum-norm target, the conclusion follows from
\citet[Theorem~5.2]{EnglHankeNeubauer1996} applied with exact data.

\emph{3.} The asserted rates in \eqref{eq:source_rates} follow from \citet[Theorem~4.3]{EnglHankeNeubauer1996} specialized to the Tikhonov penalty in \citet[Example~4.15]{EnglHankeNeubauer1996}.

\end{proof}


\section{Proofs for Section~\ref{sec:stat_prop}}
\label{appendix:proofs_stat_prop}

\paragraph{Derivation of the finite-dimensional representation.}
Let $\bs a \in \Rbb^T$. From the expressions of $\widehat J_T$ and $\widehat J_T^\ast$ in \eqref{eq:empirical_payoff_operator} and \eqref{eq:empirical_payoff_adjoint}, 
\begin{equation*}
    \bigl[\widehat J_T\widehat J_T^\ast\bs a\bigr]_t
    =
    \frac1T\sum_{s=1}^T
    a_s\langle\psi_s,\psi_t\rangle_\Hscr
    =
    \left[\frac1T\bs K_T\bs a\right]_t,
\end{equation*}
where the symmetry of the Gram matrix is used in the second equality.
Thus, $\widehat J_T\widehat J_T^\ast=T^{-1}\bs K_T$.
Using the push-through identity in \eqref{eq:empirical_solution},
\begin{align*}
    \widehat h_{\lambda,T}
    &= \left(\widehat J_T^\ast \widehat J_T + \lambda I_\Hscr\right)^{-1} \widehat J_T^\ast \bs y_T
    \\
    &= \widehat J_T^\ast
    \bigl(
        \widehat J_T\widehat J_T^\ast+\lambda \bs I_T
    \bigr)^{-1}
    \bs y_T = 
    \frac1{\sqrt T}
    \sum_{t=1}^T
    \left[
        \left(
            \frac1T\bs K_T+\lambda \bs I_T
        \right)^{-1}
        \bs y_T
    \right]_t
    \psi_t=
    \sum_{t=1}^T\widehat b_{\lambda,t}\psi_t,
\end{align*}
where, using $\bs y=\sqrt T\,\bs y_T$, we obtain $ \widehat{\bs b}_{\lambda,T}=
(\bs K_T+T\lambda \bs I_T)^{-1}\bs y.$ Since $\bs K_T$ is positive semidefinite and $\lambda>0$,
$\bs K_T+T \lambda \bs I_T$ is positive definite. Hence, 
$\widehat{\bs b}_{\lambda,T}$ is equivalently the unique minimizer of $  \tfrac12
\bs b^\top(\bs K_T+T\lambda \bs I_T)\bs b-\bs y^\top\bs b,$ which proves \eqref{eq:coefficient_quadratic_problem}.

\paragraph{Finite-sample estimation error.}
The main arguments rely on the concentration inequality of
\citet[Corollary~3.7]{BlanchardZadorozhnyi2019}, stated in terms of the effective sample size defined in their equation~(3.2), and
separate derivations of high-probability bounds on a statistical fluctuation term and a stability event. From the expression of the estimation error in \eqref{eq:estimation_error}, the quantity of interest is the difference $\widehat h_{\lambda, T} - h_\lambda$. Using \eqref{eq:solution_characterizations}, the difference equals 
\begin{equation*}
    \widehat h_{\lambda, T} - h_\lambda = \widehat \Sig_{\lambda, T}^{-1}\widehat \mu_{T} - \Sig_\lambda^{-1} \mu =\widehat\Sig_{\lambda, T}^{-1}\big(\widehat \mu_{T} - \mu\big) + \left(\widehat \Sig_{\lambda, T}^{-1} -\Sig_\lambda^{-1}\right)\mu.
\end{equation*}
Applying the identity $A^{-1} - B^{-1} = A^{-1}(B-A)B^{-1}$ for invertible operators, we obtain 
\begin{equation}\label{eq:resolvent_identity}
    \widehat\Sig_{\lambda, T}^{-1} -\Sig_\lambda^{-1} = -\widehat \Sig_{\lambda, T}^{-1}\big(\widehat \Sig_T - \Sig\big)\Sig_\lambda^{-1}.
\end{equation}
Finally, since $\mu = \Sig_{\lambda} h_\lambda$, we obtain the decomposition
\begin{equation}\label{eq:difference}
    \widehat h_{\lambda,T}-h_\lambda
    =
    \widehat\Sig_{\lambda,T}^{-1}
    \Bigl(
        (\widehat\mu_T-\mu)
        -
        (\widehat\Sig_T-\Sig)h_\lambda
    \Bigr).
\end{equation}
Therefore, 
\begin{align*}
    \Sig_\lambda^{1/2} \big( \widehat h_{\lambda,T}-h_\lambda\big) \; &= \;  \Sig_\lambda^{1/2}\widehat\Sig_{\lambda,T}^{-1}
    \Bigl((\widehat\mu_T-\mu)-(\widehat\Sig_T-\Sig)h_\lambda\Bigr) \\
    &= \; \Big(\Sig_\lambda^{1/2} \widehat \Sig_{\lambda,T}^{-1} \Sig_{\lambda}^{1/2}\Big) \Sig_{\lambda}^{-1/2}  \Bigl((\widehat\mu_T-\mu)-(\widehat\Sig_T-\Sig)h_\lambda\Bigr).
\end{align*}
By the submultiplicative property of the norm,
\begin{equation}\label{eq:error_decomposition}
    \left\| \Sig_\lambda^{1/2} \big( \widehat h_{\lambda,T}-h_\lambda\big)\right\|_{\Hscr} \leq \left\|\Sig_\lambda^{1/2} \widehat \Sig_{\lambda,T}^{-1} \Sig_{\lambda}^{1/2}\right\|_{\op} \, \left\|\Sig_{\lambda}^{-1/2}  \Bigl((\widehat\mu_T-\mu)-(\widehat\Sig_T-\Sig)h_\lambda\Bigr)\right\|_{\Hscr}.
\end{equation}
The first term on the right side determines a stability event, while the second is a whitened statistical fluctuation. To control these, for $\lambda>0$, we set 
\begin{equation}\label{eq:Delta_lambda}
    \Delta(\lambda,T)
    \isdef
    \Sig_\lambda^{-1/2}
    (\widehat\Sig_T-\Sig)
    \Sig_\lambda^{-1/2} \in \mathsf{HS}(\Hscr), 
\end{equation}
 and define the following two processes:
\begin{equation}\label{eq:appendix_whitened_processes}
\begin{aligned}
    \xi_t(\lambda)
    &\isdef
    \Sig_\lambda^{-1/2}
    \Bigl(\big(Y_t\psi_t-\mu)-(\psi_t\otimes\psi_t-\Sig\big)h_\lambda\Bigr)
    \in\Hscr,
    \\
    Z_t(\lambda)
    &\isdef
    \Sig_\lambda^{-1/2}
    \big(\psi_t\otimes\psi_t-\Sig\big)
    \Sig_\lambda^{-1/2}
    \in\mathsf{HS}(\Hscr).
\end{aligned}
\end{equation}
We first show the following lemma, recalling that $R_\lambda \isdef \kappa_Y + \kappa \|h_\lambda\|_\Hscr$.

\begin{lemma}[Process-specific bounds]\label{lemma:finite_sample_process_constants}
Under Assumptions~\ref{assumption:boundedness}--\ref{assumption:tau_mixing}, the processes in \eqref{eq:appendix_whitened_processes} are centered, strictly stationary and $\tau$-mixing. Moreover, they satisfy the $\Pbb$-almost sure bounds (left) and second-moment bounds (right)
\begin{equation*}
\begin{alignedat}{3}
    \|\xi_t(\lambda)\|_\Hscr
    &\leq\frac{2\kappa R_\lambda}{\sqrt\lambda}, &\qquad \Ebb^\Pbb\|\xi_t(\lambda)\|_\Hscr^2
    &\leq R_\lambda^2\Ncal(\lambda)\\
    \|Z_t(\lambda)\|_{\HS}
    &\leq\frac{2\kappa^2}{\lambda}, &\qquad  \Ebb^\Pbb\|Z_t(\lambda)\|_{\HS}^2
    &\leq\frac{\kappa^2\Ncal(\lambda)}{\lambda}.
\end{alignedat}
\end{equation*}
Their $\tau$-mixing coefficients satisfy
\begin{equation}\label{eq:appendix_transformed_tau_rates}
\begin{aligned}
    \tau_{\xi}(m)\leq
    \frac{(\kappa+2R_\lambda)c_\tau}{\sqrt\lambda}
    m^{-\gamma_\tau}, \qquad
    \tau_{Z}(m)\leq
    \frac{2\kappa c_\tau}{\lambda}
    m^{-\gamma_\tau}.
\end{aligned}
\end{equation}
\end{lemma}

\begin{proof}
\emph{Step 1:} Set $W_t(\lambda)
    \isdef
    \Sig_\lambda^{-1/2}\psi_t
    \bigl(Y_t-\langle h_\lambda,\psi_t\rangle_\Hscr\bigr).$ Since $\mu=\Sig_\lambda h_\lambda$, $\xi_t(\lambda)=W_t(\lambda)-\Ebb^\Pbb[ W_t(\lambda)]$ and thus, is centered. From Assumption~\ref{assumption:boundedness} and $\|\Sig_\lambda^{-1/2}\|_{\op} \leq 1/\sqrt{\lambda}$,
\begin{equation*}
    \|\Sig_\lambda^{-1/2}\psi_t\|_\Hscr\leq\frac\kappa{\sqrt\lambda},
    \qquad
    |Y_t-\langle h_\lambda,\psi_t\rangle_\Hscr|\leq |Y_t| + \|\psi_t\|_\Hscr \|h_\lambda\|_\Hscr \leq R_\lambda.
\end{equation*}
Thus, $\|W_t(\lambda)\|_\Hscr\leq\kappa R_\lambda/\sqrt\lambda$. Jensen's inequality gives the same bound for $\|\Ebb^\Pbb W_t(\lambda)\|_\Hscr$, and hence, the triangle inequality proves the almost sure bound for $\|\xi_t(\lambda)\|_\Hscr$. Centering reduces the second moment, and therefore $   \Ebb^\Pbb\|\xi_t(\lambda)\|_\Hscr^2
    \leq
    R_\lambda^2\Ebb^\Pbb\|\Sig_\lambda^{-1/2}\psi_t\|_\Hscr^2
    =
    R_\lambda^2\Ncal(\lambda).$
    
\emph{Step 2:} Set $a_t(\lambda)\isdef\Sig_\lambda^{-1/2}\psi_t$. Then $   Z_t(\lambda)=
    a_t(\lambda)\otimes a_t(\lambda)-
    \Ebb^\Pbb[a_t(\lambda)\otimes a_t(\lambda)].$ It is centered, and $\|a_t(\lambda)\|_\Hscr\leq\kappa/\sqrt\lambda$ from the arguments of Step 1. Since $\|a\otimes a\|_{\HS}=\|a\|_\Hscr^2$, the application of Jensen and triangle inequalities gives the almost sure bound for $\|Z_t(\lambda)\|_{\HS}$. Moreover, $ \Ebb^\Pbb\|Z_t(\lambda)\|_{\HS}^2 \leq
    \Ebb^\Pbb\|a_t(\lambda)\|_\Hscr^4 \leq
    \frac{\kappa^2}{\lambda}
    \Ebb^\Pbb\|a_t(\lambda)\|_\Hscr^2  =
    \frac{\kappa^2\Ncal(\lambda)}{\lambda}$.

\emph{Step 3: stationarity and inherited mixing rates.}
Strict stationarity follows because both processes are fixed measurable transformations of
$\zeta_t=(\psi_t,Y_t)$. It remains to show how the $\tau$-mixing rates are inherited from that of $\zeta_t$ under Assumption~\ref{assumption:tau_mixing}. Toward that end, by a slight abuse of notation, for $U\in\{\xi,Z\}$ we write $U(\zeta_t;\lambda)=U_t(\lambda)$ to denote the corresponding measurable transformation. The inherited $\tau$-mixing bound of the transformed processes can be obtained by multiplying the coefficients of the original process by the Lipschitz constant of the transformation, see \citet[Page 869]{Dedecker2004}. Since both $\xi_t(\lambda)$ and $Z_t(\lambda)$ are centered processes, the centering terms cancel when taking differences. Consequently, it suffices to study the uncentered maps $W_t(\lambda)$ and $a_t(\lambda)\otimes a_t(\lambda)$ introduced in Steps 1 and 2. Under Assumption~\ref{assumption:boundedness}, it holds, 
\begin{align*}
  \|\xi(\zeta_t;\lambda) - \xi(\zeta_s;\lambda)\|_\Hscr  &= \left\|W(\zeta_t;\lambda) - W(\zeta_s;\lambda)\right\|_\Hscr  \\
  &= \left\|\Sig_{\lambda}^{-1/2}\Big(\psi_t\big(Y_t - \langle h_\lambda, \psi_t\rangle_\Hscr\big) - \psi_s\big(Y_s - \langle h_\lambda, \psi_s\rangle_\Hscr\big)\Big)\right\|_\Hscr \\
   &\leq \frac{1}{\sqrt{\lambda}}\left\|\psi_t\big(Y_t - \langle h_\lambda, \psi_t\rangle_\Hscr\big) - \psi_s\big(Y_s - \langle h_\lambda, \psi_s\rangle_\Hscr\big)\right\|_\Hscr \\
  &= \frac{1}{\sqrt{\lambda}}\left\|(\psi_t - \psi_s)\big(Y_t - \langle h_\lambda, \psi_t\rangle_\Hscr\big) + \psi_s\big(Y_t -Y_s - \langle h_\lambda, \psi_t - \psi_s\rangle_\Hscr\big)\right\|_\Hscr \\
  &\leq \frac{1}{\sqrt{\lambda}}\Big(R_\lambda \|\psi_t - \psi_s\|_\Hscr + \kappa\big(|Y_t - Y_s| + \|h_\lambda\|_\Hscr \|\psi_t - \psi_s\|_\Hscr\big)\Big) \\
 &\leq \frac{1}{\sqrt{\lambda}}\left(2R_\lambda \|\psi_t - \psi_s\|_\Hscr + \kappa|Y_t - Y_s| \right) \\
 &\leq \frac{\kappa + 2R_\lambda}{\sqrt{\lambda}} \|\zeta_t - \zeta_s\|_{\Hscr \oplus \Rbb}.
\end{align*}
Hence, $\xi(\,\cdot\,;\lambda)$ is Lipschitz on this state space with constant
$(\kappa+2R_\lambda)/\sqrt{\lambda}$. Similarly,
\begin{align*}
    \left\|
        Z(\zeta_t;\lambda)-Z(\zeta_s;\lambda)
    \right\|_{\HS}
    &= \left\|a(\zeta_t;\lambda) \otimes a(\zeta_t;\lambda) - a(\zeta_s;\lambda) \otimes a(\zeta_s;\lambda)\right\|_{\HS} \\
    &= \|\Sig_\lambda^{-1/2}(\psi_t \otimes \psi_t - \psi_s \otimes \psi_s) \Sig_\lambda^{-1/2}\|_{\HS} \\
    &\leq \frac{1}{\lambda} \|\psi_t \otimes(\psi_t - \psi_s) + (\psi_t - \psi_s)\otimes \psi_s\|_{\HS}\\
    &\leq \frac{\|\psi_t\|_\Hscr+\|\psi_s\|_\Hscr}{\lambda}
    \|\psi_t-\psi_s\|_\Hscr\\
    &\leq
    \frac{2\kappa}{\lambda}
    \|\zeta_t-\zeta_s\|_{\Hscr\oplus\Rbb}.
\end{align*}
Thus, $Z(\,\cdot\,;\lambda)$ is Lipschitz with constant $2\kappa/\lambda$.
Finally, the Lipschitz-transformation property of $\tau$-mixing therefore gives
\begin{align*}
    \tau_\xi(m)
    &\leq
    \frac{\kappa+2R_\lambda}{\sqrt{\lambda}}
    \tau_\zeta(m)
    \leq
    \frac{(\kappa+2R_\lambda)c_\tau}{\sqrt{\lambda}}
    m^{-\gamma_\tau},\\
    \tau_Z(m)
    &\leq
    \frac{2\kappa}{\lambda}
    \tau_\zeta(m)
    \leq
    \frac{2\kappa c_\tau}{\lambda}
    m^{-\gamma_\tau}.
\end{align*}
This proves \eqref{eq:appendix_transformed_tau_rates}.
\end{proof}

\begin{proposition}[Stability of the empirical second-moment operator]
\label{prop:effective_sample_stability}
Under Assumptions~\ref{assumption:boundedness}--\ref{assumption:tau_mixing}, for any $\lambda>0$,  $\delta\in(0,1)$, and $T\geq T_0(\lambda,\delta)$ where $T_0(\lambda,\delta)$ is defined in \eqref{eq:T_0}, $\|\Sig_\lambda^{1/2}\widehat\Sig_{\lambda,T}^{-1}
\Sig_\lambda^{1/2}\|_{\op}\leq 2$
holds with probability at least $1-\delta/2$.\footnote{As in Section~\ref{subsec:finite_sample_learning}, we exclude the degenerate case
$\Sig=0$, for which the stability bound holds trivially
without any sample-size threshold.}
\end{proposition}

\begin{proof}
\emph{Step 1:}  Consider $\Delta(\lambda, T)$ from \eqref{eq:Delta_lambda}. It can be shown that if $\|\Delta(\lambda,T)\|_{\op} < 1$, then $\|\Sig_\lambda^{1/2}\widehat\Sig_{\lambda,T}^{-1}\Sig_\lambda^{1/2}\|_{\op} \leq\big(1 - \|\Delta(\lambda,T)\|_{\op}\big)^{-1}$, see \citet[Lemma~5.2.3, page 592]{hsu14}. Therefore, for $\|\Delta(\lambda,T)\|_{\op} \leq 1/2$, it follows that  $\|\Sig_\lambda^{1/2}\widehat\Sig_{\lambda,T}^{-1}\Sig_\lambda^{1/2}\|_{\op} \leq 2$.

\emph{Step 2:} We now verify that the above event holds with probability at least $1 - \delta/2$ under the specified lower bound on the sample size $T$. From the definition of $Z_t(\lambda)$ in \eqref{eq:appendix_whitened_processes}, we have $ \Delta(\lambda,T)=\frac{1}{T}\sum_{t=1}^T Z_t(\lambda).$ By Lemma~\ref{lemma:finite_sample_process_constants},
$\bigl(Z_t(\lambda)\bigr)_{t\in\mathbb Z}$ is a centered, strictly
stationary, $\tau$-mixing process in the separable Hilbert space
$\mathsf{HS}(\Hscr)$ satisfying
\begin{equation*}
    \|Z_t(\lambda)\|_{\HS}
    \leq
    \frac{2\kappa^2}{\lambda},
    \qquad
    \Ebb^\Pbb\|Z_t(\lambda)\|_{\HS}^2
    \leq
    \frac{\kappa^2\Ncal(\lambda)}{\lambda}, \qquad  \tau_Z(m)
    \leq
    \frac{2\kappa c_\tau}{\lambda}
    m^{-\gamma_\tau}.
\end{equation*}
Applying the concentration result of \citet[Corollary~3.7]{BlanchardZadorozhnyi2019} with
failure probability $\delta/2$ gives, with probability at least
$1-\delta/2$,
\begin{equation}\label{eq:appendix_stability_raw_bound}
\begin{aligned}
    \|\Delta(\lambda,T)\|_{\op}
    \leq
    \|\Delta(\lambda,T)\|_{\HS}
    \leq
    \log\!\left(\frac{4}{\delta}\right)
    \left[
        13\kappa
        \sqrt{
            \frac{\Ncal(\lambda)}
            {\lambda\ell_Z(\lambda,T)}}
        +\frac{42\kappa^2}
        {\lambda\ell_Z(\lambda,T)}
    \right].
\end{aligned}
\end{equation}
Here $\ell_Z(\lambda,T)$ is the \emph{effective sample size} of
$\bigl(Z_t(\lambda)\bigr)$ defined in
\citet[equation~(3.2)]{BlanchardZadorozhnyi2019} with $n=T$: the
largest integer $1\leq\ell\leq T$ such that
\begin{equation*}
    \tau_Z\bigl(\lfloor T/\ell\rfloor\bigr)
    \leq
    \frac{2\kappa^2}{\lambda\ell}
    \vee
    \kappa\sqrt{\frac{\Ncal(\lambda)}{\lambda\ell}},
\end{equation*}
or $1$ if no such integer exists. We bound $\ell_Z(\lambda,T)$ from below
directly from this definition. Set
$A_Z(\lambda,T)\isdef\bigl(\sqrt{\lambda\Ncal(\lambda)}/(2c_\tau)\bigr)^{\frac{2}{2\gamma_\tau+1}}(T/2)^{\frac{2\gamma_\tau}{2\gamma_\tau+1}}$,
and let $\ell$ be an integer with
$1\leq\ell\leq\min\{A_Z(\lambda,T),T\}$. Since $\ell\leq T$, we have
$\lfloor T/\ell\rfloor\geq1$ and $\lfloor T/\ell\rfloor\geq T/(2\ell)$.
Hence, the $\tau$-mixing bound above gives
\begin{equation*}
    \tau_Z\bigl(\lfloor T/\ell\rfloor\bigr)
    \leq
    \frac{2\kappa c_\tau}{\lambda}
    \left(\frac{2\ell}{T}\right)^{\gamma_\tau}
    \leq
    \kappa\sqrt{\frac{\Ncal(\lambda)}{\lambda\ell}},
\end{equation*}
where the last inequality is equivalent to
$\ell^{(2\gamma_\tau+1)/2}\leq\bigl(\sqrt{\lambda\Ncal(\lambda)}/(2c_\tau)\bigr)(T/2)^{\gamma_\tau}$,
that is, to $\ell\leq A_Z(\lambda,T)$. Thus, every such $\ell$ belongs
to the set in \citet[equation~(3.2)]{BlanchardZadorozhnyi2019}. Taking
$\ell=\min\{\lfloor A_Z(\lambda,T)\rfloor,T\}$, and using
$\ell_Z(\lambda,T)\geq1$ when this minimum is zero, we obtain
\begin{equation}\label{eq:appendix_Z_effective_sample}
    \ell_Z(\lambda,T)
    \geq
    \min\left\{
    \left\lfloor
        \left(
            \frac{\sqrt{\lambda\Ncal(\lambda)}}
            {2c_\tau}
        \right)^{\frac{2}{2\gamma_\tau+1}}
        \left(\frac{T}{2}\right)^{
            \frac{2\gamma_\tau}{2\gamma_\tau+1}
        }
    \right\rfloor,\;
    T
    \right\}.
\end{equation}
Since $\frac{R_\lambda}{\kappa+2R_\lambda}\leq
    \frac{1}{2},$ we have $ \frac{\sqrt{\lambda\Ncal(\lambda)}}{2c_\tau}
    \geq
    \frac{R_\lambda\sqrt{\lambda\Ncal(\lambda)}
    }{(\kappa+2R_\lambda)c_\tau}.$ Since the right-hand side of
\eqref{eq:appendix_Z_effective_sample} is nondecreasing in this ratio, it follows from the definition of $\ell^\star(\lambda,T)$ in \eqref{eq:process_effective_sample_sizes} that $  \ell_Z(\lambda,T)
    \geq\ell^\star(\lambda,T).$ Next, the condition $T\geq T_0(\lambda,\delta)$ implies
$\ell^\star(\lambda,T)\geq\ell_0(\lambda,\delta)$. Indeed, the first term of the maximum in \eqref{eq:T_0} gives
$(T/2)^{\frac{2\gamma_\tau}{2\gamma_\tau+1}}\geq\lceil\ell_0(\lambda,\delta)\rceil\bigl(R_\lambda\sqrt{\lambda\Ncal(\lambda)}/((\kappa+2R_\lambda)c_\tau)\bigr)^{-\frac{2}{2\gamma_\tau+1}}$,
so that the integer part in \eqref{eq:process_effective_sample_sizes} is at least $\lceil\ell_0(\lambda,\delta)\rceil$, while the second term gives $T\geq\lceil\ell_0(\lambda,\delta)\rceil$. Hence,
$\ell^\star(\lambda,T)\geq\lceil\ell_0(\lambda,\delta)\rceil\geq\ell_0(\lambda,\delta)$; in particular, $\ell^\star(\lambda,T)\geq1$ because $\ell_0(\lambda,\delta)>0$. Since the right-hand side of
\eqref{eq:appendix_stability_raw_bound} is decreasing in the effective
sample size, we may therefore replace $\ell_Z(\lambda,T)$ by
$\ell^\star(\lambda,T)$. Now, set $L \isdef \log(4/\delta)$ and $q\isdef \max\{\Ncal(\lambda), 1\}$. By the definition of $\ell_0(\lambda,\delta)$ from \eqref{eq:effective_sample_size_thresholds} and the inequality
$\ell^\star(\lambda,T)\geq\ell_0(\lambda,\delta)$,
\begin{equation*}
    13\kappa L
    \sqrt{
        \frac{\Ncal(\lambda)}
        {\lambda\ell^\star(\lambda,T)}
    }
    \leq
    \frac{13}{30}
    \sqrt{\frac{\Ncal(\lambda)}{q}}
    \leq
    \frac{13}{30},
\end{equation*}
and
\begin{equation*}
    \frac{
        42\kappa^2L
    }{
        \lambda\ell^\star(\lambda,T)
    }
    \leq
    \frac{42}{900qL}
    \leq
    \frac{42}{900},
\end{equation*}
where the last inequality uses $q\geq1$ and
$L>\log 4>1$. Hence,
\begin{equation*}
    \|\Delta(\lambda,T)\|_{\op}
    \leq
    \frac{13}{30}
    +
    \frac{42}{900}
    =
    \frac{432}{900}
    <
    \frac{1}{2}
\end{equation*}
with probability at least $1-\delta/2$. The result now follows from
Step 1.
\end{proof}

\begin{proof}[Proof of Theorem~\ref{thm:generic_finite_sample_rkhs_projection}]
By Lemma~\ref{lemma:finite_sample_process_constants},
$\bigl(\xi_t(\lambda)\bigr)_{t\in\mathbb Z}$ is a centered, strictly stationary,
$\tau$-mixing process in the separable Hilbert space $\Hscr$ satisfying
\begin{equation*}
    \|\xi_t(\lambda)\|_\Hscr
    \leq
    \frac{2\kappa R_\lambda}{\sqrt{\lambda}},
    \qquad
    \Ebb^\Pbb\|\xi_t(\lambda)\|_\Hscr^2
    \leq
    R_\lambda^2\Ncal(\lambda),
    \qquad
    \tau_\xi(m)
    \leq
    \frac{(\kappa+2R_\lambda)c_\tau}{\sqrt{\lambda}}
    m^{-\gamma_\tau}.
\end{equation*}
Let $\ell_\xi(\lambda,T)$ denote the effective sample size of this process
defined in \citet[equation~(3.2)]{BlanchardZadorozhnyi2019} with $n=T$.
The argument leading to \eqref{eq:appendix_Z_effective_sample} in the proof
of Proposition~\ref{prop:effective_sample_stability}, applied with the
second-moment bound $R_\lambda^2\Ncal(\lambda)$ and the mixing constant
$(\kappa+2R_\lambda)c_\tau/\sqrt{\lambda}$ in place of those of
$Z_t(\lambda)$, shows that $\ell_\xi(\lambda,T)$ is at least the right-hand
side of \eqref{eq:process_effective_sample_sizes}, that is,
$ \ell_\xi(\lambda,T)
    \geq
    \ell^\star(\lambda,T).$ Moreover, $\ell^\star(\lambda,T)\geq1$ for
$T\geq T_0(\lambda,\delta)$, as shown in that proof. Therefore,
\citet[Corollary~3.7]{BlanchardZadorozhnyi2019}, applied with failure
probability $\delta/2$, gives
\begin{equation*}
    \left\|\frac1T\sum_{t=1}^T\xi_t(\lambda)\right\|_\Hscr
    \leq
    R_\lambda\log\!\left(\frac4\delta\right)
    \left[
        13\sqrt{\frac{\Ncal(\lambda)}{\ell_\xi(\lambda,T)}}
        +
        \frac{42\kappa}
        {\sqrt{\lambda}\,\ell_\xi(\lambda,T)}
    \right]
\end{equation*}
with probability at least $1-\delta/2$. Since the right-hand side is
decreasing in the effective sample size and
$\ell_\xi(\lambda,T)\geq\ell^\star(\lambda,T)$, it follows that
\begin{equation}\label{eq:appendix_score_raw_bound}
    \left\|\frac1T\sum_{t=1}^T\xi_t(\lambda)\right\|_\Hscr
    \leq
    R_\lambda\log\!\left(\frac4\delta\right)
    \left[
        13\sqrt{\frac{\Ncal(\lambda)}{\ell^\star(\lambda,T)}}
        +
        \frac{42\kappa}
        {\sqrt{\lambda}\,\ell^\star(\lambda,T)}
    \right].
\end{equation}
Proposition~\ref{prop:effective_sample_stability} gives the stability event
with probability at least $1-\delta/2$. By the union bound, both events hold
with probability at least $1-\delta$. 
On their intersection, \eqref{eq:error_decomposition},
Proposition~\ref{prop:effective_sample_stability}, and
\eqref{eq:appendix_score_raw_bound} give
\begin{equation*}
\begin{aligned}
    \|\Sig_\lambda^{1/2}
    (\widehat h_{\lambda,T}-h_\lambda)\|_\Hscr
    \leq{}&
    2R_\lambda\log\!\left(\frac4\delta\right)
    \left[
        13\sqrt{\frac{\Ncal(\lambda)}
        {\ell^\star(\lambda,T)}}
        +
        \frac{42\kappa}
        {\sqrt{\lambda}\,\ell^\star(\lambda,T)}
    \right].
\end{aligned}
\end{equation*}
Finally, since $\Sig\leq\Sig_\lambda$, it follows from
\eqref{eq:estimation_error} that
\begin{equation*}
\begin{aligned}
    \|J_{\Pbb,\Hscr}\widehat h_{\lambda,T}
    -J_{\Pbb,\Hscr}h_\lambda\|_{L^2_\Pbb}=
    \|\Sig^{1/2}
    (\widehat h_{\lambda,T}-h_\lambda)\|_\Hscr
    \leq
    \|\Sig_\lambda^{1/2}
    (\widehat h_{\lambda,T}-h_\lambda)\|_\Hscr,
\end{aligned}
\end{equation*}
which proves \eqref{eq:payoff_norm_bound}.
\end{proof}

\paragraph{Effective dimension under log-polynomial spectral decay.}

\begin{proof}[Proof of Proposition~\ref{prop:effdim_conversion}]
\emph{Step 1:} From the definition of $\Ncal(\lambda)$ in \eqref{eq:effective_dimension}, we have $ \Ncal(\lambda) = \sum_{m\geq 1} \frac{\lambda_m(\Sig)}{\lambda + \lambda_m(\Sig)}$. For any integer $n\geq1$, split the sum at $n$. The first $n$
summands are bounded by one, while the remaining summands are bounded by $\lambda_m(\Sig)/\lambda$. Hence,
\begin{equation}\label{eq:effdim_split}
    \Ncal(\lambda)
    \leq
    n+
    \frac1\lambda\sum_{m>n}\lambda_m(\Sig).
\end{equation}

\emph{Step 2: eventual monotonicity.}
It remains to control the tail term in \eqref{eq:effdim_split}. By
the spectral condition in \eqref{eq:spectral_decay}, this reduces to bounding the tail of the series with
terms $m^{-p}(\log m)^q$. We will compare this series with the corresponding
improper integral. Toward that end, we first verify that the function $ f(x)\isdef x^{-p}(\log x)^q$ on $(1, \infty)$ is eventually nonincreasing. When $q=0$, $f$ is decreasing. When $q>0$, we have $f'(x) = x^{-p-1}(\log x)^{q-1}(q-p\log x)$; so, $f$ is decreasing once $\log x>q/p$. Choose an integer $n_1$ sufficiently
large that \eqref{eq:spectral_decay} holds for every $m\geq n_1$ and $f$ is nonincreasing on $[n_1,\infty)$.

\emph{Step 3: integral-test comparison and spectral-tail bound.}
For every integer $n\geq n_1$, the spectral condition
\eqref{eq:spectral_decay} and the upper-bound comparison in the integral
test give
\begin{equation}\label{eq:effdim_integral_test}
\begin{aligned}
    \sum_{m>n}\lambda_m(\Sig)
    \lesssim
    \sum_{m=n+1}^{\infty}m^{-p}(\log m)^q\leq
    \int_n^\infty x^{-p}(\log x)^q\,\mathrm dx =: I(n).
\end{aligned}
\end{equation}
The second inequality follows since $f$ is
nonnegative and nonincreasing: the area under $f$ over each interval
$[m-1,m]$ is at least $f(m)$. It remains to show that the improper integral $I(n)$ in \eqref{eq:effdim_integral_test} is finite and to determine its
order. When $q=0$, direct integration gives $ I(n)
    =\frac{n^{1-p}}{p-1}.$ Now suppose that $q>0$. We first show that $I(n)<\infty$. Set $\varepsilon\isdef(p-1)/2>0$. Since $x^{-\varepsilon}(\log x)^q$ is continuous on $[n,\infty)$ and tends to zero as $x\to\infty$, it is bounded there by a constant $C_n<\infty$. Hence, $x^{-p}(\log x)^q\leq C_n x^{-(p+1)/2}$ for $x\geq n$, and since $(p+1)/2>1$, $I(n)\leq C_n\int_n^\infty x^{-(p+1)/2}\,\mathrm dx<\infty$. Since $p>1$, $x^{1-p}(\log x)^q\to0$ as $x\to\infty$; so the boundary term at infinity vanishes. Integration by parts with
$u=(\log x)^q$ and $\mathrm dv=x^{-p}\mathrm dx$ yields
\begin{equation*}
    I(n)
    =
    \frac{n^{1-p}(\log n)^q}{p-1}
    +\frac{q}{p-1}
    \int_n^\infty x^{-p}(\log x)^{q-1}\,\mathrm dx.
\end{equation*}
For $x\geq n>1$, $ (\log x)^{q-1}
    =
    \frac{(\log x)^q}{\log x}
    \leq
    \frac{(\log x)^q}{\log n}.$ Therefore,
\begin{equation*}
    I(n)
    \leq
    \frac{n^{1-p}(\log n)^q}{p-1}
    +
    \frac{q}{(p-1)\log n}I(n).
\end{equation*}
Set $ n_2
    \isdef
    \left\lceil
    \max\left\{2,\exp\left(\frac{2q}{p-1}\right)\right\}
    \right\rceil.$ Then, for every integer $n\geq n_2$, $ \frac{q}{(p-1)\log n}\leq\frac12.$ Moving the resulting term to the left-hand side gives $ I(n) \lesssim
    n^{1-p}(\log n)^q.$
The same bound holds when $q=0$ by the preceding direct calculation.
Combining \eqref{eq:effdim_integral_test} with the above bound, we conclude that the following spectral bound holds for every $\lambda >0$ and for every integer $n \geq n_0,$ where $n_0 \isdef \max\{n_1, n_2\}$.
\begin{equation}\label{eq:spectral_bound}
    \frac{1}{\lambda}\sum_{m>n}\lambda_m(\Sig)
    \lesssim
    \lambda ^{-1} n^{1-p}(\log n)^q.
\end{equation}

\emph{Step 4:} We now choose the cutoff so that the two terms on the right-hand side of
\eqref{eq:effdim_split} have the same order. Define
\begin{equation*}
    a_\lambda
    \isdef
    \lambda^{-1/p}
    \bigl(\log(1/\lambda)\bigr)^{q/p},
    \qquad
    n_\lambda
    \isdef
    \lceil a_\lambda\rceil.
\end{equation*}
The function $\lambda\mapsto a_\lambda$ is decreasing on $(0,1)$ and
$a_\lambda\to\infty$ as $\lambda\downarrow0$. Hence, there exists
$\lambda_0\in(0,1)$ such that $a_\lambda\geq n_0$ for every $\lambda\in(0,\lambda_0]$. In particular, $n_\lambda\geq n_0$, so that \eqref{eq:spectral_bound} applies at $n=n_\lambda$. Moreover, since $a_\lambda\geq1$, $a_\lambda\leq n_\lambda\leq 2a_\lambda.$ We now use these two inequalities to bound the leading term $n_\lambda$ separately, and the spectral tail-term in \eqref{eq:effdim_split}. First, using the upper bound for $n_\lambda$, we have 
\begin{equation*}
    n_\lambda \leq 2 a_\lambda = 2 \lambda^{-1/p}
    \bigl(\log(1/\lambda)\bigr)^{q/p} \implies n_\lambda \lesssim  \lambda^{-1/p}
    \bigl(\log(1/\lambda)\bigr)^{q/p}.
\end{equation*}
We next bound the spectral tail term. Applying
\eqref{eq:spectral_bound} at $n=n_\lambda$ gives 
\begin{equation*}
    \frac{1}{\lambda}
    \sum_{m>n_\lambda}\lambda_m(\Sig)
    \lesssim
    \lambda^{-1}
    n_\lambda^{1-p}(\log n_\lambda)^q.
\end{equation*}
The two factors involving $n_\lambda$ are controlled separately. Since
$1-p<0$, the lower bound $n_\lambda\geq a_\lambda$ implies $  n_\lambda^{1-p}
    \leq
    a_\lambda^{1-p}.$ On the other hand, the upper bound $n_\lambda\leq2a_\lambda$ gives
\begin{equation*}
    \log n_\lambda
    \leq
    \log(2a_\lambda)
    =
    \log 2
    +
    \frac{1}{p}\log(1/\lambda)
    +
    \frac{q}{p}\log\log(1/\lambda).
\end{equation*}
Since $\lambda_0$ is fixed,  $\log(1/\lambda) \geq \log(1/\lambda_0)>0$ and $\log x \leq x$ for any $x>0$, the right side is bounded by a constant multiple of $\log(1/\lambda)$, uniformly for $\lambda \in (0, \lambda_0]$. Therefore,
\begin{equation*}
\begin{aligned}
    \frac{1}{\lambda}
    \sum_{m>n_\lambda}\lambda_m(\Sig)
    &\lesssim
    \lambda^{-1}
    a_\lambda^{1-p}
    \bigl(\log(1/\lambda)\bigr)^q
    \\
    &=
    \lambda^{-1}
    \lambda^{(p-1)/p}
    \bigl(\log(1/\lambda)\bigr)^{-q(p-1)/p}
    \bigl(\log(1/\lambda)\bigr)^q=
    \lambda^{-1/p}
    \bigl(\log(1/\lambda)\bigr)^{q/p}.
\end{aligned}
\end{equation*}
Thus, both the leading term and the spectral tail term in
\eqref{eq:effdim_split} are bounded by a constant multiple of
$\lambda^{-1/p}\bigl(\log(1/\lambda)\bigr)^{q/p}$. Substituting the two
bounds into \eqref{eq:effdim_split} proves the claim for every $\lambda\in(0,\lambda_0]$.
\end{proof}


\section{Estimator asymptotics under dependence}\label{appendix:estimator_prop}
Throughout, we use the following.
\begin{equation*}
\begin{split}
\mu = \Ebb^{\Pbb}[Y_t\,\psi_t], &\qquad \Sig = \Ebb^{\Pbb}[\psi_t\otimes \psi_t],\\
\widehat\mu_T = \frac{1}{T}\sum_{t=1}^T Y_t\,\psi_t, &\qquad \widehat\Sig_T = \frac{1}{T}\sum_{t=1}^T \psi_t \otimes \psi_t, \\
\widetilde \psi_t = Y_t\,\psi_t - \mu, &\qquad \widetilde \Sig_t = \psi_t \otimes \psi_t - \Sig.
\end{split}
\end{equation*}

\subsection{Preliminary results for estimator asymptotics}\label{appendix:proofs_error}

\begin{proof}[Proof of Proposition~\ref{proposition:well_defined}]
Since $\mu\in\Hscr$ and $\Sig\in\mathsf{HS}(\Hscr)$ are well-defined, so $  \wt\zeta_t=\big(\widetilde\psi_t,\widetilde \Sig_t\big)=\big(Y_t\,\psi_t-\mu,\ \psi_t\otimes\psi_t-\Sig\big)\in\Hbb$
is well-defined for every $t \in \Zbb$. The map $\varphi\colon\Hscr\oplus \Rbb\to\Hbb$, $\varphi(u, y) \isdef (yu-\mu, u \otimes u -\Sig)$ is measurable and $\wt\zeta_t=\varphi(\zeta_t)$ is a fixed measurable transformation of the process $\zeta_t$.

\emph{1.} Strict stationarity of $(\wt\zeta_t)_{t\in \Zbb}$ is inherited from $(\zeta_t)_{t\in \Zbb}$ by applying the measurable transformation $\varphi$. 

\emph{2.} For the mixing coefficients, $\sigma(\wt\zeta_k\colon m\le k\le n)\subset\sigma(\zeta_k\colon m\le k\le n)$ for all $-\infty\le m\le n\le\infty$, whence $\alpha_{\wt\zeta}(m)\le\alpha_\zeta(m)$ for every $m\ge1$ by the definition of the strong mixing coefficients in \eqref{eq:strong_mixing_coefficients}.

\emph{3.} Assumption~\ref{assumption:boundedness} gives $\|\mu\|_\Hscr\leq\kappa_Y\kappa$ and $\|\Sig\|_{\HS}\leq\Ebb^{\Pbb}\|\psi_t\|_\Hscr^2\leq\kappa^2$. Hence, almost surely,
\begin{equation*}
    \|\widetilde\psi_t\|_\Hscr
    \le |Y_t|\,\|\psi_t\|_{\Hscr}+\|\mu\|_\Hscr
    \leq2\kappa_Y\kappa,
    \qquad
    \|\widetilde \Sig_t\|_{\HS}
    \leq
    \|\psi_t\otimes\psi_t\|_{\HS}+\|\Sig\|_{\HS}
    \leq
    2\kappa^2,
\end{equation*}
using $\|\psi_t\otimes\psi_t\|_{\HS}=\|\psi_t\|_{\Hscr}^2\le\kappa^2$, and therefore $\|\wt\zeta_t\|_{\Hbb}\le\kappa_{\wt\zeta}\isdef\{(2\kappa_Y\kappa)^2+(2\kappa^2)^2\}^{1/2}$ a.s.

\emph{4.} Immediate from \emph{(3)}: $\Ebb^{\Pbb}\|\wt\zeta_1\|_{\Hbb}^{2+\delta}\le\kappa_{\wt\zeta}^{2+\delta}<\infty$ for every finite $\delta>0$.
\end{proof}

\begin{lemma}[Properties of $F_\lambda$]\label{lemma:F}
    Let $F_\lambda\colon\Hbb\to\Hscr$ be defined in \eqref{eq:F}. Then, $F_\lambda$ is a linear and bounded map that satisfies 
\begin{equation*}
    \|F_\lambda\|_{\op} \; \leq \; \sqrt{1 + \|h_\lambda\|_{\Hscr}^2} < \infty.
\end{equation*} 
\end{lemma}
\begin{proof}
   $F_\lambda$ is linear by construction. We now show that it is bounded in the operator norm.
   \begin{equation*}
        \|F_\lambda\|_{\op}= \underset{\|(h, A)\|_\Hbb = 1}{\operatorname{sup}} \|h -  A h_\lambda\|_\Hscr \leq \underset{\|(h, A)\|_\Hbb = 1}{\operatorname{sup}} \|h\|_\Hscr +  \|A h_\lambda\|_\Hscr \leq \underset{\|(h,A)\|_\Hbb = 1}{\operatorname{sup}} \|h\|_\Hscr +  \|A\|_{\op} \|h_\lambda\|_\Hscr.
   \end{equation*}
   Using the inequality $a + bc \leq \sqrt{a^2 + b^2} \, \sqrt{1+ c^2}$ (this is due to Cauchy-Schwarz) applied to the expression above, we obtain 
   \begin{align*}
     \|F_\lambda\|_{\op} &\leq  \underset{\|(h,A)\|_\Hbb = 1}{\operatorname{sup}} \left\{\sqrt{\|h\|_\Hscr^2 + \|A\|_{\HS}^2} \cdot \sqrt{1 + \|h_\lambda\|_\Hscr^2}\right\} \\
     &= \underset{\|(h, A)\|_\Hbb = 1}{\operatorname{sup}}\left\{\|(h, A)\|_\Hbb \cdot \sqrt{1 + \|h_\lambda\|_\Hscr^2}\right\} = \sqrt{1 + \|h_\lambda\|_\Hscr^2} < \infty.
\end{align*}
\end{proof}

\begin{lemma}[Consistency results]\label{lemma:consistency_results}
Under the same conditions of Proposition~\ref{proposition:well_defined}, 
\begin{equation*}
    \left\|\widehat{\mu}_T-\mu\right\|_{\Hscr}
    \overset{\mathrm{a.s.}}{\longrightarrow} 0,
    \qquad
    \left\|\widehat{\Sig}_T-\Sig\right\|_{\HS}
    \overset{\mathrm{a.s.}}{\longrightarrow} 0
    \qquad \text{as } T\to\infty.
\end{equation*}
\end{lemma}

\begin{proof}
By Proposition~\ref{proposition:well_defined}, the centered process
$\bigl(\widetilde{\zeta}_t\bigr)_{t\in\Zbb}$ is strictly stationary,
strongly mixing, and almost surely bounded in the separable Hilbert
space $\Hbb$. It is therefore Bochner integrable. Moreover, strong mixing implies ergodicity; see \citet{Bradley2005}. Mourier’s strong law for integrable stationary ergodic sequences in separable Banach spaces then yields
\begin{equation*}
    \frac{1}{T}\sum_{t=1}^{T}\widetilde{\zeta}_t
    \overset{\mathrm{a.s.}}{\longrightarrow} \Ebb^\Pbb[\widetilde \zeta_1] = 0
    \qquad \text{in } \Hbb;
\end{equation*}
see \citet[Chapter II, Section IV]{Mourier1953} and
\citet[Remark~5]{DedeckerMerlevede2008}. By \eqref{eq:sample_mean}, $\frac{1}{T}\sum_{t=1}^{T}\widetilde{\zeta}_t
    =\big(\widehat{\mu}_T-\mu,\,
        \widehat{\Sig}_T-\Sig
    \big)$. By the definition of the norm on
$\Hbb=\Hscr\oplus\mathsf{HS}(\Hscr)$,
\begin{equation*}
    \left\|
        \frac{1}{T}\sum_{t=1}^{T}\widetilde{\zeta}_t
    \right\|_{\Hbb}^{2}
    =
    \left\|\widehat{\mu}_T-\mu\right\|_{\Hscr}^{2}
    +
    \left\|\widehat{\Sig}_T-\Sig\right\|_{\HS}^{2}.
\end{equation*}
Since the left-hand side converges almost surely to zero, both terms
on the right-hand side converge almost surely to zero.
\end{proof}

\subsection{Long-run covariance series}
Throughout this subsection, $\wt\zeta_t=(\widetilde\psi_t,\widetilde \Sig_t)\in\Hbb$ of \eqref{eq:zeta_tilde} is the centered process of Proposition~\ref{proposition:well_defined}; by strict stationarity every marginal and lag-$k$ moment is index-independent, and we evaluate such moments at a representative index without loss of generality. Write
\begin{equation*}
\Gamma_\ell\isdef\Ebb^{\Pbb}\big[\wt\zeta_0\otimes\wt\zeta_\ell\big]\in\mathsf{HS}(\Hbb),
    \qquad \ell\in\Zbb,
\end{equation*}
for its lag-$\ell$ autocovariance operators. The following lemma supplies the absolute summability and covariance reconciliation used in the proof of Proposition~\ref{proposition:asymptotic_gauss}. It shows that the covariance operators of the normalized partial sums converge to the bilateral autocovariance series and then identifies that limit with the covariance operator characterized by the Hilbert-space CLT.

\begin{lemma}\label{lemma:longruncovariance_limit}
Under the same conditions of Proposition~\ref{proposition:well_defined}, the series
$(\Gamma_\ell)_{\ell\in\Zbb}$ is absolutely summable, that is, $  \sum_{\ell\in\Zbb}\|\Gamma_\ell\|_{\op}<\infty.$ Moreover, with $S_T\isdef\sum_{t=1}^{T}\wt\zeta_t$,
\begin{equation*}
    \underset{T\to\infty}{\lim} \; \frac{1}{T}\Ebb^{\Pbb}[S_T\otimes S_T]
    =
     \underset{T\to\infty}{\lim} \; \sum_{|\ell|<T}
    \left(1-\frac{|\ell|}{T}\right)\Gamma_\ell =
    \sum_{\ell\in\Zbb}\Gamma_\ell,
\end{equation*}
where the convergence is in the operator norm.
\end{lemma}

\begin{proof}
\emph{Step 1: absolute summability.}
By strict stationarity,
$\Gamma_{-m}=\Gamma_m^\ast$ for every $m\in\Nbb$, so it is enough to show that
$\sum_{m\in\Nbb}\|\Gamma_m\|_{\op}<\infty$. Indeed,
\begin{equation*}
    \sum_{\ell\in\Zbb}\|\Gamma_\ell\|_{\op}
    =\|\Gamma_0\|_{\op}
    +\sum_{m\in\Nbb}
    \left(\|\Gamma_m\|_{\op}+\|\Gamma_m^\ast\|_{\op}
    \right)
    =\|\Gamma_0\|_{\op}
    +2\sum_{m\in\Nbb}\|\Gamma_m\|_{\op}.
\end{equation*}
Moreover, $  \|\Gamma_0\|_{\op}
    \leq
    \Ebb^{\Pbb}\|\wt\zeta_0\|_{\Hbb}^{2}
    \leq
    \kappa_{\wt\zeta}^{2}
    <\infty.$ Now fix $\eta,\vartheta\in\Hbb$ such that
$\|\eta\|_{\Hbb}=\|\vartheta\|_{\Hbb}=1$, and define
$a_\vartheta\isdef\langle\wt\zeta_0,\vartheta\rangle_{\Hbb}$ and
$b_\eta\isdef\langle\wt\zeta_m,\eta\rangle_{\Hbb}$ for some $m\in\Nbb$.
The random variable $a_\vartheta$ is measurable with respect to
$\Fscr_{-\infty}^{0}(\wt\zeta)$, while $b_\eta$ is measurable with respect to
$\Fscr_{m}^{\infty}(\wt\zeta)$. Moreover, they are both centered due to the
same property of $(\wt\zeta_t)$ and satisfy the almost-sure bounds
\begin{equation*}
    |a_\vartheta|
    \leq
    \|\wt\zeta_0\|_{\Hbb}\|\vartheta\|_{\Hbb}
    \leq
    \kappa_{\wt\zeta},
    \qquad
    |b_\eta|
    \leq
    \|\wt\zeta_m\|_{\Hbb}\|\eta\|_{\Hbb}
    \leq
    \kappa_{\wt\zeta}.
\end{equation*}
This follows from Cauchy--Schwarz and
Proposition~\ref{proposition:well_defined}. By the definition of the
rank-one operator,
\begin{equation*}
\begin{aligned}
    \langle\Gamma_m\eta,\vartheta\rangle_{\Hbb}
    =
    \left\langle
        \Ebb^{\Pbb}[\wt\zeta_0\otimes\wt\zeta_m]\eta,
        \vartheta
    \right\rangle_{\Hbb}
    =
    \Ebb^{\Pbb}
    \left[
        \langle\wt\zeta_m,\eta\rangle_{\Hbb}
        \langle\wt\zeta_0,\vartheta\rangle_{\Hbb}
    \right] 
    =
    \Ebb^{\Pbb}[a_\vartheta b_\eta]
    =
    \operatorname{Cov}(a_\vartheta,b_\eta).
\end{aligned}
\end{equation*}
The covariance inequality of \citet[Lemma~1.2]{Ibragimov1962} gives $  |\operatorname{Cov}(a_\vartheta,b_\eta)|
    \leq
    4\kappa_{\wt\zeta}^{2}\alpha_{\wt\zeta}(m).$ Therefore, using Proposition~\ref{proposition:well_defined} and
Assumption~\ref{assumption:strong_mixing},
\begin{equation*}
    |\langle\Gamma_m\eta,\vartheta\rangle_{\Hbb}|
    \leq
    4\kappa_{\wt\zeta}^{2} \, \alpha_{\wt\zeta}(m)
    \leq
    4\kappa_{\wt\zeta}^{2} \, \alpha_\zeta(m)
    \leq
    4c_\alpha\kappa_{\wt\zeta}^{2} \, m^{-\gamma_\alpha}.
\end{equation*}
The bound does not depend on the unit vectors $\eta$ and $\vartheta$. Hence,
\begin{equation*}
    \|\Gamma_m\|_{\op}
    =
    \underset{
        \|\eta\|_{\Hbb}=\|\vartheta\|_{\Hbb}=1
    }{\sup}
    |\langle\Gamma_m\eta,\vartheta\rangle_{\Hbb}|
    \leq
    4c_\alpha\kappa_{\wt\zeta}^{2}m^{-\gamma_\alpha}.
\end{equation*}
We can therefore conclude that
$\sum_{m\in\Nbb}\|\Gamma_m\|_{\op}<\infty$, since
$\gamma_\alpha>2$. Thus,
$\sum_{\ell\in\Zbb}\|\Gamma_\ell\|_{\op}<\infty$, and hence the operator series $\sum_{\ell\in\Zbb}\Gamma_\ell$ converges in the operator norm.

\emph{Step 2: covariance identity.}
Expanding $S_T\otimes S_T$ and using strict stationarity gives
\begin{equation*}
\begin{aligned}
    \frac{1}{T}\Ebb^{\Pbb}[S_T\otimes S_T]
    =
    \frac{1}{T}
    \sum_{t=1}^{T}\sum_{s=1}^{T}
    \Ebb^{\Pbb}[\wt\zeta_t\otimes\wt\zeta_s] 
    =
    \frac{1}{T}
    \sum_{t=1}^{T}\sum_{s=1}^{T}
    \Gamma_{s-t} 
    =
    \sum_{|\ell|<T}
    \left(1-\frac{|\ell|}{T}\right)\Gamma_\ell.
\end{aligned}
\end{equation*}
Define
\begin{equation*}
    w_T(\ell)
    \isdef
    \left(1-\frac{|\ell|}{T}\right)
    \mathbbm{1}\{|\ell|<T\},
    \qquad \ell\in\Zbb.
\end{equation*}
Consequently, as $T\to \infty$,
\begin{equation*}
\begin{aligned}
    \left\|
        \frac{1}{T}\Ebb^{\Pbb}[S_T\otimes S_T]
        -
        \sum_{\ell\in\Zbb}\Gamma_\ell
    \right\|_{\op}
    =
    \left\|
        \sum_{\ell\in\Zbb}
        \bigl(w_T(\ell)-1\bigr)\Gamma_\ell
    \right\|_{\op} 
    \leq
    \sum_{\ell\in\Zbb}
    |w_T(\ell)-1|\|\Gamma_\ell\|_{\op}
    \longrightarrow 0.
\end{aligned}
\end{equation*}
Indeed, for every fixed $\ell\in\Zbb$, $|w_T(\ell)-1|\to 0$, while
$|w_T(\ell)-1|\|\Gamma_\ell\|_{\op}\leq \|\Gamma_\ell\|_{\op}$ and
$\sum_{\ell\in\Zbb}\|\Gamma_\ell\|_{\op}<\infty$ by Step~1. The conclusion
therefore follows by dominated convergence for the counting measure on
$\Zbb$.
\end{proof}


\subsection{Asymptotic properties of the sample estimator}\label{appendix:proofs_asymptotic_props}

\begin{proof}[Proof of Proposition~\ref{proposition:asymptotic_gauss}]
\emph{Step 1: verification of the Hilbert-space CLT conditions.}
By Proposition~\ref{proposition:well_defined}, $(\wt\zeta_t)_{t\in\Zbb}$ is a strictly stationary centered process that is almost surely bounded in the separable Hilbert space $\Hbb$, and strongly mixing with coefficients $\alpha_{\wt\zeta}(m)\leq\alpha_\zeta(m)\leq c_\alpha m^{-\gamma_\alpha}$, where $\gamma_\alpha>2$; the last inequality is due to Assumption~\ref{assumption:strong_mixing}. Choose any finite $\delta>2/(\gamma_\alpha-1)$. Then,
\begin{equation*}
    \sum_{m\geq1}m^{\frac{2}{\delta}}\alpha_{\wt\zeta}(m)
    \leq
    c_\alpha\sum_{m\geq1}
    m^{\frac{2}{\delta}-\gamma_\alpha}
    <\infty,
\end{equation*}
and almost-sure boundedness gives
$\Ebb^{\Pbb}\|\wt\zeta_1\|_{\Hbb}^{2+\delta}<\infty$. This verifies all the conditions as detailed in \citet[Corollary 1]{Merlevede1997}.

\emph{Step 2: Hilbert-space CLT and its covariance operator.}
Applying the Hilbert-space CLT of
\citet[Theorem~4]{Merlevede1997}, (see also for additional reference, \citet[Theorem~2.1]{Merlevede2003}) therefore yields
\begin{equation}\label{eq:CLT_1}
    T^{-1/2}\sum_{t=1}^T\wt\zeta_t
    \overset{d}{\longrightarrow}
    \Ncal_{\Hbb}(0,\Gamma),
\end{equation}
where $\Gamma$ is the covariance operator of the limiting
$\Hbb$-valued Gaussian random element. Let $(e_j)_{j\geq1}$ be an
orthonormal basis of $\Hbb$. By the characterization in
\citet[Theorem~4]{Merlevede1997}, for every $i,j\geq1$,
\begin{equation*}
    \left\langle\Gamma e_j,e_i\right\rangle_{\Hbb}
    =
    \lim_{T\to\infty}
    \frac{1}{T}
    \Ebb^{\Pbb}\!\left[
        \left\langle
            \sum_{t=1}^T\wt\zeta_t,e_i
        \right\rangle_{\Hbb}
        \left\langle
            \sum_{t=1}^T\wt\zeta_t,e_j
        \right\rangle_{\Hbb}
    \right].
\end{equation*}
On the other hand, Lemma~\ref{lemma:longruncovariance_limit} gives
\begin{equation*}
    \lim_{T\to\infty} \frac{1}{T}
    \Ebb^{\Pbb}\!\left[
        \left(
            \sum_{t=1}^T\wt\zeta_t
        \right)
        \otimes
        \left(
            \sum_{t=1}^T\wt\zeta_t
        \right)
    \right]
  =
    \sum_{\ell\in\Zbb}\Gamma_\ell
\end{equation*}
in the operator norm. Therefore, for every $i,j\geq1$,
\begin{equation*}
    \left\langle\Gamma e_j,e_i\right\rangle_{\Hbb}
    =
    \left\langle
        \left(
            \sum_{\ell\in\Zbb}\Gamma_\ell
        \right)e_j,e_i
    \right\rangle_{\Hbb}.
\end{equation*}
Since the two operators have the same coefficients with respect to
the orthonormal basis $(e_j)_{j\geq1}$, they coincide. Hence,
\begin{equation}\label{eq:Gamma_series_identification}
    \Gamma
    =
    \sum_{\ell\in\Zbb}\Gamma_\ell
    =
    \sum_{\ell\in\Zbb}
    \Ebb^{\Pbb}[\wt\zeta_0\otimes\wt\zeta_\ell].
\end{equation}
Since the limit is an $\Hbb$-valued Gaussian random element, its covariance operator $\Gamma$ is positive, self-adjoint, and trace class, see \citet{Baker1973}.

\emph{Step 3: bounded linear transformation.}
The map $F_\lambda\colon\Hbb\to\Hscr$ is bounded and linear by Lemma~\ref{lemma:F}. The continuous mapping theorem applied to \eqref{eq:CLT_1} gives
\begin{equation*}
    \sqrt T\,F_\lambda(\Bar S_T)
    \overset{d}{\longrightarrow}
    \Ncal_{\Hscr}(0,\Qcal_\lambda),
    \qquad
    \Qcal_\lambda=F_\lambda\Gamma F_\lambda^\ast.
\end{equation*}
Finally, applying bounded linearity term by term to the operator-norm convergent series \eqref{eq:Gamma_series_identification}, together with
$F_\lambda(a\otimes b)F_\lambda^\ast=F_\lambda a\otimes F_\lambda b$, yields the series representation \eqref{eq:Q_lambda}.
\end{proof}

\begin{proof}[Proof of Theorem~\ref{thm:asymptotic_properties}]
\noindent \emph{1.} We have $\|\widehat\mu_T-\mu\|_{\Hscr}\overset{a.s.}{\longrightarrow}0$ and
$\|\widehat\Sig_T-\Sig\|_{\op}\le \|\widehat\Sig_T-\Sig\|_{\HS}\overset{a.s.}{\longrightarrow}0$
from Lemma~\ref{lemma:consistency_results}. Now, $\widehat\Sig_{\lambda,T}^{-1}\in\Bscr(\Hscr)$ with $\|\widehat\Sig_{\lambda,T}^{-1}\|_{\op} \le 1/\lambda < \infty$. Finally, using the decomposition in \eqref{eq:difference}, we obtain that as $T\to\infty$,
\begin{equation*}
    \|\widehat h_{\lambda,T} - h_\lambda\|_\Hscr \leq \underbrace{\|\widehat\Sig_{\lambda,T}^{-1}\|_{\op}}_{< \infty} \, \Big(\underbrace{\|\widehat\mu_T - \mu\|_\Hscr}_{\overset{a.s.}{\longrightarrow} 0} + \underbrace{\|\widehat\Sig_T - \Sig\|_{\op}}_{\overset{a.s.}{\longrightarrow} 0} \, \underbrace{\|h_\lambda\|_\Hscr}_{< \infty}\Big) \overset{a.s.}{\longrightarrow} 0,
\end{equation*}
which proves the first claim.

\emph{2.} It holds, $\Sig_\lambda^{-1}\in \Bscr(\Hscr)$ with $\|\Sig_\lambda^{-1}\|_{\op} \le 1/\lambda < \infty$. We have the operator-norm bound from the identity in \eqref{eq:resolvent_identity},
\begin{equation*}
    \|\widehat\Sig_{\lambda,T}^{-1} - \Sig_\lambda^{-1}\|_{\op} \leq \underbrace{\|\widehat\Sig_{\lambda,T}^{-1}\|_{\op}}_{< \infty} \, \underbrace{\|\widehat\Sig_T - \Sig\|_{\op}}_{ \overset{a.s.}{\longrightarrow} 0} \, \underbrace{\|\Sig_\lambda^{-1}\|_{\op}}_{< \infty}  \overset{a.s.}{\longrightarrow} 0.
\end{equation*}
Therefore, $\widehat\Sig_{\lambda,T}^{-1} \overset{p}{\longrightarrow}\Sig_\lambda^{-1}$. Moreover, from Proposition~\ref{proposition:asymptotic_gauss}, $\sqrt{T} \, F_\lambda(\Bar S_T)  \overset{d}{\longrightarrow} \Ncal_\Hscr(0, \mathcal Q_\lambda)$. Applying \emph{Slutsky's theorem} to \eqref{eq:difference_equivalent} gives
    \begin{equation*}
            \sqrt{T}\big(\widehat h_{\lambda,T} - h_\lambda\big) = \underbrace{\widehat\Sig_{\lambda,T}^{-1}}_{\overset{p}{\longrightarrow} \Sig_\lambda^{-1}} \cdot \underbrace{\sqrt{T} \, F_\lambda(\Bar S_T)}_{\overset{d}{\longrightarrow} \Ncal_\Hscr(0, \mathcal Q_\lambda)} \overset{d}{\longrightarrow} \Ncal_\Hscr(0, \Vcal_\lambda),
    \end{equation*}
    where $\Vcal_\lambda$ is defined as in \eqref{eq:V_lambda}.
\end{proof}

\subsection{Finite-grid asymptotics and long-run covariance estimation}\label{appendix:proofs_finite_grid}

\paragraph{Finite-grid Gaussian limit.}
Proposition~\ref{proposition:asymptotic_convexity} follows by applying the bounded derivative-evaluation operator $\Dscr_{n,\bs\alpha}$ to the Hilbert-space central limit theorem in Theorem~\ref{thm:asymptotic_properties}. It remains only to identify the covariance matrix of the resulting finite-dimensional Gaussian limit.
\begin{proof}[Proof of Proposition~\ref{proposition:asymptotic_convexity}]
By Theorem~\ref{result:zhou} and \eqref{eq:second_derivative_evaluation},
\begin{equation*}
    \Dscr_{n,\bs\alpha}(\widehat h_{\lambda,T} - h_\lambda) = \Bigl[\langle \widehat h_{\lambda,T} - h_\lambda, \phi^{(\bs \alpha)}(\bs    \xi_j) \rangle_\Hscr\Bigr]_{j=1}^n = \Bigl[ \widehat h_{\lambda,T}^{(\bs \alpha)}(\bs \xi_j) -  h_\lambda^{(\bs \alpha)}(\bs  \xi_j)\Bigr]_{j=1}^n \in \Rbb^{n}.
\end{equation*}
From (2) of Theorem~\ref{thm:asymptotic_properties}, $\sqrt{T}\bigl(\widehat h_{\lambda,T} - h_\lambda\bigr) \overset{d}{\longrightarrow} \Ncal_\Hscr(0, \Vcal_\lambda)$ with $\Vcal_\lambda = \Sig_\lambda^{-1} \mathcal Q_\lambda\Sig_\lambda^{-1}$. Since $\Dscr_{n,\bs\alpha}$ is bounded and linear, the CMT gives
\begin{equation*}
    \sqrt{T}\Bigl(\widehat{\bs\theta}_{\lambda,n,T} - \bs\theta_{\lambda,n}\Bigr) = \varsigma\,\sqrt{T} \, \Dscr_{n,\bs\alpha}(\widehat h_{\lambda,T} - h_\lambda) \overset{d}{\longrightarrow} \Ncal_n(\bs 0, \varsigma^2\, \Dscr_{n,\bs\alpha} \Vcal_\lambda \Dscr_{n,\bs\alpha}^\ast)=\Ncal_n(\bs 0, \Dscr_{n,\bs\alpha} \Vcal_\lambda \Dscr_{n,\bs\alpha}^\ast),
\end{equation*}
using $\varsigma^2=1$. It remains to show that $\Dscr_{n,\bs\alpha} \Vcal_\lambda \Dscr_{n,\bs\alpha}^\ast = \bs \Omega_{\lambda,n}$. The operator $\Dscr_{n,\bs\alpha} \Vcal_\lambda \Dscr_{n,\bs\alpha}^\ast\colon\Rbb^n \to \Rbb^n$ is linear; hence, it can be represented as a matrix with respect to the canonical basis $\{\bs e_j\}_{j=1}^n$ of $\Rbb^n$. For $1 \leq k, j \leq n$,
\begin{align*}
 \big[\Dscr_{n,\bs\alpha} \Vcal_\lambda \Dscr_{n,\bs\alpha}^\ast\big]_{k,j} &= \langle \bs e_k, \Dscr_{n,\bs\alpha} \Vcal_{\lambda} \Dscr_{n,\bs\alpha}^\ast \, \bs e_j \rangle_{\Rbb^n} \\
 &= \langle \Dscr_{n,\bs\alpha}^\ast \, \bs e_k, \Vcal_{\lambda} \Dscr_{n,\bs\alpha}^\ast \, \bs e_j \rangle_{\Hscr} \\
 &= \left\langle \phi^{(\bs \alpha)}(\bs \xi_k), \Vcal_\lambda\phi^{(\bs \alpha)}(\bs \xi_j)\right\rangle_{\Hscr}.
\end{align*}
Using the definition of $\Vcal_\lambda$ and $u_j$ from \eqref{eq:V_lambda} and \eqref{eq:u},
\begin{align*}
     \big[\Dscr_{n,\bs\alpha} \Vcal_\lambda \Dscr_{n,\bs\alpha}^\ast\big]_{k,j} &= \langle \phi^{(\bs \alpha)}(\bs \xi_k), \Sig_\lambda^{-1}\mathcal Q_\lambda \Sig_\lambda^{-1} \phi^{(\bs \alpha)}(\bs \xi_j)\rangle_{\Hscr} \\
     &=  \langle \Sig_\lambda^{-1}\phi^{(\bs \alpha)}(\bs \xi_k), \mathcal Q_\lambda \Sig_\lambda^{-1} \phi^{(\bs \alpha)}(\bs    \xi_j)\rangle_{\Hscr} = \langle u_k, \mathcal Q_\lambda u_j \rangle_{\Hscr} = \big[\bs \Omega_{\lambda,n}\big]_{k,j},
\end{align*}
where the last equality follows from \eqref{eq:Omega_longrun}. This proves the claim.
\end{proof}

\paragraph{Notation and uniform bounds.}
Throughout the remainder of this subsection,
$F_t\isdef F_\lambda(\wt\zeta_t)$ denotes the centered population score
in \eqref{eq:F}, $\widehat F_t$ denotes its plug-in counterpart in
\eqref{eq:sample_Q}, and $\bs g_t,\widehat{\bs g}_t\in\Rbb^n$ denote the
ideal and feasible grid scores in \eqref{eq:gridscore_main} and
\eqref{eq:gridscore_feasible}, respectively. The elements
$u_j$ and $\widehat u_j$ are defined in \eqref{eq:u}. 
Since the testing grid is fixed and finite,
Theorem~\ref{result:zhou} ensures that
\begin{equation}\label{eq:rho_derivative_grid}
\rho\isdef
\max_{1\le j\le n}
\|\phi^{(\bs\alpha)}(\bs\xi_j)\|_{\Hscr}
<\infty.
\end{equation}
Moreover, as $n$ is fixed,
all matrix norms on $\Rbb^{n\times n}$ are equivalent; throughout, we use
the Euclidean operator norm $\|\cdot\|_{\op}$. We next state two uniform almost-sure bounds used repeatedly below. By
Assumption~\ref{assumption:boundedness},
$\|\psi_t\|_{\Hscr}\le\kappa$ and $|Y_t|\le\kappa_Y$ a.s. Moreover,
$\|\mu\|_{\Hscr}\le\kappa_Y\kappa$ and $ \|\Sig\|_{\op}
    \le \Ebb^{\Pbb}\|\psi_t\|_{\Hscr}^2
    \le \kappa^2.$ As a result,
\begin{equation}\label{eq:Fbound}
    \sup_{t\in\Zbb}\|F_t\|_{\Hscr}
    \le
    2\kappa_Y\kappa
    +2\kappa^2\|h_\lambda\|_{\Hscr}
    \isdef M_F
    \qquad\text{a.s.}
\end{equation}
By \eqref{eq:u}, \eqref{eq:rho_derivative_grid},
and the bounds
$\|\Sig_\lambda^{-1}\|_{\op},
\|\widehat\Sig_{\lambda,T}^{-1}\|_{\op}\le\lambda^{-1}$,
we obtain
\begin{equation*}
    \max_{1\le j\le n}\|u_j\|_{\Hscr}
    \le \frac{\rho}{\lambda},
    \qquad
    \max_{1\le j\le n}\|\widehat u_j\|_{\Hscr}
    \le \frac{\rho}{\lambda}.
\end{equation*}
It also follows from \eqref{eq:Fbound} and the Cauchy--Schwarz inequality that
\begin{equation}\label{eq:Gbound}
\begin{aligned}
     \sup_{t\in\Zbb}\|\bs g_t\|_2
    =
    \sup_{t\in\Zbb}
    \left(\sum_{j=1}^n|\langle u_j,F_t\rangle_{\Hscr}|^2\right)^{1/2}
    &\le
    \left(\sum_{j=1}^n\|u_j\|_{\Hscr}^2\right)^{1/2}
    \sup_{t\in\Zbb}\|F_t\|_{\Hscr} \\
    &\le
    \frac{\rho\sqrt n\,M_F}{\lambda}
    =:
     M_g
    \quad\text{a.s.}
\end{aligned}
\end{equation}

\paragraph{Uniform approximation of feasible grid scores.}
The next lemma shows that estimating the population score $F_t$ and the
derivative representers $u_j$ change the grid scores by
$O_{p}(T^{-1/2})$ uniformly over the observed sample. This bound will
be used to replace the idealized long-run covariance estimator by its feasible
counterpart.

\begin{lemma}[Uniform plug-in error of the grid scores]\label{lem:C7}
Under the same conditions as Proposition~\ref{proposition:well_defined},
\begin{equation*}
    \max_{1\le t\le T}
    \big\|\widehat{\bs g}_t-\bs g_t\big\|_2
    =
    O_{p}\big(T^{-1/2}\big).
\end{equation*}
\end{lemma}

\begin{proof}
Set $ \varepsilon_T
    \isdef
    \|\widehat\mu_T-\mu\|_{\Hscr}
    +
    \|\widehat\Sig_T-\Sig\|_{\HS}.$ By \eqref{eq:sample_mean} and \eqref{eq:SN_rate}, $ \varepsilon_T
    \le
    \sqrt{2}\,\|\Bar S_T\|_{\Hbb}
    =
    O_p(T^{-1/2}).$ The proof now proceeds in three steps.

\emph{Step 1:} Using the decomposition in \eqref{eq:difference},
\begin{equation*}
    \|\widehat h_{\lambda,T}-h_\lambda\|_{\Hscr}
    \le
    \lambda^{-1}
    \left(
        \|\widehat\mu_T-\mu\|_{\Hscr}
        +
        \|h_\lambda\|_{\Hscr}
        \|\widehat\Sig_T-\Sig\|_{\op}
    \right)
    \le
    \lambda^{-1}
    \max\{1,\|h_\lambda\|_{\Hscr}\}\,
    \varepsilon_T.
\end{equation*}
Since $\varepsilon_T = O_p(T^{-1/2})$, it follows that $ \|\widehat h_{\lambda,T}-h_\lambda\|_{\Hscr} = O_p(T^{-1/2})$. Again, using \eqref{eq:resolvent_identity},
\begin{align*}
    \max_{1\le j\le n}
    \|\widehat u_j-u_j\|_{\Hscr}
    &\le
    \left(
        \max_{1\le k\le n}\|\phi^{(\bs\alpha)}(\bs\xi_k)\|_{\Hscr}
    \right)
    \|\widehat\Sig_{\lambda,T}^{-1}
        -\Sig_\lambda^{-1}\|_{\op} \\
    &\le \frac{\rho}{\lambda^2}
    \|\widehat\Sig_T-\Sig\|_{\op} \le \frac{\rho}{\lambda^2}\varepsilon_T
=O_p(T^{-1/2}).
\end{align*}

\emph{Step 2:} Substitute the first-order conditions
$\widehat\Sig_{\lambda,T}\widehat h_{\lambda,T}=\widehat\mu_T$
and $\Sig_\lambda h_\lambda=\mu$ into \eqref{eq:sample_Q} and
\eqref{eq:F}. We then obtain
\begin{equation*}
    \widehat F_t
    =
    \bigl(
        Y_t-\langle\widehat h_{\lambda,T},\psi_t\rangle_{\Hscr}
    \bigr)\psi_t
    -
    \lambda\widehat h_{\lambda,T},
    \qquad
    F_t
    =
    \bigl(
        Y_t-\langle h_\lambda,\psi_t\rangle_{\Hscr}
    \bigr)\psi_t
    -
    \lambda h_\lambda.
\end{equation*}
Hence, with $\Delta_t\isdef\widehat F_t-F_t$, $  \Delta_t
    =
    -(\psi_t\otimes\psi_t+\lambda I_{\Hscr})
    (\widehat h_{\lambda,T}-h_\lambda).$ Therefore,
\begin{equation*}
    \max_{1\le t\le T}\|\Delta_t\|_{\Hscr}
    \le
    (\kappa^2+\lambda)
    \|\widehat h_{\lambda,T}-h_\lambda\|_{\Hscr}
    \qquad\text{a.s.}
\end{equation*}
Step~1 implies that $\max_{1 \leq t \leq T} \|\Delta_t\|_\Hscr = O_p(T^{-1/2})$. By \eqref{eq:Fbound}, 
\begin{equation*}
    \max_{1\le t\le T}\|\widehat F_t\|_{\Hscr}
    \le
    M_F+
    \max_{1\le t\le T}\|\Delta_t\|_{\Hscr}
    =
    O_p(1).
\end{equation*}

\emph{Step 3:} For every $1\le t\le T$ and $1\le j\le n$,
adding and subtracting
$\langle u_j,\widehat F_t\rangle_{\Hscr}$
in the difference $\widehat{\bs g}_{t,j}-\bs g_{t,j}$
and applying Cauchy--Schwarz yields
\begin{equation*}
    \big|
        \widehat{\bs g}_{t,j}-\bs g_{t,j}
    \big|
    \le
    \|\widehat u_j-u_j\|_{\Hscr}
    \|\widehat F_t\|_{\Hscr}
    +
    \|u_j\|_{\Hscr}\|\Delta_t\|_{\Hscr}.
\end{equation*}
Therefore,
\begin{align*}
\max_{1\le t\le T}\|\widehat{\bs g}_t-\bs g_t\|_2
\le \sqrt n\Bigg(\max_{1\le j\le n}\|\widehat u_j-u_j\|_{\Hscr}
\max_{1\le t\le T}\|\widehat F_t\|_{\Hscr}+\max_{1\le j\le n}\|u_j\|_{\Hscr}
\max_{1\le t\le T}\|\Delta_t\|_{\Hscr}
\Bigg).
\end{align*}
The first term is $O_p(T^{-1/2})$, since
$\max_j\|\widehat u_j-u_j\|_{\Hscr}=O_p(T^{-1/2})$
and $\max_t\|\widehat F_t\|_{\Hscr}=O_p(1)$.
The second term has the same order because $\max_{1\le j\le n}\|u_j\|_{\Hscr}
\le \rho/\lambda$ and $\max_t\|\Delta_t\|_{\Hscr}=O_p(T^{-1/2})$.
Thus, for fixed $n$ and $\lambda$,
\begin{equation*}
\max_{1\le t\le T}
\|\widehat{\bs g}_t-\bs g_t\|_2
=O_p(T^{-1/2}).
\end{equation*}
\end{proof}

\paragraph{Consistency of the long-run covariance estimator.}
\begin{proof}[Proof of Theorem~\ref{thm:consistency}]
The proof has two parts. We first establish the consistency of the idealized long-run covariance estimator based on the population grid scores. We then show that replacing those scores with their feasible plug-in counterparts is asymptotically negligible. For $\ell\geq0$, define the idealized lag-$\ell$ covariance matrices by $\widetilde{\bs\Gamma}^{g}_\ell
    \isdef
    \frac{1}{T}
    \sum_{t=1}^{T-\ell}
    \bs g_t\bs g_{t+\ell}^{\top},$ and the corresponding idealized counterparts,
\begin{equation*}
    \widetilde{\bs\Omega}_{\lambda,n,T}
    \isdef
    \widetilde{\bs\Gamma}^{g}_0
    +
    \sum_{\ell=1}^{L_T}
    w_\ell
    \left(
        \widetilde{\bs\Gamma}^{g}_\ell
        +
        (\widetilde{\bs\Gamma}^{g}_\ell)^\top
    \right).
\end{equation*}

\emph{Step 1: consistency of the idealized estimator.} Building on the kernel-based \footnote{The term ``kernel" here refers to a lag-window weighting function, in the smoothing sense used in kernel estimators and not to a reproducing kernel of an RKHS.} covariance-estimation framework of \citet[Theorem~1]{Hansen1992}, we apply the corrected formulation established by \citet[Theorem~2]{deJong2000} to the $\Rbb^n$-valued process $(\bs g_t)$, verifying his Assumptions~1.1, 1.2, (1.3$'$), and~1.4 in our setting. The Bartlett kernel $k(x)=(1-|x|)_+$ generating the weights $w_\ell=k(\ell/(L_T+1))$ satisfies the required kernel conditions in Assumption~1.1; see \citet[Condition~(K)]{Hansen1992}. For the moment--mixing trade-off of his Assumption~1.2 with $r=4$: since $\gamma_\alpha>2$, we may choose $p>4\gamma_\alpha/(\gamma_\alpha-2)$, equivalently $\gamma_\alpha\big(\tfrac12-\tfrac2p\big)>1$. The grid scores are almost surely bounded by \eqref{eq:Gbound}, so $\sup_t\Ebb^{\Pbb}\|\bs g_t\|_2^p\le M_g^p<\infty$, and, being fixed measurable transformations of $\zeta_t$, they are strictly stationary with $\alpha_{\bs g}(m)\le\alpha_\zeta(m)\le c_\alpha m^{-\gamma_\alpha}$; the proof follows the same arguments as in the proof of Proposition~\ref{proposition:well_defined}. Hence, the mixing-summability requirement, that is, Assumption 1.2 holds:
\begin{equation*}
    \sum_{m\ge1}\alpha_{\bs g}(m)^{\frac12-\frac2p}
    \;\le\;
    c_\alpha^{\frac12-\frac2p}\sum_{m\ge1}m^{-\gamma_\alpha(\frac12-\frac2p)}
    \;<\;\infty .
\end{equation*}
The bandwidth condition (1.3$'$) with $r=4$ is $L_T\to\infty$ with $L_T=o(T^{1/4})$, which is imposed in the statement of the theorem. Finally, Assumption~1.4 holds with
$\ell(x)=\mathbbm 1\{0\le x\le1\}$, since
$|k(x)|\le \ell(|x|)$ and
$\int_{\Rbb}|x|\ell(|x|)\,dx<\infty$. The population matrix in \citet[Theorem~2]{deJong2000} is the limit of the normalized double sum
$T^{-1}\sum_{t=1}^{T}\sum_{s=1}^T\Ebb^{\Pbb}[\bs g_t\bs g_s^{\top}]$; under strict stationarity this equals the long-run autocovariance series $\bs\Omega_{\lambda,n}$ of \eqref{eq:Omega_longrun}, by the counting identity $\sum_{t=1}^T\sum_{s=1}^T\Ebb^{\Pbb}[\bs g_t\bs g_s^{\top}]=\sum_{|\ell|<T}(T-|\ell|)\,\Ebb^{\Pbb}[\bs g_0\bs g_\ell^{\top}]$ and the absolute-summability and covariance-limit argument of Lemma~\ref{lemma:longruncovariance_limit}, applied verbatim to the bounded, strictly stationary, strongly mixing $\Rbb^n$-valued process $(\bs g_t)$. Therefore, \citet[Theorem~2]{deJong2000} gives
\begin{equation*}
\widetilde{\bs\Omega}_{\lambda,n,T}\overset{p}{\longrightarrow}\bs\Omega_{\lambda,n}.
\end{equation*}
\emph{Step 2: feasible-minus-idealized negligibility.} Set $  \varepsilon_T^{g}
    \isdef
    \max_{1\le t\le T}
    \|\widehat{\bs g}_t-\bs g_t\|_2.$ By Lemma~\ref{lem:C7},
$\varepsilon_T^{g}=O_p(T^{-1/2})$ while \eqref{eq:Gbound} implies
$ \max_{1\le t\le T}\|\widehat{\bs g}_t\|_2
    \le
    M_g+\varepsilon_T^{g}
    =
    O_p(1).$ Thus, for every $0\le\ell\le L_T$,
\begin{equation*}
    \widehat{\bs g}_t\widehat{\bs g}_{t+\ell}^{\top}
    -
    \bs g_t\bs g_{t+\ell}^{\top}
    =
    (\widehat{\bs g}_t-\bs g_t)
    \widehat{\bs g}_{t+\ell}^{\top}
    +
    \bs g_t
    (\widehat{\bs g}_{t+\ell}-\bs g_{t+\ell})^{\top}.
\end{equation*}
Taking operator norms, averaging over $t$, and using \eqref{eq:Gbound} gives $\big\|
        \widehat{\bs\Gamma}^{g}_\ell
        -
        \widetilde{\bs\Gamma}^{g}_\ell
    \big\|_{\op}
    \le
    \varepsilon_T^{ g}
    \bigl(2M_g+\varepsilon_T^{ g}\bigr)$ uniformly over $0\le\ell\le L_T$. Summing over the lags in
\eqref{eq:HAC_main}, and using that the Bartlett weights are bounded by one,
\begin{equation*}
    \big\|
        \widehat{\bs\Omega}_{\lambda,n,T}
        -
        \widetilde{\bs\Omega}_{\lambda,n,T}
    \big\|_{\op}
    \le
    (2L_T+1)\,
    \varepsilon_T^{ g}
    \bigl(2M_g+\varepsilon_T^{ g}\bigr)
    =
    O_p(L_TT^{-1/2})
    =
    o_p(1),
\end{equation*}
since $L_T=o(T^{1/4})$; in fact, $L_T=o(T^{1/2})$ would suffice for this
step, while the stronger rate is required by Assumption~(1.3$'$) in Step~1.
The triangle inequality combined with Step~1 gives
\begin{equation*}
    \widehat{\bs\Omega}_{\lambda,n,T}
    \overset{p}{\longrightarrow}
    \bs\Omega_{\lambda,n}.
\end{equation*}
In addition, the Bartlett weights make both $\widetilde{\bs\Omega}_{\lambda,n,T}$ and $\widehat{\bs\Omega}_{\lambda,n,T}$ positive semidefinite \citep[Theorem~1]{NeweyWest1987}, and $\bs\Omega_{\lambda,n}\geq0$ as the limit of positive semidefinite matrices. Finally, if $\bs\Omega_{\lambda,n}>0$, Weyl's inequality \citep[Corollary~III.2.6]{Bhatia1997},
\begin{equation*}
    \big|\lambda_{\min}\big(\widehat{\bs\Omega}_{\lambda,n,T}\big)-\lambda_{\min}(\bs\Omega_{\lambda,n})\big|
    \le
    \big\|\widehat{\bs\Omega}_{\lambda,n,T}-\bs\Omega_{\lambda,n}\big\|_{\op},
\end{equation*}
gives $\lambda_{\min}(\widehat{\bs\Omega}_{\lambda,n,T})\overset{p}{\longrightarrow}\lambda_{\min}(\bs\Omega_{\lambda,n})>0$, so the estimator $\widehat{\bs\Omega}_{\lambda,n,T}$ is invertible with probability tending to one. This proves the claim.
\end{proof}


\section{Chi-bar-square projection theory}\label{appendix:chi_bar_projection}

This appendix presents the standard cone-projection results used in
Section~\ref{subsec:test_stat}, establishes continuity under an estimated
Mahalanobis metric, and proves Theorem~\ref{thm:test_statistic} together
with Propositions~\ref{prop:level}, \ref{prop:fixed_power}, and
\ref{prop:regularization_bridge}.

\paragraph{Standard cone-projection identities.}
We use Theorem~\ref{thm:projection_theorem} under the Mahalanobis inner product of Section~\ref{subsec:test_stat}, where
$\Mscr^\circ$ denotes the polar cone with respect to
the corresponding inner product.

\begin{proposition}[{\citet[Proposition 3.4.1]{Silvapulle2001}}]\label{prop:chi_bar_squared}
Let $\Mscr \subset \Rbb^n$ be a closed convex cone, and let $\bs Z \sim \Ncal_n(\bs 0, \bs V)$, where $\bs V$ is a symmetric and positive definite matrix. Then
\begin{enumerate}
    \item $ \|\bs Z\|^2_{\bs V^{-1}}
    =
    \|P_{\bs V^{-1}}(\bs Z;\Mscr)\|_{\bs V^{-1}}^2
    +
    \|\bs Z-P_{\bs V^{-1}}(\bs Z;\Mscr)\|_{\bs V^{-1}}^2.$
    \item $  P_{\bs V^{-1}}(\bs Z;\Mscr^\circ)
    =
    \bs Z-P_{\bs V^{-1}}(\bs Z;\Mscr).$
    \item $ \|P_{\bs V^{-1}}(\bs Z;\Mscr)\|_{\bs V^{-1}}^2
    \sim
    \Bar{\chi}^2(\bs V,\Mscr),$ and $ \|\bs Z-P_{\bs V^{-1}}(\bs Z;\Mscr)\|_{\bs V^{-1}}^2
    \sim
    \Bar{\chi}^2(\bs V,\Mscr^\circ).$
\end{enumerate}
\end{proposition}
The above proposition relies on the \emph{Moreau decomposition} for nonempty closed convex cones; see, for example, \citet[Theorem~6.30]{Bauschke2017}. Next is the fundamental result we use to characterize the asymptotic law of the test statistic.
\begin{theorem}[{Chi-bar-square distribution; \citet[Theorem 3.4.2]{Silvapulle2001}}]\label{thm:chi_bar_distribution}
Let $\Mscr$ be a closed convex cone in $\Rbb^n$, and let $\bs V \in \Rbb^{n\times n}$ be a symmetric and positive definite matrix. Then the distribution of $\Bar{\chi}^2(\bs V,\Mscr)$ is given by
\begin{equation*}
\operatorname{Pr}\left(\Bar{\chi}^2(\bs V,\Mscr)\leq c\right)
=
\sum_{j=0}^n
\omega_j(n,\bs V,\Mscr)\,
\operatorname{Pr}(\chi_j^2\leq c),
\end{equation*}
where $\omega_j(n,\bs V,\Mscr)\geq 0$ for $0\leq j\leq n$, $\sum_{j=0}^n \omega_j(n,\bs V,\Mscr)=1$.
\end{theorem}

\paragraph{Continuity under Mahalanobis metric.}
The following results provide the continuity properties required for the plug-in distribution theory. We first establish continuity of the projection as the metric varies, before considering joint continuity in the projected point and the metric.
\begin{proposition}[Joint continuity of the Mahalanobis projection]
\label{prop:joint_continuity_projector}
Let $\Mscr\subset\Rbb^n$ be a nonempty closed convex cone. The map
$f\colon\Rbb^n\times\Sbb^n_{++}\to\Rbb^n$ defined by
\begin{equation*}
    f(\bs x,\bs M)
    \isdef
    P_{\bs M}(\bs x;\Mscr),
    \qquad
    (\bs x,\bs M)\in\Rbb^n\times\Sbb^n_{++},
\end{equation*}
is jointly continuous.
\end{proposition}

\begin{proof}
The proof proceeds in two steps. 

\emph{Step 1.}
Fix $\bs x\in\Rbb^n$. Assume that
$(\bs M_k)_{k\in\Nbb}\subset\Sbb^n_{++}$ satisfies
$\bs M_k\to\bs M\in\Sbb^n_{++}$. By Weyl's perturbation theorem, \citet[Corollary III.2.6]{Bhatia1997}, it holds that $\max_{1\leq j\leq n}
    \left|
    \lambda_j(\bs M_k)-\lambda_j(\bs M)
    \right|
    \leq
    \|\bs M_k-\bs M\|_{\op},$ where the eigenvalues are ordered consistently. In particular, for every $j=1,\ldots,n$, $  \left|
    \lambda_j(\bs M_k)-\lambda_j(\bs M)
    \right|
    \leq
    \|\bs M_k-\bs M\|_{\op}.$ Hence, $\lambda_j(\bs M_k)\to\lambda_j(\bs M)$, for $j=1,\ldots,n.$ Since $\bs M$ is positive definite, there exist constants $c,C>0$ and
$k_0\in\Nbb$ such that for all $k\geq k_0$, $ 0<c
    \leq\lambda_{\min}(\bs M_k)
    \leq\lambda_{\max}(\bs M_k)
    \leq C<\infty$. Consequently, for every $\bs z\in\Rbb^n$,
\begin{equation}
\label{eq:norm_inequality}
    c\|\bs z\|_2^2
    \leq
    \bs z^\top\bs M_k\bs z
    =
    \|\bs z\|_{\bs M_k}^2
    \leq
    C\|\bs z\|_2^2,
    \qquad k\geq k_0.
\end{equation}
Set $\bs u_k\isdef P_{\bs M_k}(\bs x;\Mscr)$. By the best-approximation property of the projection, see Theorem~\ref{thm:projection_theorem}, $ \|\bs u_k-\bs x\|_{\bs M_k}
    \leq
    \|\bs u-\bs x\|_{\bs M_k}$ for every $\bs u\in\Mscr.$ Since $\bs 0\in\Mscr$, choosing $\bs u=\bs 0$ and using the triangle inequality gives $ \|\bs u_k\|_{\bs M_k}
    \leq
    \|\bs x\|_{\bs M_k}
    +
    \|\bs u_k-\bs x\|_{\bs M_k}
    \leq
    2\|\bs x\|_{\bs M_k}
    \leq
    2\sqrt{C}\|\bs x\|_2.$ Together with \eqref{eq:norm_inequality}, this yields $ \|\bs u_k\|_2
    \leq
    \frac{1}{\sqrt{c}}\|\bs u_k\|_{\bs M_k}
    \leq
    2\sqrt{\frac{C}{c}}\|\bs x\|_2.$ Thus, $(\bs u_k)_{k\in\Nbb}$ is bounded in the Euclidean norm. By the Bolzano--Weierstrass theorem, it has a convergent subsequence $\bs u_{k_\ell}\to\bs u^\star$. Since $\bs u_{k_\ell}\in\Mscr$ for every $\ell$ and $\Mscr$ is closed, it follows that
$\bs u^\star\in\Mscr$. By Theorem~\ref{thm:projection_theorem}, for every $\bs v\in\Mscr$, $\left\langle
    \bs x-\bs u_{k_\ell},
    \bs v-\bs u_{k_\ell}
\right\rangle_{\bs M_{k_\ell}} \leq 0$. Equivalently,
\begin{align*}
    \left\langle
    \bs x-\bs u_{k_\ell},
    \bs M\bigl(\bs v-\bs u_{k_\ell}\bigr)
    \right\rangle_{\Rbb^n}
    +
    \left\langle
    \bs x-\bs u_{k_\ell},
    (\bs M_{k_\ell}-\bs M)
    \bigl(\bs v-\bs u_{k_\ell}\bigr)
    \right\rangle_{\Rbb^n}
    \leq 0.
\end{align*}
The first term converges to $\left\langle
    \bs x-\bs u^\star,
    \bs M(\bs v-\bs u^\star)
\right\rangle_{\Rbb^n}.$ For the second term,
\begin{align*}
    \left|
    \left\langle
    \bs x-\bs u_{k_\ell},
    (\bs M_{k_\ell}-\bs M)
    \bigl(\bs v-\bs u_{k_\ell}\bigr)
    \right\rangle_{\Rbb^n}
    \right|
    \leq
    \|\bs M_{k_\ell}-\bs M\|_{\op}
    \|\bs x-\bs u_{k_\ell}\|_2
    \|\bs v-\bs u_{k_\ell}\|_2,
\end{align*}
which converges to zero because
$\bs M_{k_\ell}\to\bs M$ and
$\bs u_{k_\ell}\to\bs u^\star$.
Therefore, $ \left\langle
    \bs x-\bs u^\star,
    \bs v-\bs u^\star
    \right\rangle_{\bs M}
    \leq 0$ for every $\bs v\in\Mscr.$ Since $\bs u^\star\in\Mscr$, Theorem~\ref{thm:projection_theorem} implies that $ \bs u^\star
    =
    P_{\bs M}(\bs x;\Mscr).$ Thus, every convergent subsequence of $(\bs u_k)_{k\in\Nbb}$ has the same limit
$P_{\bs M}(\bs x;\Mscr)$. Since the sequence is bounded, it follows that
\begin{equation}
\label{eq:projection_metric_continuity}
    P_{\bs M_k}(\bs x;\Mscr)
    \to
    P_{\bs M}(\bs x;\Mscr).
\end{equation}

\emph{Step 2.}
Now, consider any sequence
$(\bs x_k,\bs M_k)_{k\in\Nbb}\subset
\Rbb^n\times\Sbb^n_{++}$ such that $ (\bs x_k,\bs M_k)
    \to
    (\bs x,\bs M)
    \in
    \Rbb^n\times\Sbb^n_{++}.$ By the triangle inequality,
\begin{equation}\label{eq:projection_inequality}
\begin{split}
    \|P_{\bs M_k}(\bs x_k;\Mscr)
    -P_{\bs M}(\bs x;\Mscr)\|_2
    &\leq{}
    \|P_{\bs M_k}(\bs x_k;\Mscr)
    -P_{\bs M_k}(\bs x;\Mscr)\|_2 \\
    &\quad +
    \|P_{\bs M_k}(\bs x;\Mscr)
    -P_{\bs M}(\bs x;\Mscr)\|_2.
\end{split}
\end{equation}
The projection $P_{\bs M_k}(\cdot;\Mscr)$ is Lipschitz with constant $1$ in the $\bs M_k$-norm, see \citet[Chapter~4]{Bauschke2017}. Hence,
$ \|P_{\bs M_k}(\bs x_k;\Mscr)
    -P_{\bs M_k}(\bs x;\Mscr)\|_{\bs M_k}
    \leq
    \|\bs x_k-\bs x\|_{\bs M_k}.$ Using the uniform norm bounds established in Step 1, we obtain
\begin{equation*}
\begin{split}
    \|P_{\bs M_k}(\bs x_k;\Mscr)
    -P_{\bs M_k}(\bs x;\Mscr)\|_2
    &\leq
    \frac{1}{\sqrt{c}}
    \|P_{\bs M_k}(\bs x_k;\Mscr)
    -P_{\bs M_k}(\bs x;\Mscr)\|_{\bs M_k} \\
    &\leq
    \frac{1}{\sqrt{c}}
    \|\bs x_k-\bs x\|_{\bs M_k} \\
    &\leq
    \sqrt{\frac{C}{c}}
    \|\bs x_k-\bs x\|_2
    \to 0.
\end{split}
\end{equation*}
Moreover, \eqref{eq:projection_metric_continuity} gives $ \|P_{\bs M_k}(\bs x;\Mscr)
    -P_{\bs M}(\bs x;\Mscr)\|_2
    \to 0.$
Applying these two conclusions to \eqref{eq:projection_inequality} yields
$ P_{\bs M_k}(\bs x_k;\Mscr)
    \to
    P_{\bs M}(\bs x;\Mscr).$ Hence, the map
$(\bs x,\bs M)\mapsto P_{\bs M}(\bs x;\Mscr)$
is jointly continuous on
$\Rbb^n\times\Sbb^n_{++}$.
\end{proof}

\begin{corollary}[Joint continuity of the squared projection error]
\label{cor:projection_error_continuity}
Let $\Mscr\subset\Rbb^n$ be a nonempty, closed, convex cone. The map
$g\colon\Rbb^n\times\Sbb^n_{++}\to\Rbb$ defined by
\begin{equation*}
    g(\bs x,\bs M)
    \isdef
    \|\bs x-P_{\bs M}(\bs x;\Mscr)\|_{\bs M}^2,
    \qquad
    (\bs x,\bs M)\in\Rbb^n\times\Sbb^n_{++},
\end{equation*}
is jointly continuous.
\end{corollary}

\begin{proof}
Define the maps $ f_1\colon \Rbb^n\times\Sbb^n_{++}\to\Rbb^n\times\Sbb^n_{++}\times\Rbb^n$ and $ f_2\colon \Rbb^n\times\Sbb^n_{++}\times\Rbb^n\to \Rbb$ as
\begin{equation*}
f_1(\bs x,\bs M)
    \isdef
    \bigl(\bs x,\bs M,P_{\bs M}(\bs x;\Mscr)\bigr), \qquad f_2(\bs x,\bs M,\bs y)
    \isdef
    \left\langle
    \bs x-\bs y,
    \bs M(\bs x-\bs y)
    \right\rangle_{\Rbb^n}.
\end{equation*}
By Proposition~\ref{prop:joint_continuity_projector}, the map $f_1$ is continuous, while $f_2$ is continuous by the continuity of the bilinear form induced by the Euclidean inner product on $\Rbb^n$. Since
\begin{equation*}
    g(\bs x,\bs M)
    =
    \left\langle
    \bs x-P_{\bs M}(\bs x;\Mscr),
    \bs M\bigl(\bs x-P_{\bs M}(\bs x;\Mscr)\bigr)
    \right\rangle_{\Rbb^n}
    =
    (f_2\circ f_1)(\bs x,\bs M),
\end{equation*}
the map $g$ is jointly continuous as a composition of continuous maps.
\end{proof}

\paragraph{One-sided Wald limit.}
\begin{proof}[Proof of Theorem~\ref{thm:test_statistic}]
Since $\Rbb_+^n$ is a cone, the substitution
$\bs u=\sqrt{T}\,\bs w$ gives
\begin{equation*}
    W_T
    =
    \dist^2_{\widehat{\bs\Omega}_{\lambda,n,T}^{-1}}
    \left(
        \sqrt{T}\,\widehat{\bs\theta}_{\lambda,n,T},
        \Rbb_+^n
    \right).
\end{equation*}
Under $\bs\theta_{\lambda,n}=\bs0$,
Proposition~\ref{proposition:asymptotic_convexity} gives
\begin{equation*}
    \sqrt{T}\,\widehat{\bs\theta}_{\lambda,n,T}
    \overset{d}{\longrightarrow}
    \bs Z,
    \qquad
    \bs Z\sim\Ncal_n
    \left(
        \bs0,
        \bs\Omega_{\lambda,n}
    \right).
\end{equation*}
By Theorem~\ref{thm:consistency}, positive definiteness of
$\bs\Omega_{\lambda,n}$, and continuity of matrix inversion on $\Sbb^n_{++}$,
\begin{equation*}
    \widehat{\bs\Omega}_{\lambda,n,T}^{-1}
    \overset{p}{\longrightarrow}
    \bs\Omega_{\lambda,n}^{-1}.
\end{equation*}
By Slutsky's theorem, $\left(
    \sqrt{T}\,\widehat{\bs\theta}_{\lambda,n,T},
    \widehat{\bs\Omega}_{\lambda,n,T}^{-1}
\right)
\overset{d}{\longrightarrow}
\left(
    \bs Z,
    \bs\Omega_{\lambda,n}^{-1}
\right).$ Since Corollary~\ref{cor:projection_error_continuity} shows that $(\bs x,\bs M)
\mapsto
\dist_{\bs M}^2(\bs x,\Rbb_+^n)$ is jointly continuous on $\Rbb^n\times\Sbb_{++}^n$, the continuous mapping theorem yields
\begin{equation*}
    W_T
    \overset{d}{\longrightarrow}
    \dist^2_{\bs\Omega_{\lambda,n}^{-1}}
    \left(
        \bs Z,
        \Rbb_+^n
    \right).
\end{equation*}
By Proposition~\ref{prop:chi_bar_squared}, with
$\Mscr=\Rbb_+^n$, the limiting random variable satisfies
\begin{equation*}
    \dist^2_{\bs\Omega_{\lambda,n}^{-1}}
    \left(
        \bs Z,
        \Rbb_+^n
    \right)
    \sim
    \Bar\chi^2\!\left(
        \bs\Omega_{\lambda,n},
        (\Rbb_+^n)^\circ
    \right),
\end{equation*}
which proves \eqref{eq:WT_limit}.
\end{proof}

\paragraph{Asymptotic size control and consistency.}
\begin{proof}[Proof of Proposition~\ref{prop:level}]
Define
\begin{equation*}
     \widetilde W_T
    \isdef
    \dist^2_{\widehat{\bs\Omega}_{\lambda,n,T}^{-1}}
    \left(
        \sqrt T
        (\widehat{\bs\theta}_{\lambda,n,T}
        -\bs\theta_{\lambda,n}),
        \Rbb^n_{+}
    \right).
\end{equation*}
Since $\Rbb^n_{+}$ is a nonempty closed convex cone, it holds that $\Rbb^n_{+} + \bs c \subseteq \Rbb^n_{+}$ for any $\bs c \in \Rbb^{n}_{+}$, see \citet[Corollary~6.50]{Bauschke2017}. Thus, 
for any $\bs M\in \Sbb^n_{++}$ and $\bs x\in \Rbb^n$,
\begin{align*}
    \dist_{\bs M}^2(\bs x+\bs c,\Rbb^n_{+})
    \leq
    \inf_{\bs w\in \Rbb^{n}_{+}}
    \|\bs x+\bs c-(\bs w+\bs c)\|_{\bs M}^2 =
    \dist_{\bs M}^2(\bs x,\Rbb^n_{+}).
\end{align*}
Under $\bs\theta_{\lambda,n}\in \Rbb^n_{+}$, taking
$\bs c=\sqrt T\,\bs\theta_{\lambda,n}$ yields $  W_T\leq\widetilde W_T.$
By Proposition~\ref{proposition:asymptotic_convexity},
Theorem~\ref{thm:consistency}, and continuity of matrix inversion on
$\Sbb_{++}^n$,
\begin{equation*}
   \left(
    \sqrt T
    (\widehat{\bs\theta}_{\lambda,n,T}
    -\bs\theta_{\lambda,n}),
    \widehat{\bs\Omega}_{\lambda,n,T}^{-1}
\right)
\overset{d}{\longrightarrow}
\left(
    \bs Z,
    \bs\Omega_{\lambda,n}^{-1}
\right), 
\end{equation*}
where
$\bs Z\sim\Ncal_n(\bs0,\bs\Omega_{\lambda,n})$. Thus,
Corollary~\ref{cor:projection_error_continuity} and the continuous
mapping theorem give
\begin{equation*}
     \widetilde W_T
    \overset{d}{\longrightarrow}
    W
    \isdef
    \dist^2_{\bs\Omega_{\lambda,n}^{-1}}(\bs Z,\Rbb^n_{+}),
\end{equation*}
and Proposition~\ref{prop:chi_bar_squared} gives $ W\sim
    \Bar\chi^2\!\left(
        \bs\Omega_{\lambda,n},
        \left(\Rbb^n_{+}\right)^\circ
    \right).$ Hence, at every continuity point $c$ of the distribution of $W$, it holds, 
\begin{equation*}
      \limsup_{T\to\infty}\Pr(W_T>c)
    \leq
    \lim_{T\to\infty}\Pr(\widetilde W_T>c)
    =
    \Pr(W>c),
\end{equation*}
which proves \eqref{eq:levelbound_main}.
\end{proof}

\begin{proof}[Proof of Proposition~\ref{prop:fixed_power}]
Consistency of $\widehat h_{\lambda,T}$ (see Theorem~\ref{thm:asymptotic_properties}) and boundedness of
$\Dscr_{n,\bs\alpha}$ imply $  \widehat{\bs\theta}_{\lambda,n,T}
   \overset{p}{\longrightarrow}
    \bs\theta_{\lambda,n},$ while Theorem~\ref{thm:consistency} gives
$ \widehat{\bs\Omega}_{\lambda,n,T}
   \overset{p}{\longrightarrow}
    \bs\Omega_{\lambda,n}.$
Moreover,
\begin{equation*}
     \frac{W_T}{T}
    =
    \dist^2_{\widehat{\bs\Omega}_{\lambda,n,T}^{-1}}
    \left(
        \widehat{\bs\theta}_{\lambda,n,T},
        \Rbb_+^n
    \right).
\end{equation*}
Applying the same continuity argument as in the proof of
Theorem~\ref{thm:test_statistic} yields
\begin{equation*}
    \frac{W_T}{T}
    \overset{p}{\longrightarrow}
    \dist^2_{\bs\Omega_{\lambda,n}^{-1}}
    \left(
        \bs\theta_{\lambda,n},
        \Rbb_+^n
    \right).
\end{equation*}
The limit is strictly positive because $\Rbb_+^n$ is closed and
$\bs\theta_{\lambda,n}\notin\Rbb_+^n$. Hence
$W_T\overset{p}{\longrightarrow}\infty$ and, for every finite $c$,
$\Pr(W_T>c)\to1$.
\end{proof}

\paragraph{Relation to the minimum-norm target.}
\begin{proof}[Proof of Proposition~\ref{prop:regularization_bridge}]
For each $1\leq j\leq n$, the derivative reproducing property
from Theorem~\ref{result:zhou} and Cauchy--Schwarz give
\begin{equation*}
\begin{aligned}
    |\theta_{\lambda,n,j}-\theta_{0,n,j}|=
    \left|
        \left\langle
            h_\lambda-h_0,
            \phi^{(\bs\alpha)}(\bs\xi_j)
        \right\rangle_\Hscr
    \right|\leq
    \|h_\lambda-h_0\|_\Hscr
    \|\phi^{(\bs\alpha)}(\bs\xi_j)\|_\Hscr.
\end{aligned}
\end{equation*}
Taking the maximum over the fixed grid and using
\eqref{eq:rho_derivative_grid} and
Proposition~\ref{prop:reg_consistency} gives
\begin{equation*}
    \max_{1\leq j\leq n}
    |\theta_{\lambda,n,j}-\theta_{0,n,j}|
    \leq
    \rho \|h_\lambda-h_0\|_\Hscr 
    \lesssim
    \lambda^\nu.
\end{equation*}
If
$m_0\isdef\min_{1\leq j\leq n}\theta_{0,n,j}>0$, then
\eqref{eq:grid_regularization_bias} implies that, for all sufficiently
small $\lambda$,
\begin{equation*}
    \max_{1\leq j\leq n}
    |\theta_{\lambda,n,j}-\theta_{0,n,j}|<m_0.
\end{equation*}
Thus, for every $1\leq j\leq n$, $ \theta_{\lambda,n,j}
    \geq
    \theta_{0,n,j}
    -
    |\theta_{\lambda,n,j}-\theta_{0,n,j}|
    >
    0.$
Thus
$\bs\theta_{\lambda,n}\in\operatorname{int}(\Rbb_+^n)$,
which proves Part~1. 

If instead $\bs\theta_{0,n}\notin\Rbb_+^n$, then
$\theta_{0,n,j_0}<0$ for some $j_0$. By
\eqref{eq:grid_regularization_bias},
$\theta_{\lambda,n,j_0}<0$ for all sufficiently small $\lambda$.
Hence
$\bs\theta_{\lambda,n}\notin\Rbb_+^n$, and
Proposition~\ref{prop:fixed_power} gives $    \Pr(W_T>c)\to 1$
for every fixed finite $c$. This proves Part~2.

Finally, if
$\bs\theta_{0,n}\in\partial\Rbb_+^n$, there is no strictly positive
componentwise margin, so
\eqref{eq:grid_regularization_bias} alone does not determine the signs
of the corresponding components of $\bs\theta_{\lambda,n}$.
\end{proof}


\section{Implementation}\label{appendix:implementation}
This appendix describes the finite-dimensional computation of the empirical RKHS estimator and of the finite-grid covariance quantities required to construct the Wald statistic.

\paragraph{Efficient empirical estimator.}
Let $ \bs\Psi\isdef[\psi_1,\ldots,\psi_T]$ denote the row vector of empirical payoff representers, interpreted as the operator $  \bs\Psi\colon\Rbb^T\to\Hscr$ defined as $ \bs b\mapsto\sum_{t=1}^T b_t\psi_t.$ Recall from \eqref{eq:empirical_payoff_gram} that the Gram matrix is 
\begin{equation*}
    \bs K_T
    \isdef
    \langle\bs\Psi^\top,\bs\Psi\rangle_\Hscr
    =
    \big[\langle\psi_t,\psi_s\rangle_\Hscr\big]_{t,s=1}^T
    \in\Rbb^{T\times T}.
\end{equation*}
By \eqref{eq:coefficient_quadratic_problem}, the empirical estimator can be written as $ \widehat h_{\lambda,T}
    =\bs\Psi\widehat{\bs b}_{\lambda,T},$ where $\widehat{\bs b}_{\lambda, T}$ solves the problem
\begin{equation*}
    (\bs K_T+T\lambda \bs I_T)\widehat{\bs b}_{\lambda, T}=\bs y.
\end{equation*}
For large $T$, solving this system directly is memory-intensive, and its computational cost is cubic in $T$. We therefore use the \emph{pivoted Cholesky} algorithm of \citet{harbrecht2012low}, which does not require forming the full Gram matrix $\bs K_T$. For large $T$ and prescribed tolerance $\varepsilon > 0$, consider a rank-$r$ pivoted-Cholesky approximation
\begin{equation*}
    \bs K_T
    \approx
    \bs L_r\bs L_r^\top,
    \qquad
    \bs L_r\in\Rbb^{T\times r},
    \qquad
    r\ll T,
\end{equation*}
where the approximation satisfies, see  \citet[Theorem~4.1]{filipovic2021adaptive},
\begin{equation}\label{eq:pivoted_cholesky}
    \bs K_T-\bs L_r\bs L_r^\top\geq \bs 0,
    \qquad
    \tr\big(
        \bs K_T-\bs L_r\bs L_r^\top
    \big)
    \leq \varepsilon.
\end{equation}
The construction also provides a biorthogonal matrix
$\bs B_r\in\Rbb^{T\times r}$ satisfying
\begin{equation}\label{eq:biorthogonal_relations}
    \bs K_T\bs B_r=\bs L_r,
    \qquad
    \bs B_r^\top\bs L_r
    =
    \bs L_r^\top\bs B_r
    =
    \bs I_r.
\end{equation}
We approximate the empirical estimator by restricting the optimization to the
rank-$r$ subspace $\Hscr_r \isdef \operatorname{Im}(\bs \Psi \bs B_r)\subset \Hscr$. Accordingly, we write
\begin{equation*}
    \wt h_{\lambda,T}
    =
    \bs\Psi\bs B_r\widetilde{\bs b}_{\lambda,T},
    \qquad
    \widetilde{\bs b}_{\lambda,T}\in\Rbb^r.
\end{equation*}
The relations in \eqref{eq:biorthogonal_relations} and the expression of the Gram matrix imply
\begin{align*}
\Big[
        \langle
            \wt h_{\lambda,T},
            \psi_t
        \rangle_\Hscr
    \Big]_{t=1}^T
    &=
    \bs K_T\bs B_r\widetilde{\bs b}_{\lambda,T}
    =
    \bs L_r\widetilde{\bs b}_{\lambda,T},\\
    \big\|
        \wt h_{\lambda,T}
    \big\|_\Hscr^2
    &=
    \widetilde{\bs b}_{\lambda,T}^{\top}
    \bs B_r^\top\bs K_T\bs B_r
    \widetilde{\bs b}_{\lambda,T}
    =
    \big\|
        \widetilde{\bs b}_{\lambda,T}
    \big\|_2^2,
\end{align*}
where $ \bs B_r^\top\bs K_T\bs B_r
    =
    \bs B_r^\top\bs L_r
    =
    \bs I_r.$
Hence, restricting the empirical problem to $\Hscr_r$ gives the
$r$-dimensional ridge problem
\begin{equation*}
    \widetilde{\bs b}_{\lambda,T}
    =
    \argmin_{\bs b\in\Rbb^r}
    \left\{
        \frac{1}{2T}
        \|\bs y-\bs L_r\bs b\|_2^2
        +
        \frac{\lambda}{2}
        \|\bs b\|_2^2
    \right\},
\end{equation*}
whose first-order condition is
\begin{equation}\label{eq:reduced_linear_system}
    \left(
        \bs L_r^\top\bs L_r
        +
        T\lambda\bs I_r
    \right)
    \widetilde{\bs b}_{\lambda,T}
    =
    \bs L_r^\top\bs y.
\end{equation}
Therefore, $ \widetilde{\bs b}_{\lambda,T}
    =
    \left(
        \bs L_r^\top\bs L_r
        +
        T\lambda\bs I_r
    \right)^{-1}
    \bs L_r^\top\bs y,$ and the approximate empirical estimator is
\begin{equation*}
    \wt h_{\lambda,T}
    =
    \bs\Psi\bs B_r
    \left(
        \bs L_r^\top\bs L_r
        +
        T\lambda\bs I_r
    \right)^{-1}
    \bs L_r^\top\bs y.
\end{equation*}
The columns of $\bs\Psi\bs B_r$ are orthonormal in $\Hscr$, while
$\bs L_r$ contains their evaluations against the empirical payoff
representers. Thus, \eqref{eq:reduced_linear_system} replaces the original
$T\times T$ system with an $r\times r$ system. After the pivoted Cholesky
factorization, the required storage is of order $Tr$, and the principal
computational cost is of order $Tr^2+r^3$.


\subsection{Construction of the finite-grid test statistic}\label{appendix:test_statistic}

\paragraph{Finite-grid test statistic.}
We first construct the finite-grid Wald statistic of
Theorem~\ref{thm:test_statistic} from the full-coordinate
estimator $\widehat h_{\lambda,T}$
and then describe its reduced implementation. Recall the statistic
$W_T$ in \eqref{eq:test_statistic}, where $\widehat{\bs\theta}_{\lambda,n,T}$ is the signed vector of derivative
evaluations of $\widehat h_{\lambda,T}$ on the testing grid
$\Zcal\isdef\{\bs\xi_j\}_{j=1}^n$. For the derivative multi-index $\bs\alpha$ and sign
$\varsigma$ used in \eqref{eq:theta}, define the signed row vector
of test-grid representers
\begin{equation*}
    \bs\Phi_{\Zcal}
    \isdef
    \varsigma
    \big[
        \phi^{(\bs\alpha)}(\bs\xi_1),
        \ldots,
        \phi^{(\bs\alpha)}(\bs\xi_n)
    \big].
\end{equation*}
The corresponding sample-to-grid cross-Gram matrix is
\begin{equation*}
    \bs K_{T,\Zcal}
    \isdef
    \langle
        \bs\Psi^\top,
        \bs\Phi_{\Zcal}
    \rangle_{\Hscr}
    \in\Rbb^{T\times n}.
\end{equation*}
Set $ \bs M_{\lambda,T}
    \isdef
    \bs K_T+T\lambda\bs I_T.$ Since
$\widehat{\bs b}_{\lambda,T}
=\bs M_{\lambda,T}^{-1}\bs y$, the signed derivative vector is
\begin{equation*}
    \widehat{\bs\theta}_{\lambda,n,T}
    =
    \langle
        \bs\Phi_{\Zcal}^{\top},
        \widehat h_{\lambda,T}
    \rangle_{\Hscr}
    =
    \bs K_{T,\Zcal}^{\top}
    \widehat{\bs b}_{\lambda,T}
    =
    \bs K_{T,\Zcal}^{\top}
    \bs M_{\lambda,T}^{-1}\bs y.
\end{equation*}
The fitted empirical payoffs are then stacked into the vector
\begin{equation*}
    \widehat{\bs\eta}
    \isdef
    \langle
        \bs\Psi^\top,
        \widehat h_{\lambda,T}
    \rangle_{\Hscr}
    =
    \bs K_T\widehat{\bs b}_{\lambda,T}
    \in\Rbb^T.
\end{equation*}
Finally, define
\begin{equation*}
    \bs A_{\lambda,T,\Zcal}
    \isdef
    \bs M_{\lambda,T}^{-1}\bs K_{T,\Zcal}
    \in\Rbb^{T\times n}.
\end{equation*}

\paragraph{Grid-score covariance.}
Let $ \widehat{\bs F}
    \isdef \big[\widehat F_1,\ldots,\widehat F_T\big]$. Using the first-order condition
$\widehat\Sig_{\lambda,T}\widehat h_{\lambda,T}=\widehat\mu_T$, the feasible
influence terms in \eqref{eq:sample_Q} satisfy
\begin{equation*}
    \widehat F_t
    =
    \left(
        Y_t-
        \langle
            \widehat h_{\lambda,T},
            \psi_t
        \rangle_{\Hscr}
    \right)\psi_t
    -
    \lambda\widehat h_{\lambda,T}.
\end{equation*}
They can therefore be written as
\begin{equation}\label{eq:F_vector}
    \widehat{\bs F}
    =
    \bs\Psi
    \operatorname{diag}
    \big(
        \bs y-\widehat{\bs\eta}
    \big)
    -
    \lambda
    \widehat h_{\lambda,T}\bs 1_T^\top,
\end{equation}
where $\bs 1_T\in\Rbb^T$ is the vector of ones. Let $ \widehat{\bs U}
    \isdef \varsigma\big[
        \widehat u_1,\ldots,\widehat u_n\big]
    =\widehat\Sig_{\lambda,T}^{-1}\bs\Phi_{\Zcal}.$ Since $ \widehat\Sig_T
    =
    \frac1T\bs\Psi\bs\Psi^\ast,$ the Woodbury identity gives
\begin{equation*}
    \widehat\Sig_{\lambda,T}^{-1}
    =
    \frac1\lambda
    \left(
        I_{\Hscr}
        -
        \bs\Psi
        \bs M_{\lambda,T}^{-1}
        \bs\Psi^\ast
    \right).
\end{equation*}
Thus,
\begin{equation}\label{eq:U_vector}
    \widehat{\bs U}
    =
    \frac1\lambda
    \left(
        \bs\Phi_{\Zcal}
        -
        \bs\Psi
        \bs A_{\lambda,T,\Zcal}
    \right).
\end{equation}
Define the matrix of feasible grid scores by
\begin{equation*}
    \bs S
    \isdef
    \langle
        \widehat{\bs F}^{\top},
        \widehat{\bs U}
    \rangle_{\Hscr}
    \in\Rbb^{T\times n}.
\end{equation*}
Combining \eqref{eq:F_vector} and \eqref{eq:U_vector} gives
\begin{align*}
    \bs S
    =
    \frac1\lambda
    \operatorname{diag}
    \big(
        \bs y-\widehat{\bs\eta}
    \big)
    \left(
        \bs K_{T,\Zcal}
        -
        \bs K_T\bs A_{\lambda,T,\Zcal}
    \right)
    +
    \bs 1_T
    \left(
        \widehat{\bs\eta}^{\top}
        \bs A_{\lambda,T,\Zcal}
        -
        \widehat{\bs\theta}_{\lambda,n,T}^{\top}
    \right).
\end{align*}
Since $ \bs M_{\lambda,T}\bs A_{\lambda,T,\Zcal}
    =
    \bs K_{T,\Zcal},$ we have $ \bs K_{T,\Zcal}
    -
    \bs K_T\bs A_{\lambda,T,\Zcal}
    =
    T\lambda\bs A_{\lambda,T,\Zcal}.$ Hence,
\begin{equation*}
    \bs S
    =
    T
    \operatorname{diag}
    \big(
        \bs y-\widehat{\bs\eta}
    \big)
    \bs A_{\lambda,T,\Zcal}
    +
    \bs 1_T
    \left(
        \widehat{\bs\eta}^{\top}
        \bs A_{\lambda,T,\Zcal}
        -
        \widehat{\bs\theta}_{\lambda,n,T}^{\top}
    \right).
\end{equation*}
Writing the rows of $\bs S$ as
$\widehat{\bs s}_t^\top$, the feasible estimator
$\widehat{\bs\Omega}_{\lambda,n,T}$ is obtained by inserting the score
sequence
$\{\widehat{\bs s}_t\}_{t=1}^T$ into \eqref{eq:HAC_main}. In the serially
uncorrelated case, this reduces to
\begin{equation*}
    \widehat{\bs\Omega}_{\lambda,n,T}^{(0)}
    =
    \frac1T\bs S^\top\bs S.
\end{equation*}

\paragraph{Wald statistic.}
Once
$\widehat{\bs\theta}_{\lambda,n,T}$ and
$\widehat{\bs\Omega}_{\lambda,n,T}$ have been computed, and provided $\widehat{\bs\Omega}_{\lambda,n,T}$ is positive definite, the test statistic is
obtained from \eqref{eq:test_statistic}. Equivalently,
\begin{equation*}
    W_T
    =
    T
    \min_{\bs w\in\Rbb_+^n}
    \left\|
        \widehat{\bs\Omega}_{\lambda,n,T}^{-1/2}\bs w
        -
        \bs a
    \right\|_2^2, \qquad \bs a
    \isdef
    \widehat{\bs\Omega}_{\lambda,n,T}^{-1/2}
    \widehat{\bs\theta}_{\lambda,n,T}.
\end{equation*}
Thus, the minimizer
\begin{equation*}
    \bs w^\star
    \isdef
    \argmin_{\bs w\in\Rbb_+^n}
    \left\|
        \widehat{\bs\Omega}_{\lambda,n,T}^{-1/2}\bs w
        -
        \bs a
    \right\|_2^2
\end{equation*}
is obtained from a standard nonnegative least-squares problem. If $\bs e^\star
    \isdef
    \widehat{\bs\Omega}_{\lambda,n,T}^{-1/2}\bs w^\star
    -
    \bs a$ denotes the resulting residual, then $ W_T=T\|\bs e^\star\|_2^2.$

\paragraph{Efficient implementation.}
In computation, the same construction is
carried out in the reduced subspace
$\Hscr_r=\operatorname{Im}(\bs\Psi\bs B_r)$. More precisely, in the above, $\widehat h_{\lambda,T}$ is replaced by
$\wt h_{\lambda,T}$, the coefficient vector
$\widehat{\bs b}_{\lambda,T}$ is replaced by
$\bs B_r\widetilde{\bs b}_{\lambda,T}$, and the fitted payoff vector and
signed derivative vector become
\begin{equation*}
    \widetilde{\bs\eta}
    =
    \bs L_r\widetilde{\bs b}_{\lambda,T},
    \qquad
    \widetilde{\bs\theta}_{\lambda,n,T}
    =
    \bs K_{T,\Zcal}^{\top}
    \bs B_r\widetilde{\bs b}_{\lambda,T}.
\end{equation*}
The feasible influence terms, inverse-covariance representers, grid scores,
and the feasible long-run covariance matrix are then constructed from these reduced
quantities by applying the preceding calculations within $\Hscr_r$.
Accordingly, all dense $T\times T$ linear solves are replaced by the
$r\times r$ system in \eqref{eq:reduced_linear_system}. The resulting
statistic is the pivoted-Cholesky approximation of
$W_T$. The tolerance in \eqref{eq:pivoted_cholesky} controls the Gram residual $\bs K_T-\bs L_r\bs L_r^\top$, not the resulting error in $W_T$. As in the full-coordinate case, the reduced long-run covariance matrix must be positive definite; since the reduced grid scores lie in a space of dimension at most $r$, this requires in particular that $r\geq n$.

\paragraph{Monte Carlo calibration.}
The plug-in null distribution $\Bar\chi^2(\widehat{\bs\Omega}_{\lambda,n,T},(\Rbb_+^n)^\circ)$ is evaluated by simulation, so the mixture weights need not be computed. Conditional on $\widehat{\bs\Omega}_{\lambda,n,T}$, draw independently
\begin{equation*}
    \bs Z^{(b)}
    \sim
    \Ncal_n\bigl(\bs 0,\widehat{\bs\Omega}_{\lambda,n,T}\bigr),
    \qquad
    b=1,\ldots,B,
\end{equation*}
and compute
\begin{equation*}
    W_T^{(b)}
    \isdef
    \min_{\bs w\in\Rbb_+^n}
    \bigl(\bs Z^{(b)}-\bs w\bigr)^\top
    \widehat{\bs\Omega}_{\lambda,n,T}^{-1}
    \bigl(\bs Z^{(b)}-\bs w\bigr).
\end{equation*}
The Monte Carlo $p$-value is
\begin{equation*}
    \widehat p
    \isdef
    \frac{1+\sum_{b=1}^B\mathbbm{1}\bigl\{W_T^{(b)}\geq W_T\bigr\}}{B+1},
\end{equation*}
and the test rejects at level $\alpha$ when $\widehat p\leq\alpha$. In particular, $\widehat p=1$ whenever $W_T=0$. Calibration is performed separately for each application, subsample, conditioning state, testing region, and shape restriction.

As $B\to\infty$, this test approaches the plug-in test that rejects when $W_T\geq\widehat c_{1-\alpha}$, the $(1-\alpha)$-quantile of $\Bar\chi^2(\widehat{\bs\Omega}_{\lambda,n,T},(\Rbb_+^n)^\circ)$. Let $c_{1-\alpha}$ be the $(1-\alpha)$-quantile of $W$ in Proposition~\ref{prop:level}. Since $\Pr(W=0)=\Pr(\bs Z\in\Rbb_+^n)\leq1/2$ for $\bs Z\sim\Ncal_n(\bs 0,\bs\Omega_{\lambda,n})$, the mixture representation implies that, for $\alpha<1/2$, $c_{1-\alpha}>0$ and the distribution function of $W$ is continuous and strictly increasing at $c_{1-\alpha}$. Theorem~\ref{thm:consistency} and Corollary~\ref{cor:projection_error_continuity} then give $\widehat c_{1-\alpha}\overset{p}{\longrightarrow}c_{1-\alpha}$, and \eqref{eq:levelbound_main}, applied at continuity points arbitrarily close to $c_{1-\alpha}$, yields $\limsup_{T\to\infty}\Pr(W_T\geq\widehat c_{1-\alpha})\leq\alpha$ under the null. For given data, $\widehat p$ differs from the plug-in $p$-value by $O_p(B^{-1/2})$ as $B\to\infty$. 

\subsection{Selection of the regularization parameter}\label{appendix:lambda_selection}

The penalty determines both the regularized target $h_\lambda$ whose shape is tested and the precision with which it is estimated: a smaller penalty brings $h_\lambda$ closer to the minimum-norm target $h_0$ but can make the derivative estimates less precise. The theory of Section~\ref{sec:inference} takes $\lambda>0$ as given, whereas in applications it must be chosen from the data. Cross-validation selects the regularization parameter according to predictive performance, which need not be the appropriate criterion for inference on derivatives. Moreover, with noisy returns and relatively short estimation histories, prediction-based tuning can be imprecise. We instead select the penalty through the effective dimension of the estimator, a criterion determined by the spectrum of the empirical second-moment operator.

Let $\lambda_1(\widehat\Sig_T)\geq\lambda_2(\widehat\Sig_T)\geq\cdots\geq0$ denote the
eigenvalues of $\widehat\Sig_T$ in \eqref{eq:empirical_moments}. The empirical counterpart of the effective dimension
\eqref{eq:effective_dimension} is
\begin{equation*}
    \widehat\Ncal_T(\lambda)
    \isdef
    \tr\!\left\{\widehat\Sig_T\widehat\Sig_{\lambda,T}^{-1}\right\}
    =
    \sum_{m\geq1}\frac{\lambda_m(\widehat\Sig_T)}{\lambda_m(\widehat\Sig_T)+\lambda}.
\end{equation*}
Each eigendirection of $\widehat\Sig_T$ enters with the factor $\lambda_m(\widehat\Sig_T)/\{\lambda_m(\widehat\Sig_T)+\lambda\}\in[0,1)$, by which the penalty shrinks the fit along it. With one payoff representer per period, $\widehat\Ncal_T(\lambda)=\tr\{\bs K_T(\bs K_T+T\lambda\bs I_T)^{-1}\}$ is the trace of the matrix that maps $\bs y$ into the fitted values $\bs K_T\widehat{\bs b}_{\lambda,T}$, that is, the effective degrees of freedom of the fit as a linear smoother \citep{hastie_statisticallearning}\footnote{Appendix~\ref{appendix:additional_details} gives the pooled version of the same.}. Its population counterpart $\Ncal(\lambda)$ is the capacity term in Theorem~\ref{thm:generic_finite_sample_rkhs_projection}. If $\widehat\Sig_T\neq0$, the map $\lambda\mapsto\widehat\Ncal_T(\lambda)$ decreases continuously and strictly from the rank of $\widehat\Sig_T$ to zero, so fixing the effective dimension at a target strictly between zero and the rank of $\widehat\Sig_T$, as is common for linear smoothers, determines a single penalty.

We parameterize the penalty as $\lambda=\eta\,\tr(\widehat\Sig_T)$, where
$\tr(\widehat\Sig_T)=T^{-1}\sum_{t=1}^T\|\psi_t\|_\Hscr^2$ is the scale of the quadratic term $\langle h,\widehat\Sig_T h\rangle_\Hscr$ in the empirical criterion. Rescaling the payoff representers by $c>0$ multiplies $\widehat\Sig_T$ and $\tr(\widehat\Sig_T)$ by $c^2$ and leaves the selected $\eta$, the effective dimension, and the fitted payoffs unchanged. This normalization parallels that of \citet{KozakNagelSantosh2020}, whose ridge penalty is proportional to the trace of the factor covariance matrix; \citet{KellyKuznetsovMalamudXuZhang2026} similarly express the ridge penalty relative to the average eigenvalue of the factor second-moment matrix. We approximate the solution to $\widehat\Ncal_T(\lambda)=d$ by minimizing over a finite grid:
\begin{equation*}
    \widehat\eta
    \in
    \argmin_{\eta}
    \bigl|\widehat\Ncal_T\bigl(\eta\,\tr(\widehat\Sig_T)\bigr)-d\bigr|,
    \qquad d=5,
\end{equation*}
where the minimum is taken over the grid $\eta\in\{10^{-5},10^{-4.5},\ldots,1\}$ and
ties are resolved in favor of the largest $\eta$. The selected penalty is
$\widehat\lambda=\widehat\eta\,\tr(\widehat\Sig_T)$. We adopt this moderate effective dimension as a common target across applications and subsamples. The rule depends on the data only through $\widehat\Sig_T$, and the same $\widehat\lambda$ is used for the estimator and for the Wald statistic, for every restriction and region within an application.

The target $d$ is a practical complexity criterion; its connection to the theory is a selection-consistency argument. The results of Section~\ref{sec:inference} hold for fixed $\lambda>0$ as $T\to\infty$, a regime in which the population effective dimension $\Ncal(\lambda)$ is also fixed. In the empirical applications, $\tr(\widehat\Sig_T)=1$ (Appendix~\ref{appendix:additional_details}), so $\widehat\lambda$ takes values in a fixed finite grid. If $\widehat\Ncal_T(\lambda)\to\Ncal(\lambda)$ in probability for each grid value and the population minimizer $\lambda^\star$ of $|\Ncal(\lambda)-d|$ over the grid is unique, then $\Pr(\widehat\lambda=\lambda^\star)\to1$, and, under the remaining assumptions of Section~\ref{sec:inference}, its results apply with $\lambda=\lambda^\star$, the kernel and testing grid being taken as given (Appendix~\ref{appendix:additional_details}). The tests therefore apply to $h_{\lambda^\star}$; conclusions about $h_0$ require the conditions of Proposition~\ref{prop:regularization_bridge}.

\paragraph{Sensitivity to the effective dimension target.}
Table~\ref{tab:treasury_effdim_sensitivity} examines the sensitivity of the Treasury results for 1973--2001, the period for which Section~\ref{subsec:treasury_bill_application} contrasts the conditional and unconditional analyses, to the alternative effective dimension targets ${d=3}$ and~${d=7}$. In the conditional specification, ${d=5}$ and~${d=7}$ select the same penalty, so their results coincide. Across all three targets, weakly decreasing returns are rejected at short maturities and weakly increasing returns at long maturities in both specifications. The conditional rejections of weakly increasing returns at intermediate maturities also persist, whereas the corresponding unconditional restriction is rejected only at the lowest target. At some alternative penalties, both monotonicity restrictions are rejected within the same region, suggesting nonmonotonicity.

\begin{table}[!htbp]
\centering
\caption{Treasury-bill shape tests across effective dimension targets, 1973--2001}
\label{tab:treasury_effdim_sensitivity}
\renewcommand{\arraystretch}{1.25}

\begin{tabular}{@{}lccc@{}}
\toprule
& 1--5 Months & 5--8 Months & 8--12 Months \\
\midrule

\multicolumn{4}{@{}l}{\textbf{Panel A: Unconditional}} \\
\addlinespace[2pt]
\multicolumn{4}{@{}l}{\textit{Target: 3; Penalty: $10^{-1}$; Attained dimension: 2.77}} \\
\quad Weakly increasing & $1.000$ & $0.006^{***}$ & $<0.001^{***}$ \\
\quad Weakly decreasing & $<0.001^{***}$ & $0.202$ & $1.000$ \\
\addlinespace[2pt]
\multicolumn{4}{@{}l}{\textit{Target: 7; Penalty: $10^{-2.5}$; Attained dimension: 7.60}} \\
\quad Weakly increasing & $1.000$ & $0.429$ & $0.009^{***}$ \\
\quad Weakly decreasing & $<0.001^{***}$ & $0.004^{***}$ & $0.009^{***}$ \\

\midrule
\multicolumn{4}{@{}l}{\textbf{Panel B: Conditional}} \\
\addlinespace[2pt]
\multicolumn{4}{@{}l}{\textit{Target: 3; Penalty: $10^{-0.5}$; Attained dimension: 2.52}} \\
\quad Weakly increasing & & & \\
\qquad Short Rate 3.64\% & $0.007^{***}$ & $<0.001^{***}$ & $<0.001^{***}$ \\
\qquad Short Rate 5.73\% & $0.012^{**}$ & $<0.001^{***}$ & $<0.001^{***}$ \\
\qquad Short Rate 9.56\% & $0.002^{***}$ & $<0.001^{***}$ & $<0.001^{***}$ \\
\midrule
\quad Weakly decreasing & & & \\
\qquad Short Rate 3.64\% & $<0.001^{***}$ & $1.000$ & $1.000$ \\
\qquad Short Rate 5.73\% & $<0.001^{***}$ & $1.000$ & $1.000$ \\
\qquad Short Rate 9.56\% & $<0.001^{***}$ & $1.000$ & $1.000$ \\
\midrule
\multicolumn{4}{@{}l}{\textit{Target: 7; Penalty: $10^{-1}$; Attained dimension: 6.24}} \\
\quad Weakly increasing & & & \\
\qquad Short Rate 3.64\% & $1.000$ & $<0.001^{***}$ & $<0.001^{***}$ \\
\qquad Short Rate 5.73\% & $1.000$ & $<0.001^{***}$ & $<0.001^{***}$ \\
\qquad Short Rate 9.56\% & $0.165$ & $<0.001^{***}$ & $<0.001^{***}$ \\
\midrule
\quad Weakly decreasing & & & \\
\qquad Short Rate 3.64\% & $<0.001^{***}$ & $0.507$ & $1.000$ \\
\qquad Short Rate 5.73\% & $<0.001^{***}$ & $0.392$ & $1.000$ \\
\qquad Short Rate 9.56\% & $<0.001^{***}$ & $1.000$ & $1.000$ \\
\bottomrule
\end{tabular}

\medskip
\begin{minipage}{\textwidth}
\footnotesize
\textit{Notes:}
The null hypotheses are weakly increasing and weakly
decreasing expected excess returns in remaining maturity
for the regularized relationship of Section~\ref{subsec:data_empirical_implementation},
in Panel~B conditional on the indicated one-month Treasury yield
(percent per annum). Target is the effective dimension target $d$, Penalty the selected $\widehat\lambda$, and Attained dimension the effective dimension $\widehat\Ncal_T(\widehat\lambda)$. Results for the baseline target $d=5$ are reported in Panel~B of Tables~\ref{tab:treasury_shape_counts} and~\ref{tab:e4_shape_tests}; in the conditional specification they coincide with those for $d=7$.
The tests use 348 monthly observations, long-run covariance estimation with 5 lags, and
10,000 Monte Carlo draws.
$^{*}$, $^{**}$, and $^{***}$ indicate significance
at the 10\%, 5\%, and 1\% levels, respectively.
\end{minipage}
\end{table}

\paragraph{Fitted slopes at weaker penalties.}
We also examine the fitted slopes at the tested points for every weaker penalty on the grid with an effective dimension below 30; the largest effective dimensions considered range from 19 to 29 across applications. Slopes are obtained by numerical differentiation, with interpolation across short-rate values in the conditional analysis. The slope patterns behind the main rejections are largely preserved. Among low-beta portfolios, the fitted slopes remain positive at betas of $0.25$ and $0.5$; negative slopes appear only at zero, where observations are sparse, and, at the weakest penalties, at $0.75$. In the Treasury applications, negative slopes persist wherever weakly increasing returns are rejected, with two exceptions in the conditional analysis: over 5--8 months in 1973--2001 at the 5.73\% short rate, the slopes become positive at an effective dimension of 29, and over 8--12 months in 2009--2025 at 4.43\%, only a near-zero negative slope at twelve months remains at the weakest penalty considered. These comparisons are descriptive; the formal tests at alternative penalties are those reported in Table~\ref{tab:treasury_effdim_sensitivity}.

\section{Additional results and diagnostics}
\label{appendix:additional_results}

\subsection{Simulation studies}
\label{appendix:monte_carlo_details}

We consider the payoff model
\begin{equation*}
Y_{t+1}
=
h_0(X_{1,t})R_{1,t+1} + h_0(X_{2,t})R_{2,t+1}
+\varepsilon_{t+1},
\qquad
h_0(x)=1+4x(1-x), \quad x\in[0,1].
\end{equation*}
The state variables are generated by
\begin{equation*}
X_{j,t}=\Phi(U_{j,t}),
\qquad
U_{j,t}=0.8U_{j,t-1}
+\sqrt{1-0.8^2}\,e_{j,t},
\end{equation*}
where $\Phi$ is the CDF of $\Ncal(0,1)$
and $e_{j,t} \overset{\mathrm{i.i.d.}}{\sim} \Ncal(0,1)$.
The autoregressive processes are initialized in their
stationary distributions. This construction gives each
state variable a uniform distribution on $[0,1]$ while
introducing geometrically decaying serial dependence,
consistent with the strong-mixing condition in
Assumption~\ref{assumption:strong_mixing}. The returns $R_{j,t+1}$ are independent over time and
of the state variables, have uniform distributions on
$[-\sqrt{3},\sqrt{3}]$, and have contemporaneous
correlation $0.5$, generated through a Gaussian copula: $R_{j,t+1}=\sqrt3\,(2\Phi(V_{j,t+1})-1)$, where $(V_{1,t+1},V_{2,t+1})$ is bivariate normal with correlation $2\sin(\pi/12)$. The independent noise
$\varepsilon_{t+1}$ is uniformly distributed, with
its variance chosen to give a population $R^2$ of $0.5$.
We consider sample sizes
$T\in\{100,500,1000,5000,10000\}$,
with $1000$ Monte Carlo replications at each sample size. Increasing, decreasing, convexity, and concavity
restrictions are tested on the two regions
$[0.05,0.35]$ and $[0.65,0.95]$.
The function $h_0$ is strictly increasing in the
first region, strictly decreasing in the second,
and strictly concave throughout, so the true and
false shape restrictions are known by construction.

Estimation and inference follow the same specifications
as the empirical applications described in Appendix~\ref{appendix:additional_details}. We use the Mat\'ern-$5/2$ kernel with the
median pairwise distance lengthscale. The regularization
parameter is selected separately in each replication by the
effective dimension rule of Appendix~\ref{appendix:lambda_selection},
and the reported results average over replications.
Shape restrictions are evaluated at five equally
spaced points per region, and the least-favorable
calibration is based on $10{,}000$ Monte Carlo draws.

\begin{table}[!ht]
\centering
\caption{Estimation errors across sample sizes}
\label{tab:mc_estimation_errors}
\renewcommand{\arraystretch}{1.3}
\begin{tabular}{@{}lccccc@{}}
\toprule
& $T=100$ & $T=500$ & $T=1000$
& $T=5000$ & $T=10000$ \\
\midrule
Mean payoff-norm error
& 0.561 & 0.246 & 0.174 & 0.089 & 0.070 \\
Mean pointwise RMSE
& 0.364 & 0.160 & 0.113 & 0.057 & 0.045 \\
\bottomrule
\end{tabular}

\medskip
\begin{minipage}{\textwidth}
\footnotesize
\textit{Notes:}
Results are based on 1000 Monte Carlo replications
per sample size. Errors are measured relative to
the population function $h_0$. The pointwise RMSE is averaged over $[0.05,0.95]$.
\end{minipage}
\end{table}
\noindent Table~\ref{tab:mc_estimation_errors} reports
estimation errors across sample sizes. Both the
mean payoff-norm error and the mean pointwise RMSE
decrease steadily as the sample size increases.
From $T=100$ to $T=10000$, the mean payoff-norm
error falls from $0.561$ to $0.070$, while the
mean pointwise RMSE declines from $0.364$ to
$0.045$. Figure~\ref{fig:section6_recovery}
illustrates the pointwise recovery errors for
$T=1000$. The estimates are concentrated around
the population function, with only small systematic
deviations. Their dispersion is relatively uniform
across the interior of the domain and increases
near the endpoints.
\begin{figure}[!ht]
\centering
\includegraphics[width=0.65\textwidth]
{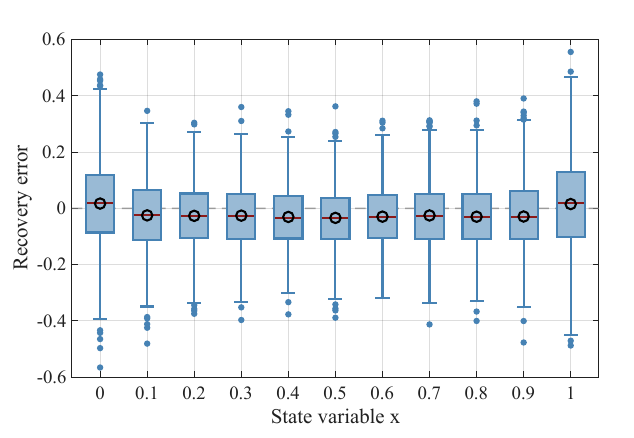}
\caption{Pointwise estimation errors.
The figure shows the distribution of differences
between the estimated and population functions
at eleven equally spaced points on $[0,1]$,
using 1000 Monte Carlo replications with $T=1000$.
The boxes summarize the distribution of errors at
each point, the circles indicate their averages,
and the dashed horizontal line marks zero error.}
\label{fig:section6_recovery}
\end{figure}

Figure~\ref{fig:section6_power} reports rejection frequencies at the $1\%$, $5\%$, and $10\%$ levels. The restrictions classified as true hold strictly for $h_0$ and, at every candidate penalty, for $h_\lambda$, so their rejection frequencies fall below the nominal level and vanish as $T$ grows, as expected for interior points of the null. The only departure is a mild over-rejection of concavity at $T=100$, with frequencies of $6.8\%$ and $8.3\%$ at the $5\%$ level, which disappears by $T=500$. Power rises much faster against monotonicity than against curvature violations. At the $5\%$ level, monotonicity violations are rejected in about half of the replications at $T=100$ and in more than $90\%$ at $T=500$, whereas convexity violations are rejected in fewer than $60\%$ at $T=1000$ and in nearly all only at $T=5000$, as second derivatives are estimated less precisely. These results are encouraging for the empirical applications in Section~\ref{sec:empirical_applications}, particularly given that they focus exclusively on monotonicity restrictions.
\begin{figure}[!htbp]
\centering
\includegraphics[width=0.85\textwidth]
{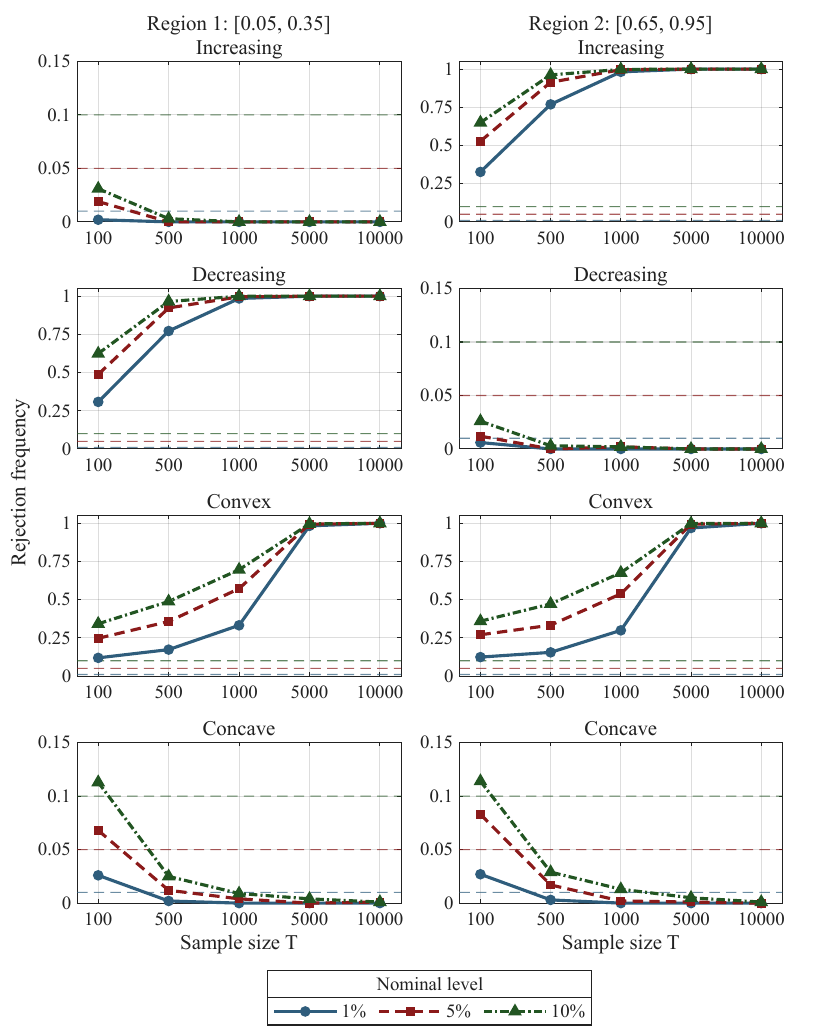}
\caption{Rejection frequencies of the shape tests.
The rows correspond to increasing, decreasing,
convexity, and concavity restrictions, and the
columns to the regions $[0.05,0.35]$ and
$[0.65,0.95]$. Increasing in the first region,
decreasing in the second region, and concavity
in both regions are true null hypotheses; the
remaining restrictions are false. The curves
show rejection frequencies at the $1\%$, $5\%$,
and $10\%$ nominal levels across five sample
sizes, based on 1000 simulated datasets per sample size, with 10,000 Monte Carlo draws used for shape-test calibration in each replication. Thin dashed lines indicate the
nominal levels.}
\label{fig:section6_power}
\end{figure}

\subsection{Additional details on the empirical exercises}
\label{appendix:additional_details}
Figure~\ref{fig:e2_histograms} shows the distributions
of the eight firm characteristics used in the
firm-characteristics application. Each observation
is the value-weighted characteristic of a decile
portfolio in a given month, rather than an individual
stock. The histograms highlight differences in the
concentration of observations, particularly in the
outer testing regions. Firm size enters the kernel as the natural logarithm of portfolio market equity; the region endpoints in Table~\ref{tab:e2_shape_tests} are converted to logarithms, and the testing points are equally spaced in log size. Because the logarithm is increasing, first-derivative signs are preserved under this transformation.
\begin{figure}[!htbp]
\centering
\includegraphics[width=0.9\textwidth]
{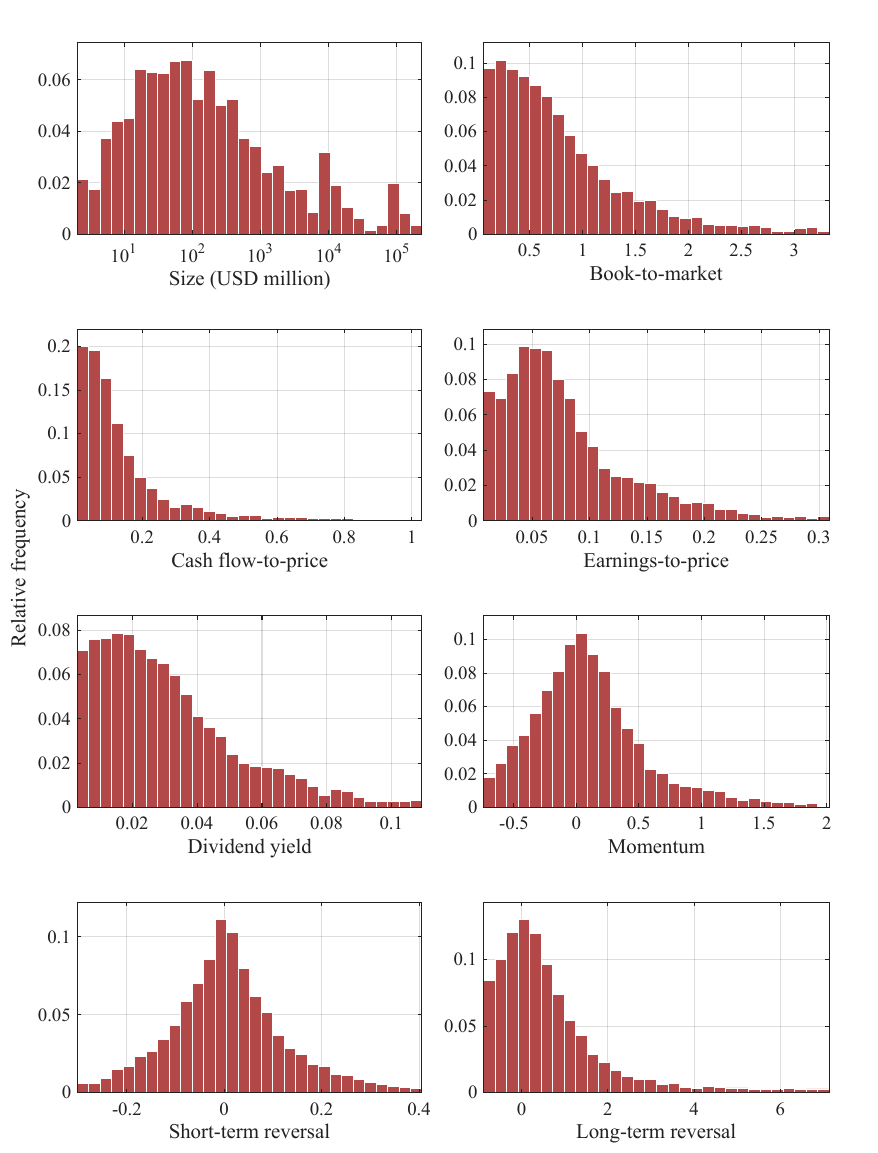}
\caption{Distributions of firm characteristics.
The panels show histograms of the portfolio-level
characteristic values used in the firm-characteristics application.
Each characteristic is displayed over its pooled
1st--99th percentile range; firm size is shown
on a logarithmic scale.}
\label{fig:e2_histograms}
\end{figure}

\medskip
For the three univariate applications (market beta, firm characteristics, and the unconditional Treasury term structure), we use the Mat\'ern-$5/2$ kernel \citep{Matern2024,genton2002}. In choosing a kernel, we seek to ensure the existence of the derivatives involved in the shape restrictions without imposing unnecessarily stringent regularity on the functions of interest. The Mat\'ern family is particularly convenient in this respect, since the associated RKHSs are norm-equivalent to Sobolev spaces whose order of smoothness is determined by the Mat\'ern smoothness parameter \citep[Corollary~10.48]{Wen05}. Among the Mat\'ern kernels with half-integer smoothness parameters, the Mat\'ern-$5/2$ kernel is the least smooth whose RKHS, for scalar inputs, consists entirely of twice continuously differentiable functions. For scalar inputs $x,x'$ and
lengthscale $\ell>0$, the kernel is
\begin{equation*}
\Kcal_{5/2}(x,x';\ell)
=
\left(
1+\frac{\sqrt{5}|x-x'|}{\ell}
+\frac{5(x-x')^2}{3\ell^2}
\right)
\exp\left(
-\frac{\sqrt{5}|x-x'|}{\ell}
\right).
\end{equation*}
For the conditional Treasury term-structure application, we use the tensor-product kernel
\begin{equation*}
\Kcal\big((x,r),(x',r')\big)
=
\Kcal_{5/2}(x,x';\ell_x)
\Kcal_{1/2}(r,r';\ell_r),
\end{equation*}
where $x$ denotes remaining maturity, $r$ is the short rate measured by the one-month Treasury yield, and the
Mat\'ern-$1/2$ (Laplacian) factor is
\begin{equation*}
\Kcal_{1/2}(r,r';\ell_r)
=
\exp\left(
-\frac{|r-r'|}{\ell_r}
\right).
\end{equation*}
The shape tests differentiate the fitted surface
only with respect to maturity.\footnote{
The Laplacian factor is not differentiable on
the diagonal, so the full tensor kernel does
not satisfy the global $\Cscr^2$ sufficient
condition of Theorem~\ref{result:zhou}.
However, the Mat\'ern-$5/2$ maturity factor
is sufficiently smooth, and the tensor-product
reproducing property ensures that
$\partial_x\Kcal((x,r),\cdot)\in\Hscr$ and
$\partial_x h(x,r)
=\langle h,\partial_x\Kcal((x,r),\cdot)\rangle_\Hscr$.
The derivative representer has uniformly bounded
norm on the compact estimation domain.
No derivatives with respect to the short-rate
coordinate are required.
A more detailed derivation is given in \citet[IA.E]{Luzzi2025}.}
Following a common practice in kernel-based statistical testing \citep{Pfister2017}, we determine the lengthscale of each kernel factor by the median pairwise distance of the corresponding inputs in the estimation sample. This choice, known as the median heuristic, adapts the lengthscale to the scale of the observed covariates without reference to the response variable. The short-rate states likewise depend only on the covariates, and the theory of Section~\ref{sec:inference} is applied with the kernel and testing grid fixed at their realized values.

\medskip
For the three univariate applications, the empirical
estimator in \eqref{eq:empirical_problem} takes the form
\begin{equation*}
\widehat h_{\lambda,T}
=
\argmin_{h\in\Hscr}
\left\{
\frac{1}{2T}
\sum_{t=1}^{T}
\sum_{j=1}^{J_t}
w_{j,t}
\bigl(Y_{t,j}-h(x_{t,j})\bigr)^2
+
\frac{\lambda}{2}\|h\|_{\Hscr}^2
\right\}.
\end{equation*}
Here, $Y_{t,j}$ is the subsequent one-month excess return on decile portfolio or Treasury bill $j$, and $x_{t,j}$ is its market beta, characteristic, or remaining maturity; the conditional Treasury application uses $h(x_{t,j},r_t)$ in place of $h(x_{t,j})$. In the equity applications, $w_{j,t}=1/J_t$, where $J_t$ is the number of retained decile portfolios in month $t$. For Treasury bills, $J_t$ is the number of eligible bills and $w_{j,t}$ is proportional to the amount outstanding. The weights sum to one within each month, so each month receives equal weight.

The pooled estimator is the kernel ridge solution in the square-root-weighted features
$\phi^w_{t,j}\isdef w_{j,t}^{1/2}\phi(x_{t,j})$ and responses
$y^w_{t,j}\isdef w_{j,t}^{1/2}Y_{t,j}$. Stacking over $(t,j)$ and writing $\bs I$ for the identity of dimension $\sum_{t=1}^TJ_t$,
\begin{equation*}
\begin{gathered}
    \relax[\bs K_T^w]_{(s,i),(t,j)}
    \isdef
    \langle\phi^w_{s,i},\phi^w_{t,j}\rangle_\Hscr
    =
    w_{i,s}^{1/2}w_{j,t}^{1/2}\Kcal(x_{s,i},x_{t,j}),
    \\
    \widehat h_{\lambda,T}
    =
    \sum_{t=1}^T\sum_{j=1}^{J_t}\widehat b_{t,j}\,\phi^w_{t,j},
    \qquad
    \widehat{\bs b}_{\lambda,T}
    =
    (\bs K_T^w+T\lambda\bs I)^{-1}\bs y^w,
\end{gathered}
\end{equation*}
and $\widehat\Ncal_T(\lambda)=\tr\{\bs K_T^w(\bs K_T^w+T\lambda\bs I)^{-1}\}$. Since $\Kcal(x,x)=1$ and the weights sum to one within each month, $\tr(\widehat\Sig_T)=1$, so the rule of Appendix~\ref{appendix:lambda_selection} gives $\widehat\lambda=\widehat\eta$. Each month contributes the moments
\begin{equation*}
    \Sig_t\isdef\sum_{j=1}^{J_t}w_{j,t}\,\phi(x_{t,j})\otimes\phi(x_{t,j}),
    \qquad
    \mu_t\isdef\sum_{j=1}^{J_t}w_{j,t}Y_{t,j}\,\phi(x_{t,j}),
\end{equation*}
which take the place of $\psi_t\otimes\psi_t$ and $Y_t\psi_t$ in Sections~\ref{sec:stat_prop}--\ref{sec:inference}. The empirical moments $\widehat\Sig_T$ and $\widehat\mu_T$ are their monthly averages, and the population target $h_\lambda=\Sig_\lambda^{-1}\mu$ minimizes the expected weighted squared loss plus penalty. Since $(\Sig_t,\mu_t)$ is a measurable function of the monthly cross-section $\bigl(J_t,(w_{j,t},x_{t,j},Y_{t,j})_{j\le J_t}\bigr)$, the asymptotic theory of Section~\ref{sec:inference} applies when Assumptions~\ref{assumption:boundedness}, \ref{assumption:stationarity}, and~\ref{assumption:strong_mixing} hold for this cross-section. Extending the finite-sample bound of
Theorem~\ref{thm:generic_finite_sample_rkhs_projection}
to this formulation would require adapting the
concentration arguments to the weighted monthly
moment operators, which need not have rank one. The moment contributions $\widehat F_t$ in \eqref{eq:sample_Q} and the grid scores $\widehat{\bs g}_t$ in \eqref{eq:gridscore_feasible} are formed from the monthly moments, and the long-run covariance \eqref{eq:HAC_main} is estimated from the $T$ monthly score vectors with Bartlett bandwidth $L_T=\lfloor 4(T/100)^{2/9}\rfloor$. Derivatives are evaluated on a grid from $0$ to $2.75$ in steps of $0.25$ for market beta and at five equally spaced points per region, including its endpoints, in the other applications. All fits and tests use the pivoted-Cholesky approximation of Appendix~\ref{appendix:implementation}.

\renewcommand{\refname}{References}
\bibliographystyle{ecta-doi}
\bibliography{literature}

\end{document}